%% file: main_arxiv.tex
\documentclass{article}

\usepackage{PRIMEarxiv}
\input{math_commands.tex}

\input{raise-lab}

\usepackage{hyperref}
\usepackage{url}
\usepackage{nicefrac}

\usepackage{amsmath}
\usepackage{amssymb}
\usepackage{mathtools}
\usepackage{amsthm}
\makeatletter

\expandafter\let\csname algorithm*\endcsname\relax
\expandafter\let\csname endalgorithm*\endcsname\relax
\makeatother
\usepackage[algoruled,vlined]{algorithm2e}
\usepackage{wrapfig}
\usepackage{needspace}
\usepackage{graphicx}   % for \resizebox
\usepackage{caption}
\usepackage{subcaption}
\usepackage[export]{adjustbox} % 
\usepackage{pifont}
\usepackage{multirow,siunitx}
\usepackage{enumitem}
\newcommand{\shad}{\cellcolor{gray!25}} 
\theoremstyle{plain}
\newtheorem{theorem}{Theorem}[section]
\newtheorem{proposition}[theorem]{Proposition}

\newtheorem{corollary}[theorem]{Corollary}
\theoremstyle{definition}

\theoremstyle{remark}
\newtheorem{remark}[theorem]{Remark}

\usepackage[algoruled,vlined]{algorithm2e}
\usepackage{wrapfig}
\usepackage{graphicx}   % for \resizebox
\usepackage{caption}
\usepackage{subcaption}
\usepackage[export]{adjustbox} % 
\usepackage{pifont}
\usepackage{booktabs,multirow,siunitx}
\usepackage{enumitem}
\usepackage[table]{xcolor}

\usepackage{natbib} 

\title{Chance-constrained Flow Matching for High-Fidelity Constraint-aware Generation
}

\author{%
  Jinhao Liang \\
  University of Virginia\\
  \texttt{jliang@email.virginia.edu} \\
  \And
  Yixuan Sun \\
  Argonne National Laboratory\\
  \texttt{yixuan.sun@anl.gov} \\
  \And
  Anirban Samaddar \\
  Argonne National Laboratory\\
  \texttt{asamaddar@anl.gov} \\
  \And
  Sandeep Madireddy\thanks{Joint advising.} \\
  Argonne National Laboratory \\
  \texttt{smadireddy@anl.gov} \\
  \And
  Ferdinando Fioretto\footnotemark[1] \\
  University of Virginia \\
  \texttt{fioretto@virginia.edu} 
}

\begin{document}
\maketitle

\begin{abstract}
Generative models excel at synthesizing high-fidelity samples from complex data distributions, but they often violate hard constraints arising from physical laws or task specifications.
A common remedy is to project intermediate samples onto the feasible set at each sampling step. However, repeated projection can disrupt the learned sampling dynamics because this feasible set is defined for ``clean'' samples rather than noisy intermediate states.
Thus, recent studies first estimate ``clean'' samples and then project them to the feasible set, but this increases algorithmic complexity and accumulates errors across steps.  
To address these issues, this paper proposes \emph{Chance-constrained Flow Matching} (CCFM), a training-free method that formulates constraint enforcement during sampling as stochastic optimization while retaining the hard-constraint feasibility guarantee of projection-based methods.
By avoiding direct projection of noisy intermediate states onto the clean-sample feasible set, CCFM mitigates the distributional distortion. Unlike approaches that first estimate and project ``clean'' samples, CCFM avoids complex multi-stage procedures.
Experiments show that CCFM outperforms current state-of-the-art constrained generative models in modeling complex physical systems governed by partial differential equations, molecular docking problems, and physics-informed motion generation, delivering superior feasibility and fidelity.
\end{abstract}

% keywords can be removed

\section{Introduction}
\label{sec:introduction}
Diffusion and flow-matching generative models have achieved remarkable performance across multiple domains, including image generation~\citep{ho2020denoising,lipman2022flow,ma2026drift,ma2026posterior,ma2026transition}, physical system modeling~\citep{cheng2024gradient,utkarsh2025physics}, and other scientific applications~\citep{yim2024improved,miller2024flowmm,christopher2025neuro}. 
Although these models are highly effective at producing content that follows complex distributions, a key challenge is ensuring that generated samples strictly satisfy specific constraints. For example, models of physical systems should satisfy partial differential equations, and generated protein-ligand complexes must satisfy steric and chemical constraints.

To address these challenges, a growing literature has started exploring constrained generative models. However, these approaches present several notable limitations:
{\em (i)} First, gradient-based guidance methods are widely used in scientific applications, yet they cannot guarantee that generated samples satisfy the imposed constraints~\citep{huang2024diffusionpde,song2021solving,shaoul2024multi};
{\em (ii)} Second, repeated projection methods can enforce strict feasibility, but they often degrade sample quality, particularly on complex tasks~\citep{christopher2024constrained,santos2025discrete}; 
{\em (iii)} Finally, a more recent line of work attempts to mitigate these issues via an \emph{extrapolation–correction–interpolation} (ECI) procedure, which estimates ``clean'' samples, projects them onto the feasible set, and interpolates the corrected samples back to the noisy sampling trajectory~\citep{cheng2024gradient,utkarsh2025physics}. However, ECI methods complicate constraint enforcement significantly and rely heavily on accurate clean-sample estimates.

\begin{figure*}[t]
  \centering
    \includegraphics[width=0.75\textwidth]{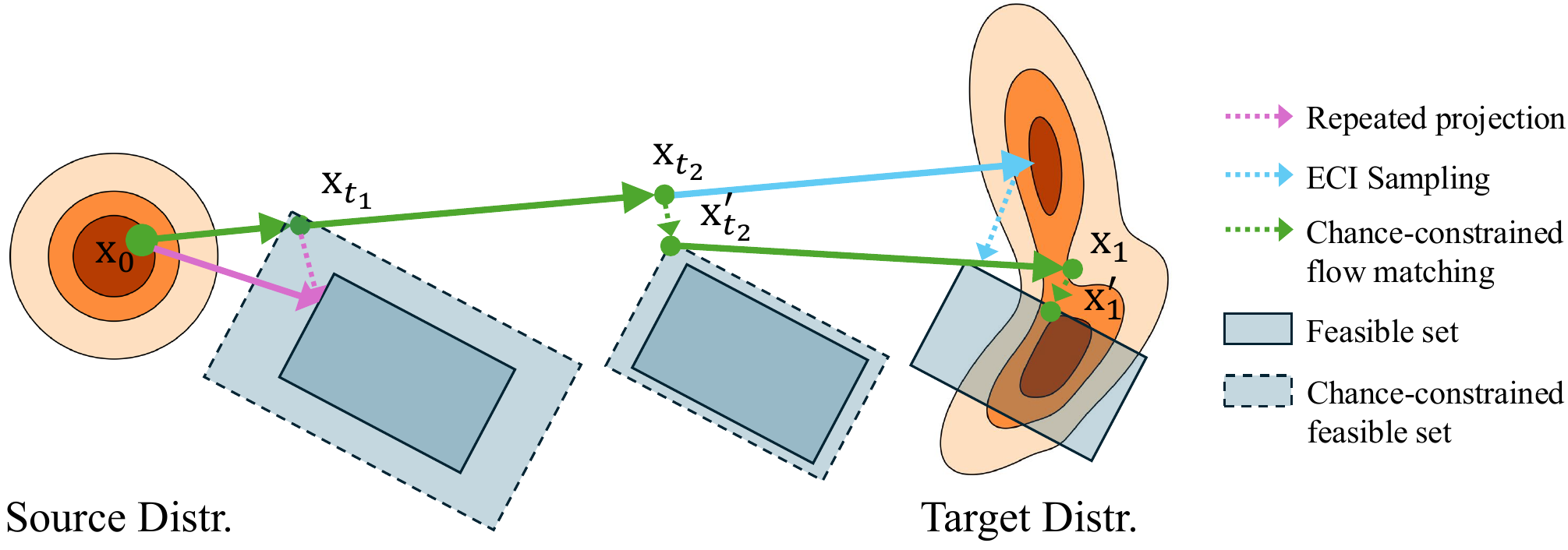}
  \caption{Overview of chance-constrained flow matching (CCFM). Solid arrows follow the sampling path; dashed arrows indicate projection steps. In CCFM, early states remain unchanged ($x_{t_1}$), intermediate states are minimally corrected ($x_{t_2} \rightarrow x_{t_2}'$), and the final state is guaranteed feasible ($x_{1}'$) in a high-probability region. Repeated projection distorts the sampling path. ECI sampling relies on the quality of one-step prediction and may project outside high-probability regions.}
  \label{fig:overview}
\end{figure*}

To address these challenges, this paper introduces \emph{Chance-constrained Flow Matching} (CCFM), a training-free framework for constrained generation that guarantees feasibility while maintaining high-fidelity sample generation. 
Specifically, feasibility means satisfying the prescribed hard constraints, such as conservation laws in PDEs. Fidelity means staying close to the target distribution, such as realistic bound structures in molecular docking.
CCFM starts by revisiting repeated projection, which applies the clean-sample feasible set to highly noisy intermediate states and thus causes overly restrictive corrections early in sampling that can degrade sample fidelity, as illustrated by the pink arrows in Figure~\ref{fig:overview}. 
This motivates us to seek a time-dependent relaxed feasible set defined directly on intermediate states. 
One key observation is that, under the standard optimal transport path, each intermediate state is affine in the initial state and the clean state~\citep{liu2022flow,lipman2024flow}. This affine relation also serves as the premise of ECI-style methods~\citep{cheng2024gradient}. 
This relation suggests pulling back the clean feasible set to the intermediate state using the initial state.
However, once the original sampling trajectory is modified by a projection step, the subsequent trajectory is no longer the deterministic characteristic associated with the originally sampled $x_0$. Therefore, we consider a random variable drawn from the known source distribution (e.g., Gaussian) and formulate constraint enforcement as a chance-constrained program, an optimization framework widely used in operations research to enforce constraints under uncertainty with probabilistic guarantees~\citep{kall1994stochastic}. This can yield a relaxed feasible set on noisy states that enforces constraints with a prescribed probability.
Consequently, CCFM enables direct intermediate-state projection, as shown by the green arrows in Figure~\ref{fig:overview}, while avoiding both the aggressive early corrections of repeated projection and complex procedures of ECI.

\noindent\textbf{Contribution.}
This work makes the following key contributions:
\textbf{(1)} It introduces \emph{Chance-constrained Flow Matching} (CCFM), a novel framework that casts the projection step from a stochastic optimization perspective, enabling concise and effective enforcement of hard constraints while preserving the high-fidelity sample generation.
\textbf{(2)} It establishes a theoretical framework for enforcing hard constraints in the sampling process of generative models from the perspective of chance-constrained programming. The analysis demonstrates that CCFM guarantees feasibility while preserving sample quality.
\textbf{(3)} It demonstrates the effectiveness of CCFM across three scientific domains: molecular docking with structural constraints, physical systems modeling governed by PDEs, and physics-informed motion generation with ODEs. In all domains, CCFM produces samples that satisfy physical principles, achieving new state-of-the-art performance.

\section{Related Work}
Constrained generative modeling modifies either training or sampling to incorporate constraints into generated samples. We group existing methods into three classes.
\textbf{(1)} \emph{Penalty-based training} adds constraint-violation penalties to the training objective~\citep{naderiparizi2025constrained,li2025gauge,baldan2025flow}. This approach is simple, but soft penalties rarely give hard feasibility guarantees and can be brittle in out-of-distribution settings.
\textbf{(2)} \emph{Gradient-based guidance} uses a gradient-based penalty term during sampling to steer samples toward the feasible region~\citep{huang2024diffusionpde, jacobsen2024cocogenphysicallyconsistentconditionedscorebased,shaoul2024multi}. Similar guidance is also used in Bayesian inverse problems with generative priors, where likelihood or data-consistency terms guide posterior sampling~\citep{song2021solving,baldassari2023conditional,kawar2022denoising}. However, it does not guarantee feasibility and often struggles with nonconvex constraints~\citep{liang2025simultaneous}.
\textbf{(3)} \emph{Projection-based methods} enforce feasibility by projecting samples onto the feasible set~\citep{christopher2024constrained,christopher2025neuro}. Repeated projection can enforce hard constraints but usually applies the clean-sample feasible set at every sampling step. However, projecting highly noisy samples at early sampling steps onto the clean-sample feasible manifold can distort the data manifold by altering the learned sampling dynamics~\citep{chung2022diffusion,santos2025discrete}. To reduce this distortion, \citet{cheng2024gradient} and \citet{utkarsh2025physics} develop ECI-style samplers that apply projection only to clean samples via clean-state prediction. However, these samplers remain sensitive to errors in clean-state prediction and interpolation in complex settings. Beyond these approaches, SGEquiDiff incorporates space-group constraints into a crystal generative model through equivariant diffusion
\citep{chang2026space}. However, this approach requires a specialized model, whereas CCFM enforces constraints during sampling from a pretrained model.

In contrast, \emph{CCFM} defines time-dependent relaxed feasible sets directly at noisy intermediate states through a chance-constrained formulation. This perspective enables direct projection during sampling without requiring clean-sample prediction. As a result, CCFM reduces projection-induced distortion of the learned sampling dynamics, simplifies hard-constraint enforcement, and preserves feasibility guarantees while maintaining high-fidelity generation.

%%%%%%%%%%%%%%%%%%%%%%%%
\section{Preliminaries}
%%%%%%%%%%%%%%%%%%%%%%%%
Before introducing CCFM, we review the concepts of generative flow matching and chance-constrained programming. 

\textbf{Flow matching.}
Flow matching (FM) is a generative framework that learns a velocity field \(u_t:[0,1]\times\mathbb{R}^d\to\mathbb{R}^d\), which defines a mapping (\emph{flow}) $\psi_t: [0, 1] \times \mathbb{R}^d \to \mathbb{R}^d$ that transforms samples \(x_0\) from a source distribution \(p\) into samples \(x_1 \;=\; \psi_1(x_0)\) on a target distribution \(q\) by solving the ordinary differential equation (ODE)~\citep{lipman2024flow}:
$\frac{d}{dt}\psi_t(x) = u_t(\psi_t(x))$, with condition $\psi_0(x) = x$.

Here, the variable $t \in [0,1]$ represents the time in FM, with $t=0$ corresponding to the initial state (noise) and $t=1$ to the final state (clean sample). To learn the velocity field, a probability path \(\{p_t\}_{t\in[0,1]}\) is defined to connect \(p_0=p\) and \(p_1=q\). This reduces distribution matching to a supervised learning problem, as the path provides intermediate states corresponding to target velocities for training \(u_t\). The optimal transport (OT) displacement path is widely adopted as it ensures smooth and cost-minimizing interpolations between distributions that yield stable intermediate states for training~\citep{liu2022flow,kerrigan2023functional}. The OT displacement interpolation defines a marginal probability path
between \(p\) and \(q\)~\citep{mccann1997convexity,villani2008optimal,villani2021topics}. Let \(X_0\sim p\) and \(X_1\sim q\) denote
source and target random variables paired along the OT path. For \(t\in[0,1]\), define
\begin{equation}
X_t=(1-t)X_0+tX_1. \label{eq:linearity}
\end{equation}
The induced law of \(X_t\) is the marginal distribution \(p_t\). The corresponding flow-matching loss is
$\mathcal{L}_{\mathrm{FM}}(\theta)
\;=\;
\mathbb{E}_{(X_0,X_1),\, t\sim \mathrm{Unif}[0,1]}
\left\|
u_\theta(X_t,t)-(X_1-X_0)
\right\|^2,
$
where \(u_\theta\) is the learnable velocity field parameterized by \(\theta\), and $\big(X_1-X_0\big)$ denotes the target velocity. After training, the initial state is drawn by \(x_0\sim p\) and propagated through the flow ODE $\frac{d}{dt}x_t \;=\; u_\theta(x_t,t)$, which can be solved with the Euler method~\citep{iserles2009first}:
\begin{equation}
x_{t+\Delta t} \;=\; x_t \;+\; \Delta t\,u_\theta(x_t,t). \label{eq:euler_method}
\end{equation}
For constrained generation, let \(g\) denote constraints defined on clean samples. The goal is to generate a final sample \(x_1\) that satisfies \(g(x_1)\le0\) while preserving fidelity to the target distribution \(q\).

\textbf{Chance-constrained Programming.}
\emph{Chance-constrained programming} (CCP)~\citep{kall1994stochastic} is widely used to handle constraints that include terms whose exact values are unavailable but whose distributions are known. In CCFM, this setting arises during sampling. Once a correction modifies an intermediate state, the subsequent sampling trajectory is no longer determined by the initially sampled $x_0$. To account for this change, CCFM introduces an auxiliary source-prior variable \(Z\sim p_0\) to construct a source-prior-induced relaxed feasible set at time \(t\).

Consider an optimization problem of minimizing \(f(x)\) subject to \(g(x)\le 0\), where \(x\) denotes the decision variable. When the constraints depend on a random input $\xi$, i.e., $g(x,\xi) \le 0$, enforcing feasibility for every realization of $\xi$ is often too conservative or intractable. CCP instead imposes a probabilistic feasibility requirement~\citep{nemirovski2007convex,liang2024joint}:
\begin{equation}
\min_x f(x) \quad \text{s.t.} \quad 
\mathbb{P}\big(g(x,\xi) \le 0\big) \ge \alpha,
\label{eq:ccp}
\end{equation}
where $\alpha \in (0,1)$ is the required probability of constraint satisfaction under the distribution of $\xi$. Thus, any feasible solution of Problem~\ref{eq:ccp} satisfies the constraint with probability at least $\alpha$, rather than for every possible realization of the uncertainty. In practice, CCP has been well studied, and exact reformulations or tractable approximations often render even complex chance constraints computationally tractable~\citep{geng2019data,gelbart2014bayesian,chen2021taylor,nemirovski2007convex}.

\section{Chance-constrained Flow Matching}
This section proposes \emph{Chance-constrained Flow Matching} (CCFM), a novel integration of flow matching with chance-constrained programming to generate high-fidelity samples with feasibility guarantees. We first recast the classical projection method. Then, we introduce chance constraints for flow-matching and derive a tractable reformulation. Finally, we provide a theoretical analysis.

\textbf{Repeated Projection.}
Flow matching evolves an initial noise sample toward a final clean sample via the learned vector field using an explicit Euler update (see \eqref{eq:euler_method}). To enforce feasibility, prior work often applies \emph{repeated projections}~\citep{christopher2024constrained,romer2024diffusion}. Let $\mathcal{C}_1 = \{x : g(x) \leq 0\}$ denote the feasible set for clean samples. At each sampling step, these methods project the intermediate state onto this set,
\(\mathcal{P}_{\mathrm{rp}}(x_t) \coloneqq \operatorname*{arg\,min}_{x' \in \mathcal{C}_1} \|x' - x_t\|_2^2\),
where $x_t'$ is the nearest point in the feasible set $\mathcal{C}_1$. With this operator, the projected sampling rule in \eqref{eq:euler_method} becomes
\begin{equation}
\label{eq:projection}
x_{t+\Delta t} = \mathcal{P}_{\mathrm{rp}}\!\left(x_t + \Delta t\,u_\theta(x_t,t)\right),
\end{equation}
that is, each tentative Euler step is followed by a projection onto $\mathcal{C}_1$. When $\mathcal{C}_1$ is convex, this projected-Euler scheme preserves feasibility at every step~\citep{christopher2025neuro}.

\subsection{Chance-constrained Flow Matching}
\begin{wrapfigure}[11]{r}{0.58\columnwidth}
  \vspace{-1\intextsep}
  \begin{minipage}{\linewidth}
  \small
  \begin{algorithm}[H]
      \DecMargin{10pt}
      \caption{Chance-constrained Flow Matching}
      \label{alg:ccfm}
    \KwIn{Learned vector field $u_\theta$, constraints $g$,
          Euler steps $N$.}
    Sample initial noise $x_0\sim p$\;
    \For{$t\leftarrow 0,1/N,\ldots,(N-1)/N$}{
        $s\leftarrow t+1/N$\;
        $y\leftarrow x_t+u_\theta(x_t,t)/N$\;
        Form the chance constraints at time $s$
        with $\alpha=\phi_n(s)$\;
        Reformulate using Proposition~\ref{prop:tractable_reformulation}\;
        Project $y$ onto the resulting feasible set to obtain $x_s$\;
    }
    \KwRet{$x_1$}\;
  \end{algorithm}
  \end{minipage}
\end{wrapfigure}
Repeated projection applies the clean feasible set $\mathcal{C}_1$ to noisy intermediate states directly. This creates a mismatch: $\mathcal C_1$ constrains clean samples, while $x_t$ still contains initial noise. In early sampling steps, projecting such noisy states onto $\mathcal C_1$ imposes overly tight constraints and can distort the induced distribution~\citep{kynkaanniemi2024applying,rojas2025theory}. To address this issue, \emph{CCFM} reformulates Problem \ref{eq:projection} via a CCP lens, allowing us to adaptively adjust the tightness of the constraints over sampling time $t$. 

Under the OT interpolation identity in \eqref{eq:linearity}, for the fixed source realization \(x_0\) and intermediate realization \(x_t\), the clean terminal state \(x_1\) can be expressed in closed form by rearranging \eqref{eq:linearity} as \(x_1=t^{-1}x_t-(1-t)t^{-1}x_0\). This gives the
pathwise pullback of the clean feasible set to time \(t\):
\begin{equation}
\mathcal C_t(x_0)
\coloneqq
\left\{
  x_t\in\mathbb R^d:
g\!\left(t^{-1}x_t-(1-t)t^{-1}x_0\right)\le 0
\right\}.
\label{eq:pathwise_pullback_set}
\end{equation}
Thus, if \(x_t\) lies on the unmodified sampling trajectory, then \(x_t\in\mathcal C_t(x_0)\) if and only if its corresponding clean state is feasible. However, once a projection modifies an intermediate state, the corrected state no longer lies on the original sampling trajectories associated with the initially sampled \(x_0\). This motivates us to replace the fixed-source pullback \(\mathcal C_t(x_0)\) with a source-prior-induced surrogate. Let \(Z\sim p_0\) be an auxiliary source variable drawn from the known source distribution. For a satisfaction level \(\alpha\in(0,1]\), define the chance-constrained feasible set:
\begin{equation}
  \widehat{\mathcal C}_t(\alpha)
  \coloneqq
  \left\{
  x_t\in\mathbb R^d:
  \mathbb P_{Z\sim p_0}
  \left[
  g\!\left(t^{-1}x_t-(1-t)t^{-1}Z\right)\le0
  \right]
  \ge \alpha
  \right\}.
  \label{eq:cc_feasible_set}
\end{equation}
The probability in \eqref{eq:cc_feasible_set} is \emph{not} posterior uncertainty over the historical $x_0$ that produced $x_t$. 
%along a realized deterministic trajectory. 
Rather, \(Z\) is an auxiliary source-prior variable used to define a time-dependent feasible set on intermediate states. The CCFM projection is then
\begin{equation}
\mathcal P_{\mathrm{cc}}^{t}(x) \coloneqq \argmin_{y\in\mathbb R^d}\; \|y-x\|_2^2 \quad
\mathrm{s.t.}\; \mathbb P_{Z\sim p_0}
\left[
g\!\left(t^{-1}y-(1-t)t^{-1}Z\right)\le0
\right]
\ge \alpha.
\label{eq:cc_projection_operator}
\end{equation}
In practice, we set \(\alpha=\phi_n(t)=\frac{1}{2}t^n\) with \(n>0\). This monotone continuation schedule controls how quickly CCFM moves from loose constraints early in sampling to tighter constraints near the terminal time. Smaller \(n\) yields larger \(\alpha\) at intermediate times and tightens the constraints earlier, whereas larger \(n\) delays tightening and keeps the sampler closer to the unconstrained flow for longer. Then, the update rule for the CCFM sampling process is given by
\begin{equation}
x_{t+\Delta t} = \mathcal P_{\mathrm{cc}}^{t+\Delta t}\left(x_t + \Delta t\,u_\theta(x_t,t)\right).
\end{equation}
The detailed procedure of CCFM is summarized in Algorithm~\ref{alg:ccfm}.

Compared with repeated projection~\citep{christopher2024constrained}, \(\widehat{\mathcal C}_t(\alpha)\) relaxes the feasible region early in sampling and tightens it as \(t\) approaches \(1\). At \(t=0\), the sample is pure source noise and no projection is applied; in early sampling, the chance-constrained set is relaxed by the scheduler. As \(t\to1\), the source-prior uncertainty vanishes, and the chance-constrained feasible set reduces to the deterministic clean feasible set $\widehat{\mathcal C}_1(\alpha)=\mathcal C_1,\ \forall \alpha>0$. Compared with ECI-style samplers~\citep{cheng2024gradient,utkarsh2025physics}, CCFM also avoids the extrapolate--project--interpolate
pipeline: it does not require clean-sample prediction, correction of that
prediction, and interpolation back to the noisy state, thereby reducing the risk
of error accumulation across sampling steps.

\begin{remark}[Scope of the OT-path assumption]
CCFM is derived under the OT displacement path in \eqref{eq:linearity}, which is a standard geometric assumption in flow matching and has been adopted in recent constrained flow-matching samplers~\citep{cheng2024gradient,utkarsh2025physics}.
In practice, learned velocities and numerical discretization may introduce slight curvature around the ideal straight-line characteristics. Nevertheless, recent theoretical results show that Lipschitz neural networks can approximate the ground-truth velocity field within \((\sqrt d+1)\epsilon\) for any \(\epsilon>0\), where \(d\) is the data dimension~\citep{zhouerror}. Discretization error also decays as the step size decreases~\citep{gao2024convergence}. The empirical success of straight-line characteristics in prior work~\citep{liu2022flow,liu2023instaflow} and in our experiments further supports this assumption.
Appendix~\ref{app:marginal_trajectory_curvature} evaluates 1,000 unconstrained Reaction--Diffusion trajectories. The mean normalized deviation from straight interpolation between the initial noise and clean sample is less than $1\%$.
\end{remark}

\subsection{Tractable Reformulation for CCFM}
The projection in \eqref{eq:cc_projection_operator} involves a probabilistic
constraint. Although general chance-constrained programs are difficult, the
source distribution in flow matching is typically Gaussian, which yields simple
deterministic reformulations.

\begin{proposition}[Gaussian affine chance constraint]
  \label{prop:tractable_reformulation}
  Fix \(t\in(0,1]\) and \(\alpha\in(0,1)\). Let \(a,x_t\in\mathbb R^d\), \(b\in\mathbb R\), and \(\xi_t=(1-t)t^{-1}Z\) with \(Z\sim\mathcal N(0,I_d)\). 
  Define \(\sigma_t=(1-t)/t\), and let \(h:\mathbb R^d\to\mathbb R\) be any deterministic function independent of $\xi_t$. Then
  \[
  \mathbb P_{\xi_t}\!\left(h(x_t)+a^\top \xi_t \le b\right)\ge \alpha
  \quad\Longleftrightarrow\quad
  h(x_t)\le b-\sigma_t\|a\|_2\Phi^{-1}(\alpha),
  \]
  where \(\Phi\) is the standard normal CDF.
\end{proposition}
  
Proposition~\ref{prop:tractable_reformulation} does not require $h$ to be convex in $x_t$: for any deterministic $h$ independent of $\xi_t$, affine-in-noise chance constraints admit the exact deterministic reformulation. CCFM therefore can handle these chance constraints with the same projection machinery used for deterministic constraints. Since $z_{\alpha}=\Phi^{-1}(\alpha)$ is monotone in $\alpha$ and the standard deviation $\sigma_t=\tfrac{1-t}{t}$ decreases as $t\!\to\!1$, the constraint bound tightens progressively with the sampling time $t$. Appendix~\ref{app:missing_proof} provides the proof.

For nonlinear dependence on the uncertainty, we use a local first-order stochastic surrogate, a standard approximation in nonlinear and stochastic optimization~\citep{chen2021taylor,nemirovski2007convex,nocedal2006numerical}. Given a reference value \(\bar\xi_t\) of the random variable \(\xi_t\), e.g., \(\bar\xi_t=\mathbb E[\xi_t]\), we linearize
\[
g(y,\xi_t)
\approx
g(y,\bar\xi_t)
+
\nabla_\xi g(y,\bar\xi_t)^\top(\xi_t-\bar\xi_t).
\]
Since this approximation is affine in the Gaussian random variable \(\xi_t\), Proposition~\ref{prop:tractable_reformulation} applies and yields a deterministic reformulation of the corresponding chance constraint. 

Following the existing constrained flow matching~\citep{utkarsh2025physics}, we solve the resulting deterministic projection problem with an active-set Gauss--Newton update. Let \(\mathcal A\) be the active constraints, \(r_{\mathcal A}\) denotes the residual vector, and \(J_{\mathcal A}=\partial r_{\mathcal A}/\partial y\). We update
\[
y
\leftarrow
y
-
J_{\mathcal A}^{\top}
\left(
J_{\mathcal A}J_{\mathcal A}^{\top}
+
\lambda I
\right)^{-1}
r_{\mathcal A},
\]
where \(\lambda>0\) is a small regularizer.

\subsection{Theoretical Analysis for CCFM}
Having established a tractable chance-constrained projection for any time $t \in (0,1]$, the next result shows that this procedure guarantees the feasibility of generated samples. 
\begin{corollary}[Feasibility Guarantee]\label{cor:cc_feas}
  For any satisfaction level \(\alpha\in(0,1]\), the chance-constrained
  projection in \eqref{eq:cc_projection_operator} satisfies
  \[
  \mathcal P_{\mathrm{cc}}^{1}(x_1)
  =
  \argmin_{y\in \mathcal C_1}\|y-x_1\|_2^2,
  \qquad
  \mathcal C_1=\{y\in\mathbb R^d: g(y)\le 0\}.
  \]
\end{corollary}
Proposition~\ref{prop:tractable_reformulation} and Corollary~\ref{cor:cc_feas} imply that the operator is probabilistic for $t<1$ and becomes deterministic at $t=1$, ensuring feasibility. However, it is worth noting that this exact feasibility guarantee applies to the constraint set $\mathcal C_1$. In practice, for computational efficiency, some highly complex constraints, such as the molecular docking constraints in our experiments, are implemented using standard approximations of the original constraint set. Despite this, the results show that our method remains robust in realistic constrained-generation settings.

Beyond feasibility, we quantify how much CCFM perturbs the unconstrained learned dynamics.

\begin{theorem}[Correction-controlled perturbation]
\label{thm:correction_controlled_perturbation}
Let \(h=1/N\), \(t_k=kh\), and \(F_k(x)=x+h u_\theta(x,t_k)\). Assume \(u_\theta(\cdot,t_k)\) is \(L\)-Lipschitz uniformly in \(k\), and set \(\gamma=1+hL\). Consider unconstrained sampler $Y_{k+1}=F_k(Y_k)$ and constrained sampler $X_{k+1}=P_k(F_k(X_k))=F_k(X_k)+e_k$ with \(X_0=Y_0\), where \(P_k\) is a projection operator. Then
\begin{equation}
  \|X_N-Y_N\|_2
\le
\sum_{k=0}^{N-1}\gamma^{N-1-k}\|e_k\|_2 .
\label{eq:pathwise_perturbation_bound}
\end{equation}
If \(X_0=Y_0\sim p_0\), the final distributions have finite second moments, and \(\mathbb E\|e_k\|_2^2<\infty\), then the $L^2$-Wasserstein distance
\begin{equation}
  \textstyle
W_2(\mu_N,\mu_N^{\mathrm{free}})
\le
\sum_{k=0}^{N-1}\gamma^{N-1-k}
\left(\mathbb E\|e_k\|_2^2\right)^{1/2},
\label{eq:w2_perturbation_bound}
\end{equation}
where \(\mu_N\) and \(\mu_N^{\mathrm{free}}\) are the distributions of \(X_N\) and \(Y_N\), respectively.
\end{theorem}
Theorem~\ref{thm:correction_controlled_perturbation} shows that distributional shift from the learned unconstrained dynamics is controlled by the correction magnitudes. Since the weights \(\gamma^{N-1-k}\) are non-increasing in \(k\), early corrections are propagated through more subsequent Euler steps and can induce larger trajectory distortion. 
This supports the design of CCFM, whose relaxed constraint sets avoid overly restrictive projections at noisy intermediate states and tighten only as \(t\to1\).
Furthermore, for distributions with finite second moments, exact Euclidean projection onto a nonempty closed convex set cannot increase the $W_2$ distance to any target distribution supported on that set~\citep{blanke2026strictly}.

\section{Experiments}
\label{sec:experiments}
This section first considers a two-dimensional example with a known constrained target distribution, allowing direct evaluation of distributional fidelity through $W_2$ distance. We then evaluate CCFM on three scientific tasks: \textbf{Molecular Docking}, \textbf{PDE Solution Generation}, and \textbf{Physics-informed Motion Generation}.
Further details on the experimental settings are provided in Appendix~\ref{app:implementation}. Additional results are presented in Appendix~\ref{app:additional_results}.

\subsection{Experimental Settings}
\textbf{Two-Dimensional Example.}
We consider the initial noise $X_0\sim\mathcal N(0,I_2)$ and a Gaussian mixture target distribution $q=0.3\,\mathcal N((0,-1),0.7^2 I_2)+0.7\,\mathcal N((2,1),0.7^2 I_2)$.
For a clean sample $z=(z_1,z_2)\in\mathbb R^2$, the feasible set is $\mathcal C=\{z:z_1=-1\}\cup\{z:z_1=1\}$.
Specifically, these two branches have unequal probability mass under the ground-truth constrained distribution $q_C$. Recovering this mass allocation tests whether a sampler preserves distributional structure while satisfying constraints.

Comparisons are made against \emph{Repeated Projection (RP)}, \emph{ECI}, and \emph{Post-only Projection}, which projects only the final sample.
Evaluation uses $W_2$ distance to $q_C$ and branch-mass error, which measures discrepancies in the probabilities assigned to the two feasible branches.

\textbf{Molecular Docking.}
The task is to predict ligand poses within protein pockets while respecting geometric and physical constraints. Experiments are performed on the PDBBind docking benchmark~\citep{liu2017forging}, which provides a standard testbed for evaluating protein–ligand docking methods. All methods are evaluated on the same 363 test complexes.
Reported success rates are computed over these complexes.
{Comparisons are made against a \emph{repeated projection (RP)} baseline} and \emph{FlexDock}~\citep{corso2025composing}, the current state-of-the-art docking framework, which enforces structural constraints by embedding penalty terms during training and further ensuring validity at inference through rejection sampling.

The evaluation focuses on: {\bf (1)} \emph{Fidelity}, which is quantified by ligand RMSD (\underline{L-RMSD}) (fractions below 1\AA and 2\AA) and all-atom RMSD (\underline{A-RMSD}) (fractions below 2\AA). {\bf (2)} \emph{Feasibility}, which is assessed by PoseBusters validity (\underline{PB Valid})~\citep{buttenschoen2024posebusters} and its conjunction with RMSD thresholds, measuring the physical plausibility of predicted complexes. 

\textbf{PDE Solution Generation.}
Next, following~\citep{utkarsh2025physics}, we consider a generative task on two representative PDE systems: the 1-D \emph{reaction–diffusion equation}, subject to linear initial conditions (IC) and nonlinear global mass conservation laws (CL), and the 2-D \emph{Navier–Stokes equation}, subject to linear initial conditions (IC) and linear conservation laws (CL). The reaction–diffusion problem is discretized on a $(128 \times 100)$ space–time grid, while the Navier–Stokes problem is posed on a $64 \times 64 \times 50$ grid. For each PDE, all methods are evaluated on the same 1,000 test trajectories.

Comparisons are conducted against several representative baselines:
(a) \emph{Physics-Constrained Flow Matching (PCFM)}\citep{utkarsh2025physics}, which is the current state of the art; it extends ECI sampling~\citep{cheng2024gradient} to nonlinear constraints and reduces to ECI in the linear setting;
(b) \emph{DiffusionPDE}~\citep{huang2024diffusionpde}, which enforces physical priors through penalty terms during iterative sampling;
(c) \emph{Functional Flow Matching (FFM)}~\citep{kerrigan2023functional}, a pretrained unconstrained FM; 
and (d) {\emph{Repeated Projection (RP)} approach, which naively projects the state onto the feasible set at each sampling step.}

Performance is evaluated along two complementary dimensions: 
{\bf (1)} \emph{Feasibility} is quantified by the constraint violation (\underline{CV}), which measures deviations from the imposed boundary, initial, and conservation constraints. 
{\bf (2)} \emph{Fidelity} is evaluated using the pointwise mean squared error of the mean (\underline{MMSE}) and the pointwise mean squared error of the standard deviation (\underline{SMSE}), capturing distributional alignment in terms of first- and second-order statistics. 

{\textbf{Physics-informed motion.}
Following the setting introduced in \citep{christopher2024constrained}, this task examines whether a generative model can synthesize video sequences that respect physical rules. The objective is to generate multi-frame trajectories of a moving object whose positions must satisfy discrete ordinary differential equation (ODE) updates coupling the position and velocity across time steps. Training data include only Earth-gravity motion, while test cases with Moon gravity provide an out-of-distribution evaluation. Comparisons are conducted against:
(a) an \emph{unconstrained} flow model, which follows the standard sampling process without any constraint enforcement;
(b) a \emph{repeated projection (RP)} baseline, which projects each intermediate sample onto the constraint set.}

Evaluation focuses on two aspects:
(1) Feasibility, measured by the violation of the ODE constraints across frames;
(2) Fidelity, assessed by the Fréchet Inception Distance (FID), which reflects the visual quality of generated sequences under Moon gravity.

\subsection{Performance Evaluation for the Two-Dimensional Example}
\label{sec:toy_experiment}
\begin{wrapfigure}[8]{r}{0.56\columnwidth}
  \vspace{-12pt}
  \centering
  \small
  \captionsetup{type=table}
  \caption{Two-dimensional example results over five runs.}
  \label{tab:toy_distribution}
  \resizebox{\linewidth}{!}{%
  \begin{tabular}{lcc}
    \toprule
    Method & $W_2$ distance ($\downarrow$) & Branch-mass Error ($\downarrow$) \\
    \midrule
    RP & 1.026 & 0.231 \\
    ECI & 2.695 & 0.152 \\
    Post-only projection & 0.671 & 0.080 \\
    CCFM (ours) & \textbf{0.412} & \textbf{0.0248} \\
    \bottomrule
  \end{tabular}
  }
\end{wrapfigure}

Empirical $W_2$ measures the distance to the ground-truth constrained distribution, while branch-mass error measures the mismatch in probability assigned to each feasible branch. Table~\ref{tab:toy_distribution} shows that CCFM achieves the lowest error on both metrics. Repeated projection and ECI commit to a branch too early, whereas post-only projection retains biased branch weights from unconstrained sampling. By delaying branch commitment, CCFM more accurately recovers the constrained distribution. Its improvement over post-only projection shows that terminal feasibility alone does not ensure distributional fidelity.

\begin{table}[t]
\caption{PDBBind docking results. 
\#S denotes the number of samples generated per complex, with the best pose selected for evaluation; \#Stp is the number of inference steps. 
L-RMSD/A-RMSD report fractions of complexes below the given thresholds, PB Valid is the proportion passing PoseBusters checks, with +L and +A further requiring L-RMSD$<2$\AA{} or A-RMSD$<2$\AA{}, respectively.}
\centering
\small
\renewcommand{\arraystretch}{0.8}

\sisetup{detect-weight=true}
\resizebox{0.85\linewidth}{!}{%
\begin{tabular}{ccc
  *{2}{c} % lig
  c       % aa
  *{3}{c} % pb
  c       % running time
}
\toprule

\multirow[c]{2}{*}{\textbf{\#S}} & 
\multirow[c]{2}{*}{\textbf{\#Stp}} & 
\multirow[c]{2}{*}{\textbf{Methods}} & 
\multicolumn{2}{c}{\textbf{L-RMSD (\%)}} &
\multicolumn{1}{c}{\textbf{A-RMSD (\%)}} &
\multicolumn{3}{c}{\textbf{PB Valid (\%)}} &
\textbf{Running} \\ 
% --- ---
\cmidrule(lr){4-5}\cmidrule(lr){6-6}\cmidrule(lr){7-9}
% --- ---
& 
& 
& 
\textbf{$< 1 \text{\AA}$} & 
\textbf{$< 2\,\text{\AA}$} & 
\textbf{$< 2\,\text{\AA}$} &
\textbf{All} & 
\textbf{\textbf{+L $<2$\AA}} & 
\textbf{\textbf{+A $<2$\AA}} & 
\textbf{Time (s)} \\% --- ---
\midrule
% --- ---
\multirow[c]{3}{*}{1} & \multirow[c]{3}{*}{2} & CCFM
& \shad {\bfseries 8.19} & \shad {\bfseries 32.20} & \shad {\bfseries 74.01} & \shad {\bfseries 47.74} & \shad {\bfseries 22.03} & \shad {\bfseries 35.88}
& \shad {0.64} \\
& & FlexDock
& 7.06 & 31.36 & 73.73 & 20.34 & 9.32 & 14.97 & {0.52} \\
& & RP
& 6.78 & 31.07 & 73.73 & 19.49 & 9.32 & 14.69 & 0.73 \\

\addlinespace
\multirow[c]{3}{*}{10} & \multirow[c]{3}{*}{2} & CCFM
& \shad {\bfseries 14.69} & \shad 37.01 & \shad {\bfseries 74.86} & \shad {\bfseries 67.51} & \shad {\bfseries 31.36} & \shad {\bfseries 51.98}
& \shad {4.34} \\
& & FlexDock
& 12.99 & {\bfseries 38.14} & 73.16 & 43.50 & 21.19 & 33.05 & {4.20} \\
& & RP
& 13.56 & 36.72 & 73.16 & 44.07 & 21.19 & 34.18 & 4.39 \\

\addlinespace
\multirow[c]{3}{*}{20} & \multirow[c]{3}{*}{2} & CCFM
& \shad 12.43 & \shad {\bfseries 39.27} & \shad {\bfseries 74.01} & \shad {\bfseries 67.80} & \shad {\bfseries 33.62} & \shad {\bfseries 51.41}
& \shad {8.51} \\
& & FlexDock
& {\bfseries 12.71} & {\bfseries 39.27} & 72.88 & 44.63 & 23.16 & 33.05 & {7.82} \\
& & RP
& 11.58 & 37.85 & 72.88 & 49.72 & 25.42 & 37.85 & 9.21 \\

\addlinespace
\multirow[c]{3}{*}{1} & \multirow[c]{3}{*}{10} & CCFM
& \shad {\bfseries 7.91} & \shad {\bfseries 32.20} & \shad 72.32 & \shad {\bfseries 78.81} & \shad {\bfseries 30.79} & \shad {\bfseries 58.19}
& \shad {0.79} \\
& & FlexDock
& 7.63 & 31.64 & {\bfseries 72.88} & 57.91 & 24.29 & 40.68 & {0.57} \\
& & RP
& {\bfseries 7.91} & 31.92 & {\bfseries 72.88} & 56.21 & 27.40 & 42.66 & 0.92 \\

\addlinespace
\multirow[c]{3}{*}{1} & \multirow[c]{3}{*}{20} & CCFM
& \shad {\bfseries 8.19} & \shad {\bfseries 31.92} & \shad 72.32 & \shad {\bfseries 76.84} & \shad {\bfseries 29.94} & \shad {\bfseries 57.06}
& \shad {0.88}  \\
& & FlexDock
& 7.63 & 31.07 & 72.60 & 60.73 & 25.99 & 43.50 & {0.62} \\
& & RP
& 7.91 & 31.36 & {\bfseries 72.88} & 50.85 & 26.55 & 40.68 & 1.10 \\

\addlinespace
\multirow[c]{2}{*}{10} & \multirow[c]{2}{*}{10} & CCFM
& \shad {\bfseries 15.54} & \shad 37.85 & \shad 72.03 & \shad {\bfseries 85.03} & \shad {\bfseries 35.88} & \shad {\bfseries 62.15}
& \shad {4.57} \\
& & FlexDock
& {\bfseries 15.54} & {\bfseries 38.14} & 72.60 & 80.51 & 35.31 & 57.06 & {3.99} \\
& & RP
& 13.84 & {\bfseries 38.14} & {\bfseries 72.88} & 76.55 & 34.75 & 56.78 & 5.15 \\

\bottomrule
\end{tabular}%
}
\label{tab:PDBBind docking results}
\end{table}

\subsection{Performance Evaluation for Molecular Docking}
Table~\ref{tab:PDBBind docking results} reports docking results on the PDBBind docking benchmark under varying numbers of generated samples and inference steps. Across all configurations, \emph{CCFM} consistently outperforms the baselines \emph{FlexDock} and \emph{RP}, \emph{with particularly large gains when the computational budget is limited}. Remarkably, with only one sample and two inference steps, CCFM achieves comparable L-RMSD and A-RMSD values, confirming that high-fidelity pose generation is preserved. At the same time, CCFM raises PB Valid to $47.74\%$, representing a \textbf{2.3$\times$} improvement over FlexDock ($20.34\%$) and \textbf{2.4$\times$} over RP ($19.49\%$). In contrast to RP, CCFM avoids the performance degradation arising from the application of strict constraints to noisy intermediate states. With this setting, CCFM even surpasses FlexDock with twenty samples, yielding a \textbf{20$\times$} improvement in sample–step efficiency (measured as samples~$\times$~steps) and a \textbf{12$\times$} runtime speedup (7.82s vs.~0.64s). On stricter metrics that jointly account for fidelity and feasibility (PB Valid+L and PB Valid+A), CCFM also achieves more than $2\times$ higher pass rates than both baselines, again matching FlexDock at 20 samples while using far fewer samples and computational time. Increasing inference steps to 10 or 20 further amplifies this advantage: at 10 steps, CCFM improves PB Valid by $20.9\%$, PB Valid+L by $6.5\%$, and PB Valid+A by $17.5\%$, while simultaneously yielding higher fidelity and feasibility than the baseline. Comparable trends are observed when both the sample count and inference steps are set to ten. Additional comparisons across more metrics and experimental settings are provided in Appendix~\ref{app:molecular_docking}.

Figure~\ref{fig:protein_sample} compares the binding structures generated by our method (CCFM) and FlexDock on the same protein–ligand pair. In molecular docking, protein–ligand atoms must stay at least the sum of their van der Waals radii minus 0.75 \AA~\cite{corso2025composing}. In this case, the van der Waals radii of the carbon (1.70 \AA) and oxygen (1.52 \AA) atoms yield a lower bound of 2.47 \AA. As shown in Figure~\ref{fig:protein_sample_ccfm}, CCFM produces a feasible structure with a distance of 3.8 \AA \, , whereas FlexDock in Figure~\ref{fig:protein_sample_flexdock} generates an infeasible structure with a distance of 2.4 \AA.

\begin{figure}[t]
  \centering
  \begin{subfigure}[t]{0.48\textwidth}
    \centering
    \includegraphics[width=\textwidth]{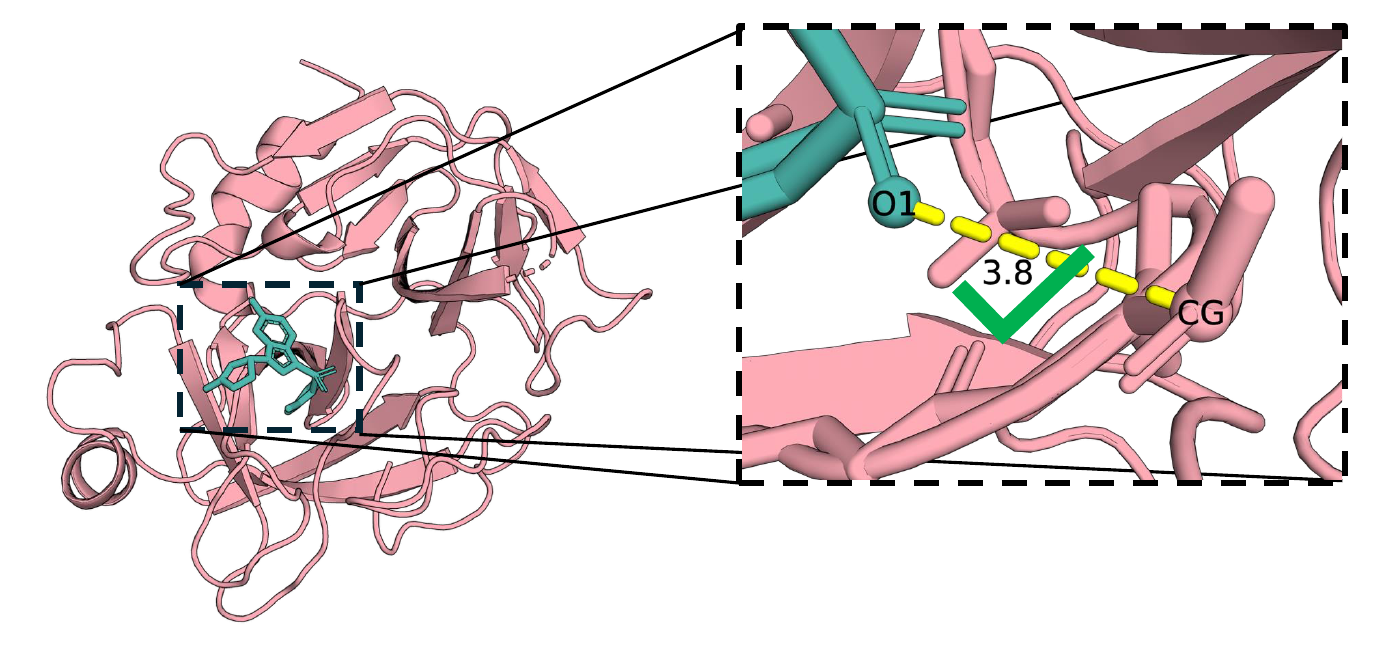}
    \vspace{-22pt}\caption{CCFM (ours)}
    \label{fig:protein_sample_ccfm}
  \end{subfigure}
  \hfill
  \begin{subfigure}[t]{0.48\textwidth}
    \centering
    \includegraphics[width=\textwidth]{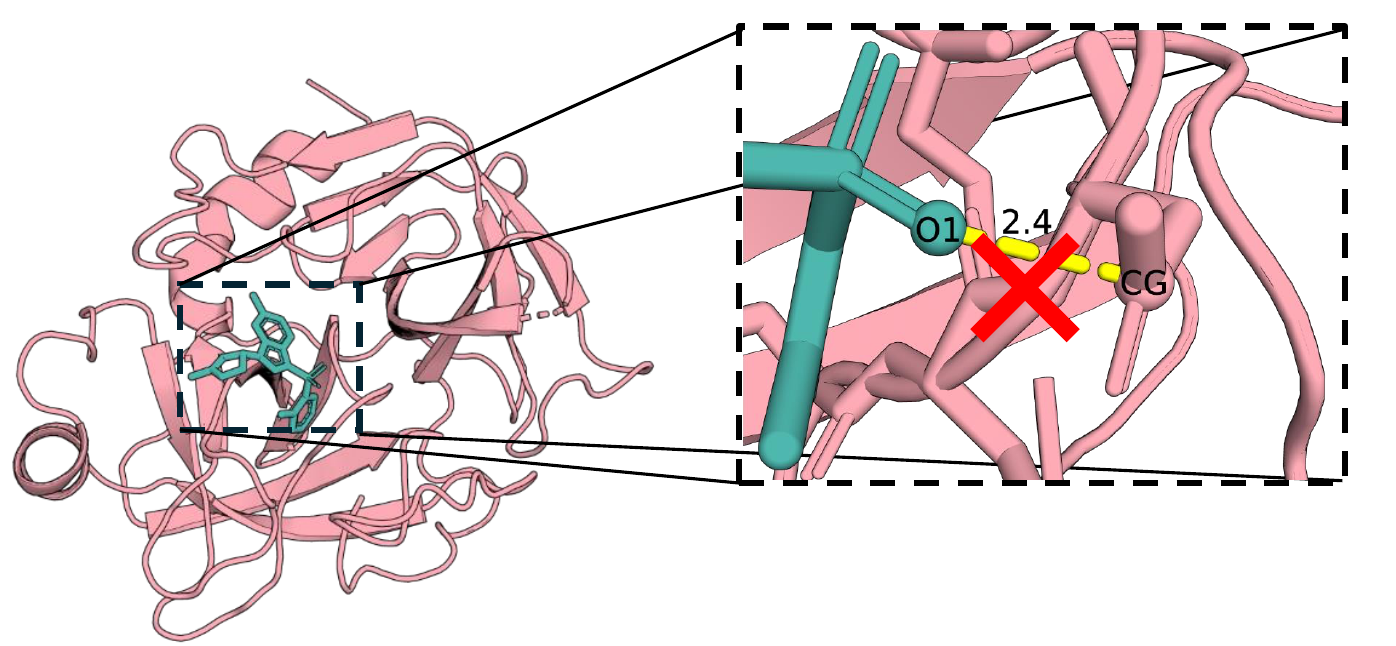}
    \vspace{-22pt}\caption{FlexDock}
    \label{fig:protein_sample_flexdock}
  \end{subfigure}
  \vspace{-4pt}
  \caption{\small Structures generated by CCFM (ours) and FlexDock. CCFM enforces the ligand–protein distance constraint, while FlexDock violates it.}
  \label{fig:protein_sample}
  \vspace{-12pt}
\end{figure}

\subsection{Performance Evaluation for PDE}\label{sec:pde_results}
\begin{table}[t]
    \centering
    \small
  \caption{Comparison of model performance in solution generation tasks for 1D Reaction--Diffusion equation and 2D
  Navier--Stokes equation across metrics.
  A lower value indicates better performance. Bold indicates the best (smallest) value in
  each row.}
 \label{tab:ccfm-unified}
    \renewcommand{\arraystretch}{0.8}
  \resizebox{0.7\linewidth}{!}{%
  \begin{tabular}{
    ll
    c % CCFM
    c % PCFM
    c % DiffusionPDE
    c % FFM
    c % RP
  }
    \toprule
    \textbf{Problem} & \textbf{Metric} & \textbf{CCFM} & \textbf{PCFM} & \textbf{DPDE} & \textbf{FFM} & \textbf{RP} \\
    \midrule
    \multirow{4}{*}{RD}
      & MMSE $\times 10^{-2}$ ($\downarrow$)   & \multicolumn{1}{c}{\bfseries 3.3} & \multicolumn{1}{c}{3.6} & \multicolumn{1}{c}{5.2} & \multicolumn{1}{c}{5.3} & \multicolumn{1}{c}{8.1} \\
      & SMSE $\times 10^{-2}$ ($\downarrow$)   & \multicolumn{1}{c}{\bfseries 2.7} & \multicolumn{1}{c}{2.9} & \multicolumn{1}{c}{3.8} & \multicolumn{1}{c}{3.8} & \multicolumn{1}{c}{8.5} \\
      & CV (IC)  $\times 10^{-2}$ ($\downarrow$) & \multicolumn{1}{c}{\bfseries 0}             & \multicolumn{1}{c}{\bfseries 0} & \multicolumn{1}{c}{8.0} & \multicolumn{1}{c}{8.2} & \multicolumn{1}{c}{\bfseries 0} \\
      & CV (CL)  $\times 10^{-15}$ ($\downarrow$) & \multicolumn{1}{c}{\bfseries 9.5}             & \multicolumn{1}{c}{\bfseries 9.5} & \multicolumn{1}{c}{3.9$\times 10^{10}$} & \multicolumn{1}{c}{3.9$\times 10^{10}$} & \multicolumn{1}{c}{10.8} \\
    \midrule
    \multirow{4}{*}{NS}
      & MMSE $\times 10^{-1}$ ($\downarrow$)        & {\bfseries 0.8} & 1.7 & 1.4 & 1.5 & 1.0 \\
      & SMSE $\times 10^{-1}$ ($\downarrow$)        & {\bfseries 0.8} & 1.6 & 8.8 & 9.0 & 0.9 \\
      & CV (IC) ($\downarrow$)                       & {\bfseries 0} & {\bfseries 0} & 0.12 & 0.13 & {\bfseries 0} \\
      & CV (CL)  $\times 10^{-15}$ ($\downarrow$)    & \multicolumn{1}{c}{3.9} & \multicolumn{1}{c}{\bfseries 3.4} & \multicolumn{1}{c}{1.2 $\times 10^{10}$} & \multicolumn{1}{c}{1.2 $\times 10^{9}$} & 4.1 \\
    \bottomrule
  \end{tabular}%
  }
 % \vspace{-12pt}
\end{table}

We next focus on generating PDE solutions. 
Table~\ref{tab:ccfm-unified} summarizes the performance of the different models tested. For the Reaction--Diffusion, the unconstrained \emph{FFM} baseline yields relatively poor accuracy (MMSE $5.3\times 10^{-2}$), while DiffusionPDE improves marginally (to $5.2\times 10^{-2}$) by incorporating penalty terms during sampling that steer samples towards the prescribed initial condition and enforce the conservation law. However, this approach fails to handle non-linear constraints and does not guarantee feasibility.
The \emph{RP} baseline successfully enforces the initial condition constraint, yet this comes at the cost of solution fidelity, yielding the worst accuracy among all methods. 
\emph{PCFM} achieves a substantial accuracy gain (MMSE reduced to $3.6\times 10^{-2}$) while simultaneously ensuring minimal constraint violations. Nevertheless, this procedure incurs a computational burden (see Appendix~\ref{app:additional_results}) due to the projections at every sampling step, resulting in a runtime that is approximately $28\times$ higher than the unconstrained FFM baseline. 
\emph{CCFM} further advances performance (MMSE $3.3\times 10^{-2}$, SMSE $2.7\times 10^{-2}$) while providing \textbf{zero constraint violations}.
Moreover, \emph{CCFM} is about 30\% faster than PCFM (see Appendix~\ref{app:additional_results}). Figure~\ref{fig:rd_mse} further illustrates its MMSE improvement across physical time, where darker regions indicate improved temporal stability.

On the Navier--Stokes task, the increased dimensionality and the out-of-distribution test setting exacerbate the limitations of the baseline methods.
\emph{PCFM} yields higher MMSE than DiffusionPDE despite preserving feasibility. Its reliance on one-step predictions from intermediate states leads to inaccuracies in high dimensions, and projection–interpolation steps amplify these errors, reducing accuracy. \emph{RP} demonstrates strong performance, outperforming both PCFM and the unconstrained baselines. However, it still lags behind \emph{CCFM}.
By making only small corrections toward the feasible region and avoiding unstable one-step extrapolations, \emph{CCFM} achieves the best overall performance (MMSE $8\times 10^{-2}$, SMSE $8\times 10^{-2}$) while ensuring feasibility; Figure~\ref{fig:ns_mse} shows that it suppresses error accumulation and aliasing under nonlinear dynamics.

\begin{figure}[t]
  \centering
  \begin{subfigure}[t]{0.23\textwidth}
    \centering
    \includegraphics[width=\textwidth]{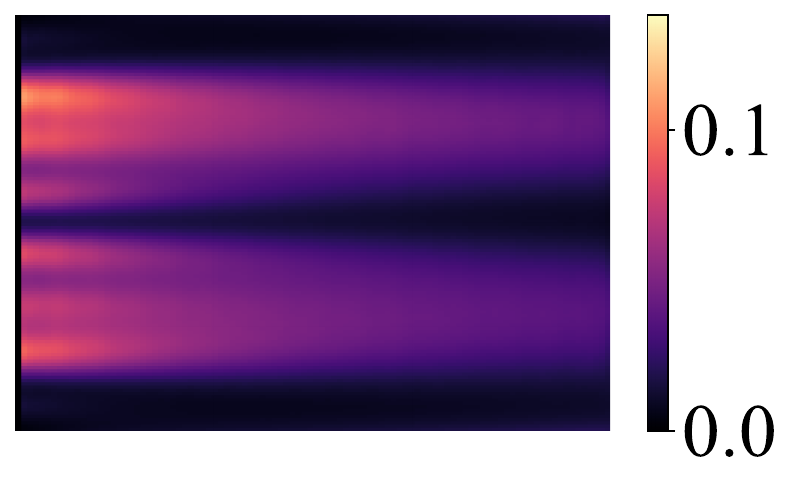}
    \vspace{-18pt}\caption{CCFM (ours)}
  \end{subfigure}
  \hfill
  \begin{subfigure}[t]{0.23\textwidth}
    \centering
    \includegraphics[width=\textwidth]{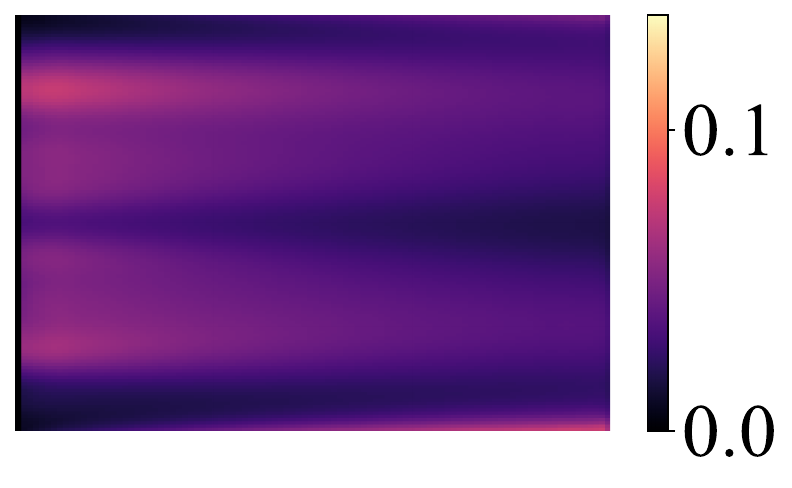}
    \vspace{-18pt}\caption{PCFM}
  \end{subfigure}
  \hfill
  \begin{subfigure}[t]{0.23\textwidth}
    \centering
    \includegraphics[width=\textwidth]{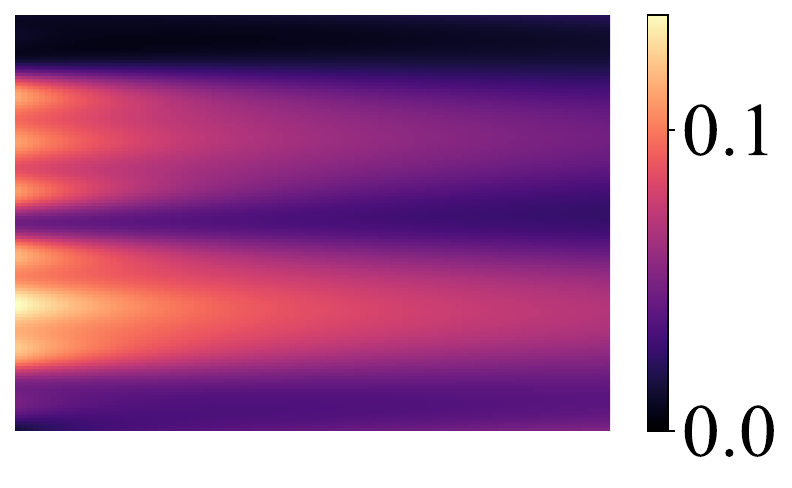}
    \vspace{-18pt}\caption{DPDE}
  \end{subfigure}
  \hfill
  \begin{subfigure}[t]{0.23\textwidth}
    \centering
    \includegraphics[width=\textwidth]{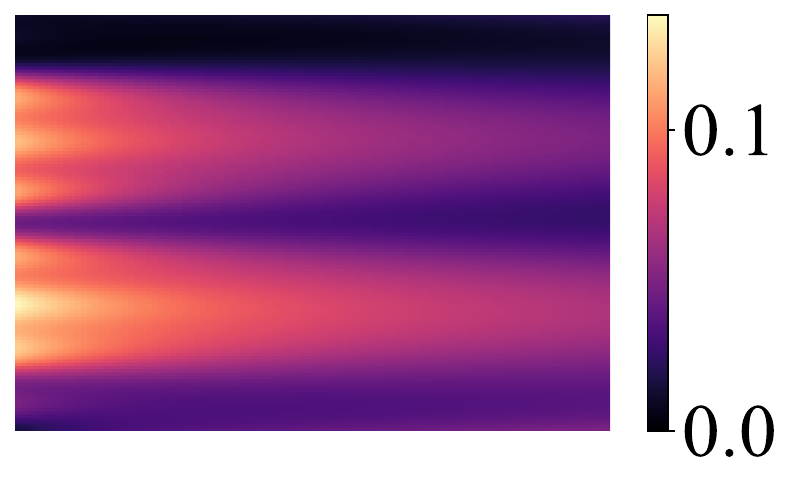}
    \vspace{-18pt}\caption{FFM}
  \end{subfigure}
  \vspace{-4pt}
  \caption{\small MMSE (the darker the better) of Reaction--Diffusion solutions over physical time, with the horizontal axis representing time evolution and the vertical axis indicating errors at each point.}
  \label{fig:rd_mse}
\end{figure}
\begin{figure}[t]
\vspace{-8pt}
  \centering
  \begin{subfigure}[t]{0.23\textwidth}
    \centering
    \includegraphics[width=\textwidth]{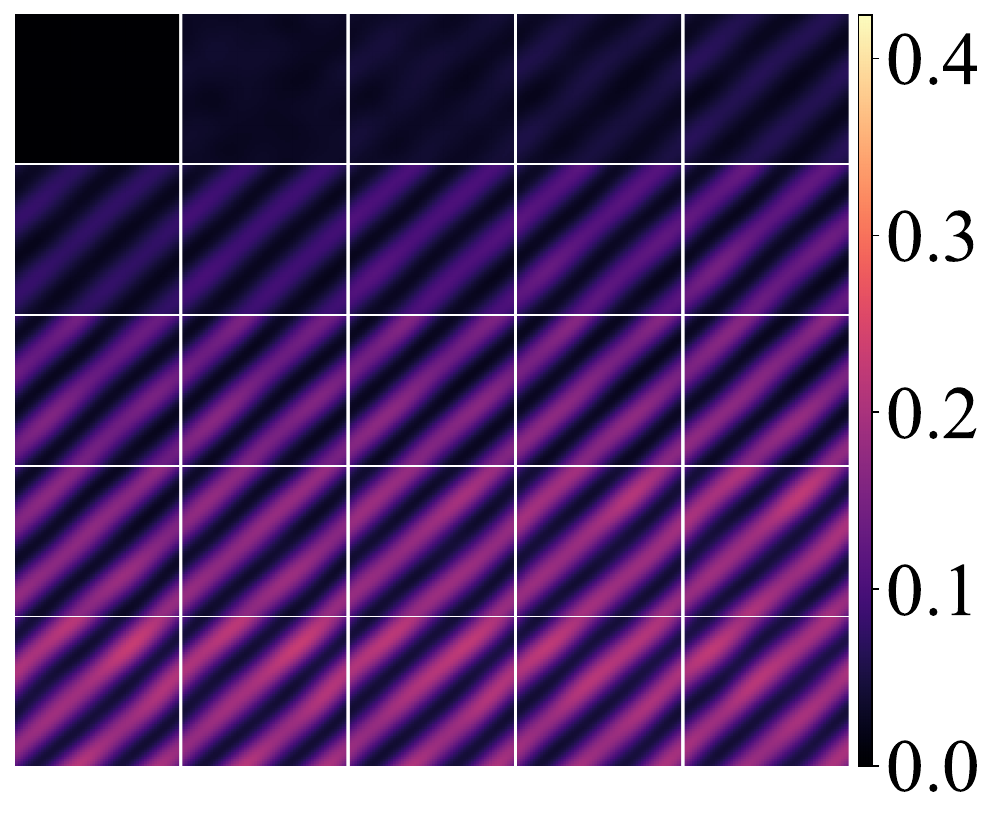}
    \vspace{-18pt}\caption{CCFM (ours)}
  \end{subfigure}
  \hfill
  \begin{subfigure}[t]{0.23\textwidth}
    \centering
    \includegraphics[width=\textwidth]{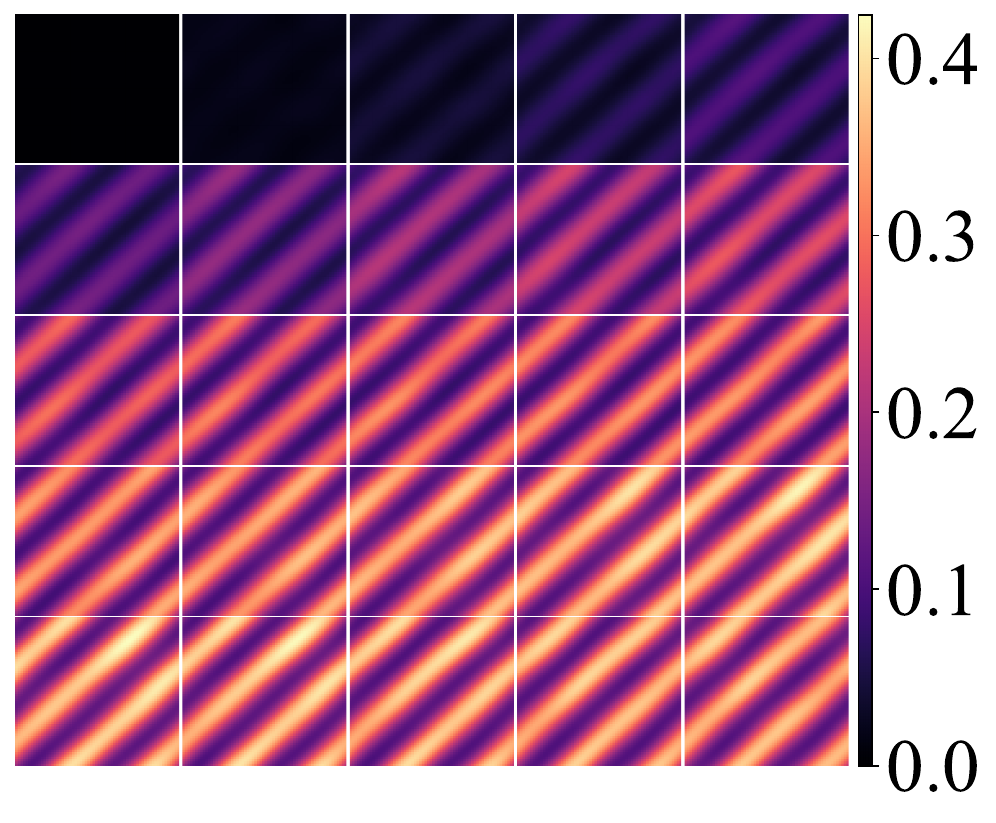}
    \vspace{-18pt}\caption{PCFM}
  \end{subfigure}
  \hfill
  \begin{subfigure}[t]{0.23\textwidth}
    \centering
    \includegraphics[width=\textwidth]{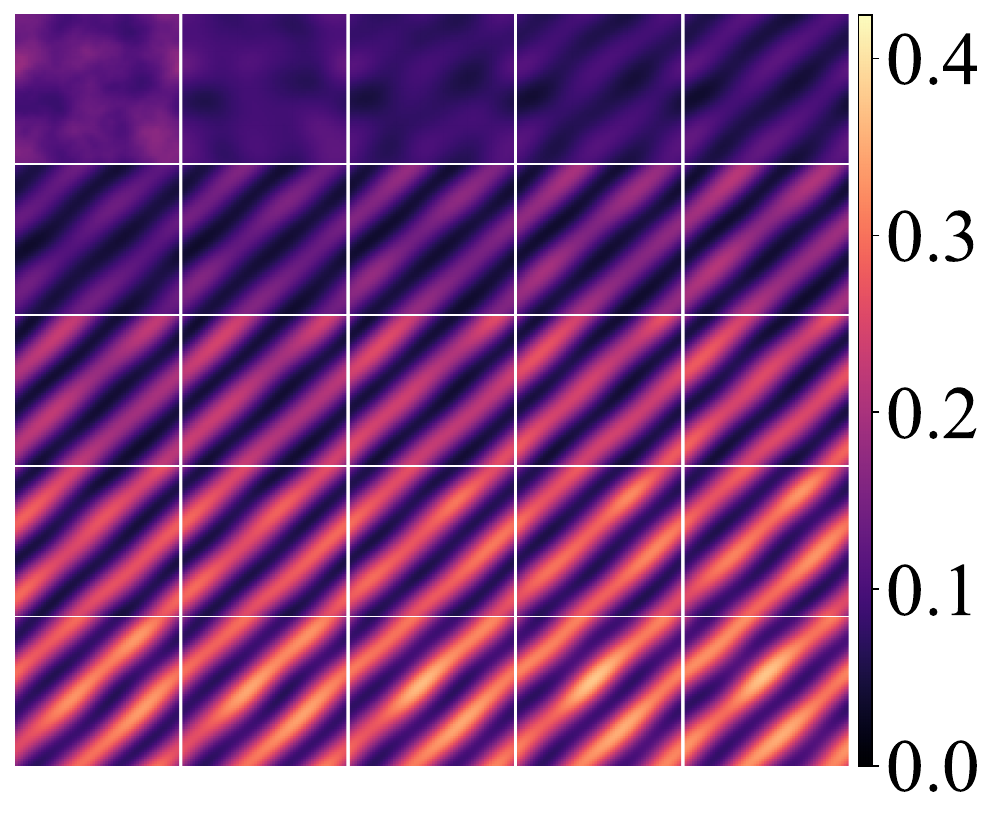}
    \vspace{-18pt}\caption{DPDE}
  \end{subfigure}
  \hfill
  \begin{subfigure}[t]{0.23\textwidth}
    \centering
    \includegraphics[width=\textwidth]{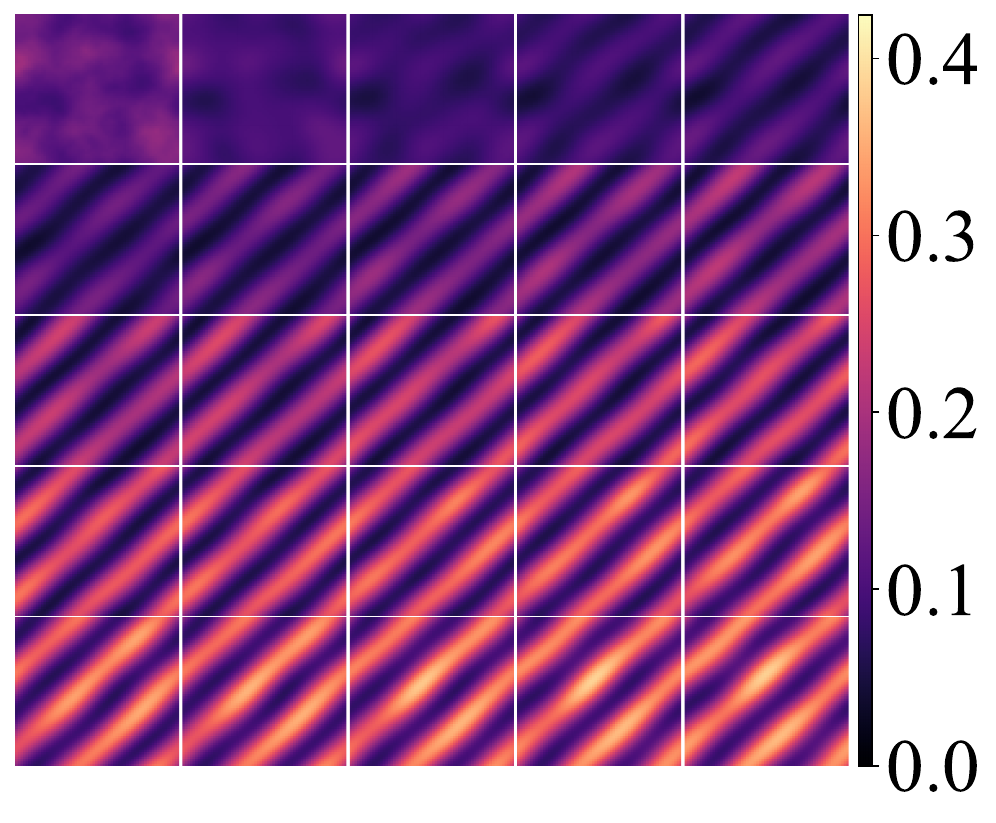}
    \vspace{-18pt}\caption{FFM}
  \end{subfigure}
  \vspace{-4pt}
  \caption{\small MMSE (the darker the better) of Navier--Stokes solutions over physical time, downsampled from 50 to 25 frames for clarity. Each frame corresponds to the error at a given physical time.}
  \label{fig:ns_mse}
\end{figure}

\subsection{Performance Evaluation for Physics-informed Motion}
\begin{wraptable}[7]{r}{0.56\columnwidth}
    \centering
    \vspace{-1\intextsep}
    \caption{Comparison of model performance in the physics-informed motion task.}
    \label{tab:pim}
    \resizebox{\linewidth}{!}{
    \begin{tabular}{lccc}
        \toprule
        & \textbf{Unconstrained} & \textbf{RP} & \textbf{CCFM} \\
        \midrule
        Constraint Violation (MSE) ($\downarrow$) & 278.46 & 0.8333 & \textbf{0.0100} \\
        FID ($\downarrow$)                        & \textbf{16.985} & 27.184 & 26.767 \\
        \bottomrule
    \end{tabular}
    }
\end{wraptable}
Table~\ref{tab:pim} shows that CCFM achieves the lowest constraint violation and improved FID compared to the RP approach.
The unconstrained model attains better FID because it avoids background artifacts despite its large constraint violation. Although RP and CCFM do not directly modify the background, the Earth-trained base model tends to generate in-distribution frames. After projection moves the object to the Moon-gravity location, the model may partially pull it back toward its in-distribution trajectory, creating slight background ghosting. 
This also contributes to the higher violation observed for RP, since violation is computed from detected object centers. Ghosting artifacts can interfere with center detection, increasing the measured violation.

\section{Conclusion}
\label{sec:conclusion}
This paper introduced \emph{Chance-constrained Flow Matching} (CCFM), a training-free framework for constraint-aware generative modeling. Its key idea is a reformulation of the repeated projection \emph{during flow-matching sampling} through a stochastic-optimization lens. Under exact terminal projection, CCFM guarantees terminal feasibility. Our analysis also bounds the perturbation introduced by constraint corrections relative to the unconstrained sampler.
Experiments on molecular docking, PDE solution generation, and physics-informed motion demonstrate its effectiveness, with particularly strong gains in docking under limited sampling budgets.

\textbf{Limitations and future work.} Despite the strong potential of CCFM, several directions remain open. While Theorem~\ref{thm:correction_controlled_perturbation} provides a finite-horizon stability bound relative to the unconstrained sampler, the convergence of the generated distribution to the exact constrained target distribution $q_C$ remains a direction for future work.
In domains such as robotics and other scientific simulation applications, constraints may not admit closed forms. This aspect is out of the scope of the work, and the challenge of embedding non-differentiable black-box simulators (e.g., collision checkers, invariants) remains an interesting direction for future work.

\section*{Acknowledgment}
The material is based upon work supported by National Science Foundations (NSF) awards 2533631, 2401285, 2334448, and 2334936, and Defense Advanced Research Projects Agency (DARPA) under Contract No.~\#HR0011252E005. 
The research was also supported by the U.S. Department of Energy, Advanced Scientific Computing Research, through the SciDAC-RAPIDS2 institute, and Competitive Portfolios project on Energy Efficient Computing: A Holistic Methodology, under contract DE-AC02-06CH11357.
The authors acknowledge Research Computing at the University of Virginia, the Argonne Leadership Computing Facility, which is a DOE Office of Science User Facility supported under Contract DE-AC02-06CH11357, Laboratory Computing Resource Center (LCRC) at the Argonne National Laboratory for providing computational resources that have contributed to the results reported within this paper. 
Any opinions, findings, conclusions, or recommendations expressed in this material are those of the authors and do not necessarily reflect the views of NSF, DARPA, or DOE. The authors declare no competing interests.
\newpage

\bibliographystyle{plainnat}  
\bibliography{ref}  

\newpage
\appendix

\section{Missing Proofs}
\label{app:missing_proof}
% This section provides the theoretical results.

\subsection{Proof of Proposition~\ref{prop:tractable_reformulation}}
\label{app:proof_prop_1}

\begin{proof}[Proof of Proposition~\ref{prop:tractable_reformulation}]
Let
\[
\sigma_a:=\sigma_t\|a\|_2 .
\]
Since $\xi_t\sim\mathcal N(0,\sigma_t^2 I_d)$, we have
\[
a^\top \xi_t\sim\mathcal N(0,\sigma_a^2).
\]
If $\sigma_a=0$, then $a^\top\xi_t=0$ almost surely. Hence,
\[
\mathbb P_{\xi_t}\!\left(h(x_t)+a^\top \xi_t \le b\right)\ge \alpha
\]
holds if and only if $h(x_t)\le b$, which is the stated deterministic
inequality because $\sigma_t\|a\|_2=0$.

Assume $\sigma_a>0$. Since $h(x_t)$ is deterministic with respect to $\xi_t$,
\[
\mathbb P_{\xi_t}\!\left(h(x_t)+a^\top \xi_t \le b\right)
=
\mathbb P\!\left(a^\top \xi_t\le b-h(x_t)\right)
=
\Phi\!\left(\frac{b-h(x_t)}{\sigma_a}\right).
\]
Because $\Phi$ is strictly increasing and $\alpha\in(0,1)$,
\[
\Phi\!\left(\frac{b-h(x_t)}{\sigma_a}\right)\ge \alpha
\quad\Longleftrightarrow\quad
\frac{b-h(x_t)}{\sigma_a}\ge \Phi^{-1}(\alpha).
\]
Rearranging gives
\[
h(x_t)\le b-\sigma_a\Phi^{-1}(\alpha)
= b-\sigma_t\|a\|_2\Phi^{-1}(\alpha),
\]
which proves the equivalence.
\end{proof}

\subsection{Proof of Corollary~\ref{cor:cc_feas}}
\begin{proof}
From \eqref{eq:cc_feasible_set}, the random term is
$\xi_t=(1-t)t^{-1}Z$. At \(t=1\), it follows that \(\xi_t=0\). Hence, for
any candidate \(y\),
\[
\mathbb{P}_{Z\sim p_0}\!\left(g(t^{-1}y - \xi_t)\le 0\right)
=
\mathbb{P}_{Z\sim p_0}\!\left(g(y)\le 0\right)
=
\begin{cases}
1, & g(y)\le 0,\\
0, & \text{otherwise}.
\end{cases}
\]
Since \(\alpha\in(0,1]\), the terminal chance constraint is therefore
equivalent to \(g(y)\le 0\). Thus, at \(t=1\), the chance-constrained
feasible set coincides with the clean feasible set
\(\mathcal C_1=\{y\in\mathbb R^d:g(y)\le 0\}\). Consequently,
\(\mathcal P_{\mathrm{cc}}^{1}(x_1)\) reduces to the Euclidean projection
onto \(\mathcal C_1\), which proves the claim.
\end{proof}

\subsection{Proof of Theorem~\ref{thm:correction_controlled_perturbation}}
\label{app:proof_thm_perturbation}

\begin{proof}
Since \(u_\theta(\cdot,t_k)\) is \(L\)-Lipschitz, the Euler map \(F_k\) is \(\gamma=1+hL\)-Lipschitz. Indeed, for any \(x,y\in\mathbb R^d\),
\[
\begin{aligned}
\|F_k(x)-F_k(y)\|_2
&=
\|x-y+h(u_\theta(x,t_k)-u_\theta(y,t_k))\|_2  \\
&\le
\|x-y\|_2+h\|u_\theta(x,t_k)-u_\theta(y,t_k)\|_2  \\
&\le
(1+hL)\|x-y\|_2
=
\gamma\|x-y\|_2 .
\end{aligned}
\]
Let \(d_k=\|X_k-Y_k\|_2\). Since \(X_0=Y_0\), we have \(d_0=0\). Using \(X_{k+1}=F_k(X_k)+e_k\) and \(Y_{k+1}=F_k(Y_k)\), we obtain
\[
\begin{aligned}
d_{k+1}
&=
\|X_{k+1}-Y_{k+1}\|_2  \\
&=
\|F_k(X_k)+e_k-F_k(Y_k)\|_2  \\
&\le
\|F_k(X_k)-F_k(Y_k)\|_2+\|e_k\|_2  \\
&\le
\gamma d_k+\|e_k\|_2 .
\end{aligned}
\]
Unrolling this recursion from \(d_0=0\) yields
\[
d_N
\le
\sum_{k=0}^{N-1}\gamma^{N-1-k}\|e_k\|_2,
\]
which proves \eqref{eq:pathwise_perturbation_bound}.

For the Wasserstein bound, couple \(\mu_N\) and \(\mu_N^{\mathrm{free}}\) by using the same initial sample \(X_0=Y_0\sim p_0\). This is a valid coupling between the distributions of \(X_N\) and \(Y_N\). Therefore, by the definition of \(W_2\),
\[
W_2(\mu_N,\mu_N^{\mathrm{free}})
\le
\left(\mathbb E\|X_N-Y_N\|_2^2\right)^{1/2}.
\]
Applying Minkowski's inequality to the pathwise bound gives
\[
\begin{aligned}
\left(\mathbb E\|X_N-Y_N\|_2^2\right)^{1/2}
&\le
\left(
\mathbb E
\left[
\sum_{k=0}^{N-1}\gamma^{N-1-k}\|e_k\|_2
\right]^2
\right)^{1/2} \\
&\le
\sum_{k=0}^{N-1}\gamma^{N-1-k}
\left(\mathbb E\|e_k\|_2^2\right)^{1/2}.
\end{aligned}
\]
This proves \eqref{eq:w2_perturbation_bound}.

Finally, for any target distribution $q$ with finite second moment, the triangle inequality for $L^2$ Wasserstein distance \(W_2\) implies
\[
W_2(\mu_N,q)
\le
W_2(\mu_N,\mu_N^{\mathrm{free}})
+
W_2(\mu_N^{\mathrm{free}},q).
\]
Combining this inequality with \eqref{eq:w2_perturbation_bound} gives
\[
W_2(\mu_N,q)
\le
W_2(\mu_N^{\mathrm{free}},q)
+
\sum_{k=0}^{N-1}\gamma^{N-1-k}
\left(\mathbb E\|e_k\|_2^2\right)^{1/2}.
\]
\end{proof}

\section{Implementation Details}
\label{app:implementation}

\subsection{Molecular Docking Formulation}
\label{app:docking}

Given a receptor with atom coordinates \(\mathcal{Y} \in \mathbb{R}^{N_{\mathrm{rec}}\times 3}\) and a ligand template conformer \(\mathcal{X}_0 \in \mathbb{R}^{N_{\mathrm{lig}}\times 3}\), we parameterize a candidate pose by a rigid-body motion \((R,t)\in \mathrm{SO}(3)\times\mathbb{R}^3\) together with a vector of rotatable-bond torsions \(\phi \in \mathbb{T}^{K}\). The posed ligand is expressed as
\[
\mathcal{X}(R,t,\phi) \;=\; R\,T_{\mathrm{lig}}(\mathcal{X}_0;\phi) \;+\; \mathbf{1}\,\Delta^\top,
\]
where \(T_{\mathrm{lig}}(\cdot;\phi)\) applies torsional rotations about designated bond axes. The goal of molecular docking is to recover ligand poses that faithfully reproduce experimentally observed binding modes while maintaining chemical validity and physical plausibility. This work considers four classes of constraints:

\paragraph{Bond Lengths.}
For each ligand bond \((i,j) \in \mathcal{E}_{\mathrm{bond}}\),
\[
(1-\delta_b)\,\ell_{ij}^{\mathrm{low}} \;\le\; \|\mathcal{X}_i - \mathcal{X}_j\| \;\le\; (1+\delta_b)\,\ell_{ij}^{\mathrm{up}},
\]
where \(\mathcal{E}_{\mathrm{bond}}\) is the set of ligand covalent bonds, $\delta_b =0.25$ denotes the tolerance, \(\ell_{ij}^{\mathrm{low}}, \ell_{ij}^{\mathrm{up}}\) are chemically derived lower/upper bounds (in \(\text{\AA}\)) for bond \((i,j)\), \(\mathcal{X}_i\in\mathbb{R}^3\) and \(\mathcal{X}_j\in\mathbb{R}^3\) denote the Cartesian coordinates (in \(\text{\AA}\)) of the \(i\)-th and \(j\)-th ligand atom, respectively. 

\paragraph{Bond Angles.}
For every bond-angle triplet \((i\!-\!j\!-\!k) \in \mathcal{E}_{\angle}\),
\[
(1-\delta_\alpha)\,\alpha_{ijk}^{\mathrm{low}} \;\le\; \angle(i\!-\!j\!-\!k) \;\le\; (1+\delta_\alpha)\,\alpha_{ijk}^{\mathrm{up}},
\]
where \(\mathcal{E}_{\angle}\) denotes the set of bond-angle triplets, \(\angle(i\!-\!j\!-\!k)\) is the geometric angle at vertex \(j\), \(\alpha_{ijk}^{\mathrm{low}}, \alpha_{ijk}^{\mathrm{up}}\) are lower/upper bounds (in radians) for angle \((i\!-\!j\!-\!k)\), $\delta_\alpha =0.25$ denotes the bond angles tolerance.

\paragraph{Internal Steric Clash (Intra-ligand).}
For all ligand atom pairs \((i,j)\) not directly bonded or forming an angle, i.e., \((i,j)\notin \mathcal{E}_{\mathrm{bond}}\cup\mathcal{E}_{\angle}\),
\[
\|\mathcal{X}_i - \mathcal{X}_j\| \;\ge\; (1-\delta_{\mathrm{steric}})\, d_{ij}^{\mathrm{low}},
\]
where \(d_{ij}^{\mathrm{low}}\) is the lower bound (in \(\text{\AA}\)) on permissible separation for nonbonded, nonangle ligand pairs, \(\delta_{\mathrm{steric}} = 0.2\) is the internal steric clash tolerance.

\paragraph{Minimum Protein–Ligand Contact.}
Defining the closest cross-molecule distance as
\[
d_{\min}(\mathcal{X},\mathcal{Y}) \;=\; \min_{i\in\mathrm{lig},\, j\in\mathrm{rec}} \|\mathcal{X}_i - \mathcal{Y}_j\|,
\]
We require that the van der Waals overlap does not exceed a tolerance~$\delta_\text{pl}=0.75$\AA:
\[
\forall (i,j): \; d_{lp} \;\ge\; (\rho_i + \rho_j) - \delta_\text{pl},
\quad \text{equivalently } \;
\max_{i,j}\,\max\!\big(0,\, \rho_i + \rho_j - d_{ij}\big) \;<\; \delta_\text{pl},
\]
where $d_{ij}$ is the distance between ligand atom $i$ and protein atom $j$, and $\rho_i, \rho_j$ denote their van der Waals radii.  
Second, to ensure the ligand is not placed too far from the binding pocket, we require the existence of at least one contact.

\subsection{PDE Problem Formulation}
\label{app:pde}
\subsubsection{Reaction--Diffusion Equation}
We consider the one-dimensional nonlinear reaction--diffusion equation
\[
\frac{\partial v}{\partial t} \;=\; \nu \frac{\partial^2 v}{\partial s^2} + \rho\, v (1 - v),
\]
where $v(s, t)$ denotes the state of the species at position $s$ and time $t$, $\nu > 0$ is the diffusion coefficient, and $\rho > 0$ is the reaction rate constant.

The solution $v(s,t)$ must satisfy the following two constraints:

\paragraph{Initial Condition.}
\[
- \delta_\text{pde} \;\leq\; v(s,0) - v_{\mathrm{IC}}(s) \;\leq\; \delta_\text{pde},
\]
where $v_{\mathrm{IC}}(s)$ denotes the initial state, $\delta_\text{pde}=1 \times 10^{-13}$ is the tolerance. This condition ensures that the solution starts from a specific initial condition $v_{\mathrm{IC}}(s)$ at time $t = 0$.

\paragraph{Nonlinear Mass Conservation.}
\[
- \delta_\text{pde} \;\leq\; m(t) - \left( m(0) \;+\; \int_0^t \rho  \nu (1-\nu) dt \;+\; \int_0^t (g_L - g_R)\, dt \right) \;\leq\; \delta_\text{pde},
\]
where $m(t) = \int v(s,t)\, ds$ is the total mass at time $t$, and $g_L$, $g_R$ are the Neumann boundary fluxes at the left and right boundaries, respectively.

These constraints ensure consistency with the prescribed initial condition and the correct nonlinear evolution of the total mass due to reaction and boundary effects.

\subsubsection{Navier–-Stokes Equation}
The dynamics of incompressible fluids are governed by the Navier--Stokes equations:
\[
\frac{\partial \mathbf{v}}{\partial t} + (\mathbf{v}\cdot\nabla)\mathbf{v}
= -\nabla \pi + \nu \nabla^2 \mathbf{v} + \mathbf{f}, 
\qquad \nabla \cdot \mathbf{v} = 0,
\]
where $\mathbf{v}(s,t)$ is the velocity field, $\pi(s,t)$ is the pressure field, $\nu > 0$ is the kinematic viscosity, $\mathbf{f}(s,t)$ is an external force (e.g., gravity).

The solution must satisfy the following two constraints:

\paragraph{Initial Condition.}
\[
- \delta_\text{pde} \;\leq\;  \mathbf{v}(s,0) - \mathbf{v}_{\mathrm{IC}}(s) \;\leq\; \delta_\text{pde},
\]
where $\mathbf{v}_{\mathrm{IC}}(s)$ is the initial state. This condition ensures the velocity field starts from a prescribed initial condition.

\paragraph{Global Mass Conservation.}
\[
- \delta_\text{pde} \;\leq\;  \iint \mathbf{v}(s,t)\,ds \;-\; \iint \mathbf{v}(s,0)\,ds, \;\leq\; \delta_\text{pde},
\]
ensuring that the total mass of the incompressible fluid remains constant over time.

\subsection{Physics-Informed Motion Formulation}
\label{app:motion}

We model the vertical motion of a rigid object falling under uniform gravity. Each trajectory contains six frames corresponding to a physical duration of \(4\) seconds. The horizontal location is fixed, and only the vertical coordinate evolves according to the gravitational dynamics.

\paragraph{Free-Fall Dynamics.}
Let \(\mathbf{p}_t \in \mathbb{R}^2\) denote the object position at frame \(t\), and let \(\mathbf{v}_t\) denote its velocity. The motion obeys the discrete-time update rule
\begin{align}
\mathbf{p}_t &= \mathbf{p}_{t-1} + \left( \mathbf{v}_t + 0.5 \,\frac{\partial \mathbf{v}_t}{\partial t} \right),  \tag{12a}\\[3pt]
\mathbf{v}_{t+1} &= \frac{\partial \mathbf{p}_t}{\partial t} + \frac{\partial \mathbf{v}_t}{\partial t}, \tag{12b}
\end{align}
where \(\frac{\partial \mathbf{v}_t}{\partial t}\) is the gravitational acceleration. During training we use Earth gravity
\[
\frac{\partial \mathbf{v}_t}{\partial t} \;=\; \mathbf{a}_{\mathrm{E}} = (0,\,-9.8),
\]
and during evaluation we impose Moon gravity
\[
\frac{\partial \mathbf{v}_t}{\partial t} \;=\; \mathbf{a}_{\mathrm{M}} = (0,\,-1.62).
\]

We use \(2048\) trajectories for training and \(205\) for validation. For evaluation, we generate \(100\) trajectories under Moon gravity using the three methods (unconstrained generation, repeated projection, and CCFM).

\subsection{Constraints Enforcement}
\label{app:constrain_enforcement}
\subsubsection{Molecular Docking}
To enforce geometric feasibility on molecular conformations, we adopt an iterative, gradient-free projection strategy. Starting from an initial state $x^{(0)}$, the algorithm generates a sequence of iterates $x^{(k+1)}=\text{Project}(x^{(k)})$, where each projection step reduces the violation of geometric constraints. The process terminates once no constraint is violated or a fixed iteration budget is reached, yielding a final feasible conformation.

At each step, we first identify the active set of violated constraints, denoted $\mathcal{A}(x)$. For each violated constraint, we compute a corrective displacement along the axis connecting the two atoms involved. The displacement magnitude is proportional to the violation and scaled by a small fixed step size, producing a push or pull effect that nudges atoms toward feasibility. The updates across all active constraints are aggregated into a single correction $\Delta x$, with per-atom displacements capped to ensure numerical stability. The coordinates are then updated as $x \leftarrow x + \Delta x$.

While a single update reduces constraint violations, it may not achieve strict feasibility. Iterative application guarantees convergence to a geometrically valid conformation. By restricting updates to the active set, the method avoids unnecessary corrections on already satisfied constraints, while still ensuring that all geometric conditions are satisfied at convergence.

\subsubsection{PDE}
Following the enforcement strategy of~\cite{utkarsh2025physics}, we also adopt the Gauss--Newton method for enforcing constraints in PDE formulations. For linear constraints, the projection is exact and guarantees feasibility. For nonlinear constraints, the single projection provides only an approximate correction, and strict feasibility is deferred to the final step, where multiple Gauss--Newton iterations are applied.
In addition, since our setting considers inequality constraints, we adopt an \emph{active set strategy}. Specifically, at each step we identify the subset of constraints that are currently violated and apply projection updates only to this active set. By restricting correction to infeasible constraints, this approach avoids unnecessary computations on already satisfied constraints, while still ensuring that all constraints are satisfied at convergence.

It is important to highlight that, although all constraint enforcement is based on the Gauss--Newton method, our approach differs from PCFM. In PCFM, projections are repeatedly applied to the original constraints defined on the clean sample, regardless of sampling time. In contrast, our method introduces a \emph{chance-constrained projection operator}, which adaptively adjusts the tightness of constraints according to the sampling time. This adaptive mechanism is supported by theoretical guarantees, ensuring both efficiency and correctness.

\paragraph{Linear Projection}
For linear inequality constraints, projection is exact. 
At each sampling step, we solve
\[
x_{\mathrm{proj}} \;=\; \arg\min_{z} \|z-x\|_2^2 
\quad \text{s.t.}\quad g(z)\le 0,
\]
which admits a closed-form solution due to the rank-1 structure of the Jacobian. 
For instance, in the Navier–-Stokes Equation, the total-vorticity constraint reduces to a constant-shift correction
\[
\mathbf{v}(\cdot,\cdot,k) \;\leftarrow\; \mathbf{v}(\cdot,\cdot,k) - \text{average residual}.
\]
This closed-form update yields exact feasibility at negligible cost and is applied after every sampling step.

\paragraph{Gauss--Newton Projection}
For general nonlinear constraints, we employ an approximate correction. 
Given a candidate $x\in\mathbb{R}^n$ subject to $g_i(x)\le 0$, we first identify the active set of violated constraints
\[
\mathcal{A}(x) = \{\, i \mid g_i(x) > 0 \,\},
\]
with residuals and Jacobian
\[
r = g_{\mathcal{A}(x)}(x),\qquad 
J = \nabla g_{\mathcal{A}(x)}(x)^\top \in \mathbb{R}^{n \times |\mathcal{A}(x)|}.
\]
A single Gauss--Newton update
\[
x_{\mathrm{proj}} = x - J^\top (J J^\top)^{-1} r
\]
reduces the residuals of the active constraints, pushing the iterate back toward the feasible region. 
In PCFM and CCFM, this step is applied after each sampling iteration.

Specifically, we implement the active set via hinge residuals $r=\mathrm{ReLU}(g(x))$ for efficient GPU batching. 
The linear system is stabilized by adding the regularization
\[
y = (J J^\top + \lambda I)^{-1} r, \quad \Delta x = -J^\top y, \quad x \leftarrow x + \Delta x,
\]
with $\lambda \approx 10^{-6}$. 

\paragraph{Final Projection for Strict Feasibility}
To guarantee exact satisfaction of all constraints in the generated samples, we perform Newton--Schur iterations restricted to the active set:
\[
x^{(k+1)} 
= x^{(k)} - J_k^\top (J_k J_k^\top)^{-1} g_{\mathcal{A}(x^{(k)})}\!\big(x^{(k)}\big), 
\quad J_k = \nabla g_{\mathcal{A}(x^{(k)})}(x^{(k)})^\top.
\]
Under smoothness and full-rank assumptions, these iterations converge quadratically~\citep{utkarsh2025physics}. 
In both PCFM and CCFM, we allocate a fixed budget of 30 iterations after the final sampling step to ensure strict feasibility.

\subsection{Dataset Details}

\subsubsection{Molecular Docking}
This work utilizes the same dataset as adopted in~\citep{corso2025composing}, namely the PDBBind benchmark~\citep{liu2017forging}, which has become a standard for evaluating docking models. Following previous work~\citep{stark2022equibind,corso2025composing}, we employ the time-based split, where approximately 17,000 complexes released before 2019 are used for training and validation, while the 363 complexes released after 2019 constitute the test set.

To ensure consistency between apo and holo structures, we first sanitize the raw PDBBind files using PDBFixer~\citep{eastman2017openmm} to replace non-standard residues and add missing heavy atoms. 
For the apo state, we extract protein sequences from the processed files and generate structural predictions with ESMFold~\citep{lin2022language}. These predictions are further refined by PDBFixer to add missing terminal atoms, while the crystallographic bound conformations from PDBBind serve as the holo state. 
Finally, hydrogen atoms are removed during alignment to reduce inconsistencies arising from variable hydrogen placement.

\subsubsection{PDE}
Following the dataset construction strategies in~\cite{cheng2024gradient,utkarsh2025physics}, we generate training and testing trajectories by combining randomized initial conditions with varying boundary or forcing configurations. The key distinction is that, whereas prior works typically employ only 10 initial conditions in the test set, we adopt 90. This design probes model performance under more diverse and complex scenarios, thereby yielding a more challenging and discriminative evaluation.

\paragraph{Reaction--Diffusion Equation}
The 1D reaction--diffusion system is simulated on the spatial domain $s \in [0,1]$ with $n_s=128$ grid points and over $n_t=100$ discrete time steps. The numerical simulations were performed using FiPy~\citep{guyer2009fipy}, a Python-based finite volume solver for partial differential equations. We employ a semi-implicit scheme with a fixed diffusion coefficient and reaction rate of $(\nu, \rho) = (0.005, 0.01)$.

The initial conditions are drawn from randomized combinations of sinusoidal modes and localized perturbation functions, providing both smooth and sharp spatial profiles. Neumann boundary conditions are enforced with time-varying fluxes $(g_L, g_R)$ sampled from a prescribed range.

The dataset is split into training and testing sets. To form the training set, we pair 100 distinct initial conditions with 100 unique boundary flux configurations, yielding a total of 10,000 solution trajectories. The test set consists of 1,000 trajectories generated using initial conditions and boundary flux configurations distinct from those used for training. The conditions used for the training and test sets are entirely disjoint. 

\paragraph{Navier–-Stokes Equation}
We consider the 2D incompressible Navier--Stokes equations in vorticity form on a periodic square domain discretized with a $64 \times 64$ spatial grid. Temporal evolution is simulated over $T=49$ time units, with 50 uniformly sampled snapshots recorded per trajectory. We adopt a Crank–Nicolson spectral solver to ensure stability for the viscosity parameter $\nu = 10^{-3}$. The initial vorticity fields are sampled from Gaussian random fields with prescribed covariance structure, while the external forcing is defined as
\[
f(x) \;=\; \tfrac{0.1}{\sqrt{2}} \,\sin\!\big(2\pi(x_1 + x_2) + \varphi\big),
\qquad \varphi \sim \mathcal{U}(0, \pi/2).
\]
To construct the training dataset, we generate 10,000 trajectories by sampling 100 random initial conditions and 100 random forcing phases. The test set consists of 1,000 randomly selected trajectories generated using initial conditions and forcing phases distinct from those used for training. The initial conditions used for training and evaluation are disjoint.

\subsection{Model Details}
\subsubsection{Molecular Docking}
For our experiments, we adopt the FlexDock framework, a state-of-the-art generative model for flexible molecular docking based on Unbalanced Flow Matching (UFM)~\citep{corso2025composing}. FlexDock integrates two complementary modules: a manifold docking flow, which operates in reduced degrees of freedom to efficiently capture ligand rigid-body motions and protein side-chain flexibility, and a subsequent relaxation flow, which refines atomic coordinates in Euclidean space under energy-based constraints to ensure physically valid conformations.

This work directly utilizes the official pre-trained FlexDock models provided by the authors, trained on the PDBBind dataset with time-based splits. This setup allows us to leverage a robust, well-validated backbone without additional fine-tuning, thereby ensuring reproducibility, comparability with prior benchmarks, and fairness in experimental comparisons.

\subsubsection{PDE}
For all PDE benchmarks, we adopt the Functional Flow Matching (FFM) framework parameterized by a Fourier Neural Operator (FNO) backbone~\citep{kerrigan2023functional}. The FNO encoder processes the concatenation of the current state, a sinusoidal positional embedding of the spatial grid, and a Fourier time embedding, thereby enabling spatiotemporal generalization across functional domains. The resulting time-dependent vector field defines the flow dynamics used in all experiments. Table~\ref{tab:hyper_rd} and Table~\ref{tab:hyper_ns} summarize the hyperparameters used during training for reaction--Diffusion and Navier--Stokes equations, respectively.

\begin{table}[h!]
\centering % Center the two minipages on the page

\begin{minipage}{0.45\textwidth}
\small
    \centering
    \caption{Params for Reaction--Diffusion.}
    \begin{tabular}{cc}
    \toprule % Using booktabs for better style
    Hyperparameter & Value \\
    \midrule
    Fourier Layers     & 6 \\
    Fourier Modes  & [32,32] \\
    Hidden Channels      & 64\\
    Projection Channels     & 256 \\
    Batch Size & 256 \\
    Optimizer    & adam \\
    Scheduler & plateau \\
    Prior Distribution         & randn \\
    Training Steps & 20000 \\
    \bottomrule
    \end{tabular}
    \label{tab:hyper_rd}
\end{minipage}% % The '%' character here is important to prevent an unwanted space
\hfill % This command creates a flexible space between the two minipages
\begin{minipage}{0.45\textwidth}
\small
    \centering
    \caption{Params for Navier--Stokes.}
    \begin{tabular}{cc}
    \toprule
    Hyperparameter & Value \\
    \midrule
    Fourier Layers     & 4 \\
    Fourier Modes  &  [16,16,16] \\
    Hidden Channels      & 64\\
    Projection Channels     & 256 \\
    Batch Size & 16 \\
    Optimizer    & adamw \\
    Scheduler & plateau \\
    Prior Distribution         & randn \\
    Training Steps & 200000 \\
    \bottomrule
    \end{tabular}
    \label{tab:hyper_ns} % Changed the label to be unique
\end{minipage}

\end{table}

\subsection{Sampling Setup}
\subsubsection{Molecular Docking}
FlexDock employs two distinct flow-matching models for molecular docking: one dedicated to docking and the other to relaxation. Constraint enforcement is primarily handled by the second model, which applies an Euler-based update scheme. At each inference step, multiple candidate samples are generated, and a discriminator is subsequently used to select the single most promising configuration as the final output. We evaluate FlexDock under varying sample counts (1, 5, 10, 15, 20) and inference steps (2, 5, 10, 15, 20), reporting the detailed results in Section~\ref{sec:experiments} and the remaining results in the appendix~\ref{app:additional_results}.

CCFM adopts a sampling setup largely consistent with FlexDock, employing the same ranges of sample counts and inference steps. The only distinction is that CCFM enforces constraints after each inference step using a gradient-free method shown in Appendix~\ref{app:constrain_enforcement}, followed by a final refinement of 20 iterations once the sampling trajectory is complete. 

\subsubsection{PDE}
For all constraint-aware methods, we employ the modified Euler scheme with 200 sampling steps to ensure a fair comparison.

\paragraph{CCFM (ours).}
Our method performs explicit Euler integration with constraint correction at every step. Each intermediate state is projected onto the constraint manifold using a single Newton update. As shown in Proposition~\ref{prop:tractable_reformulation}, the deterministic reformulation of chance constraints is governed by both the variance of the induced random variable and the probability of constraint satisfaction. Specifically, the variance depends on the sampling time $t$, while the satisfaction probability is given by $\phi_n(t)=\frac{1}{2}t^n$. The parameter $n$ thus controls the tightening schedule: larger values (e.g., $n=0.5$) delay tightening, while smaller values (e.g., $n=0.1$) apply stronger corrections earlier. We adopt $n=0.5$ for the reaction--diffusion task and $n=0.1$ for the Navier--Stokes task. A detailed sensitivity analysis with respect to $n$ is provided in the Appendix~\ref{app:additional_results}. After the sampling process, we apply 30 iterations of Gauss--Newton refinement to guarantee the strict feasibility of all constraints.

\paragraph{PCFM.}
This baseline relies on iterative extrapolation–correction–interpolation steps combined with periodic mixing and noise resampling. To align with our fixed sampling budget, we restrict the number of mixing steps to 2. Constraint correction follows the same projection-based strategy as in CCFM.

\paragraph{Diffusion PDE.}
We follow the gradient-guided sampling procedure~\citep{huang2024diffusionpde} on a pretrained FFM backbone. Constraint satisfaction is encouraged by augmenting the generation loss with explicit penalties on both the initial condition and the conservation law. The composite loss is weighted by a factor of 200 to balance fidelity and feasibility.

\subsection{Hardware}
For the molecular docking experiments, inference was conducted on six NVIDIA A6000 GPUs. For the Reaction--Diffusion, both training and inference were performed on a single NVIDIA A100 GPU. For the Navier--Stokes, the training process was done on six NVIDIA A6000 GPUs, while inference was performed on a single NVIDIA A100 GPU.

\section{Additional Results}
\label{app:additional_results}

\subsection{Marginal Trajectory Curvature}
\label{app:marginal_trajectory_curvature}
To quantify the deviation of learned marginal sampling trajectories from the straight conditional interpolation in \eqref{eq:linearity}, we evaluate 1000 unconstrained Reaction--Diffusion sampling trajectories.
Let $x_0$ be the initial noise and $x_1$ the clean sample. We compare each intermediate state $x_t$ with the straight interpolation $\ell_t=(1-t)x_0+t x_1$ using
\begin{equation}
  d_t = \frac{\|x_t-\ell_t\|_2}{\|x_1-x_0\|_2}.
  \label{eq:marginal_chord_deviation}
\end{equation}
The deviation is normalized by the distance between the initial noise and clean sample. Table~\ref{tab:marginal_trajectory_curvature} shows that the mean normalized deviation from straight interpolation between the initial noise and clean sample is below $1\%$. These measurements provide empirical support for the straight-line approximation.

\begin{table}[t]
  \centering
  \caption{Normalized deviations from straight interpolation between the initial noise and clean sample.}
  \label{tab:marginal_trajectory_curvature}
  \begin{tabular}{ccc}
    \toprule
    Mean & Median & 95th percentile \\
    \midrule
    $0.8622\%$ & $0.8544\%$ & $1.2987\%$ \\
    \bottomrule
  \end{tabular}
\end{table}

\subsection{Molecular Docking}
\label{app:molecular_docking}
\begin{figure}[t]
  \centering
  \begin{subfigure}[t]{0.95\textwidth}
    \centering
    \includegraphics[width=\linewidth]{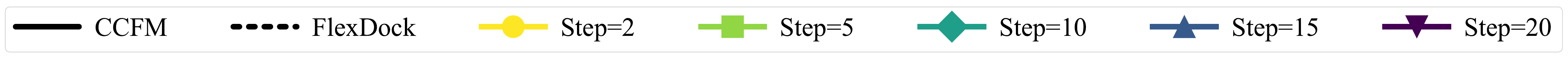}
  \end{subfigure}
  \hfill
  \begin{subfigure}[t]{0.4\textwidth}
    \centering
    \includegraphics[width=\linewidth]{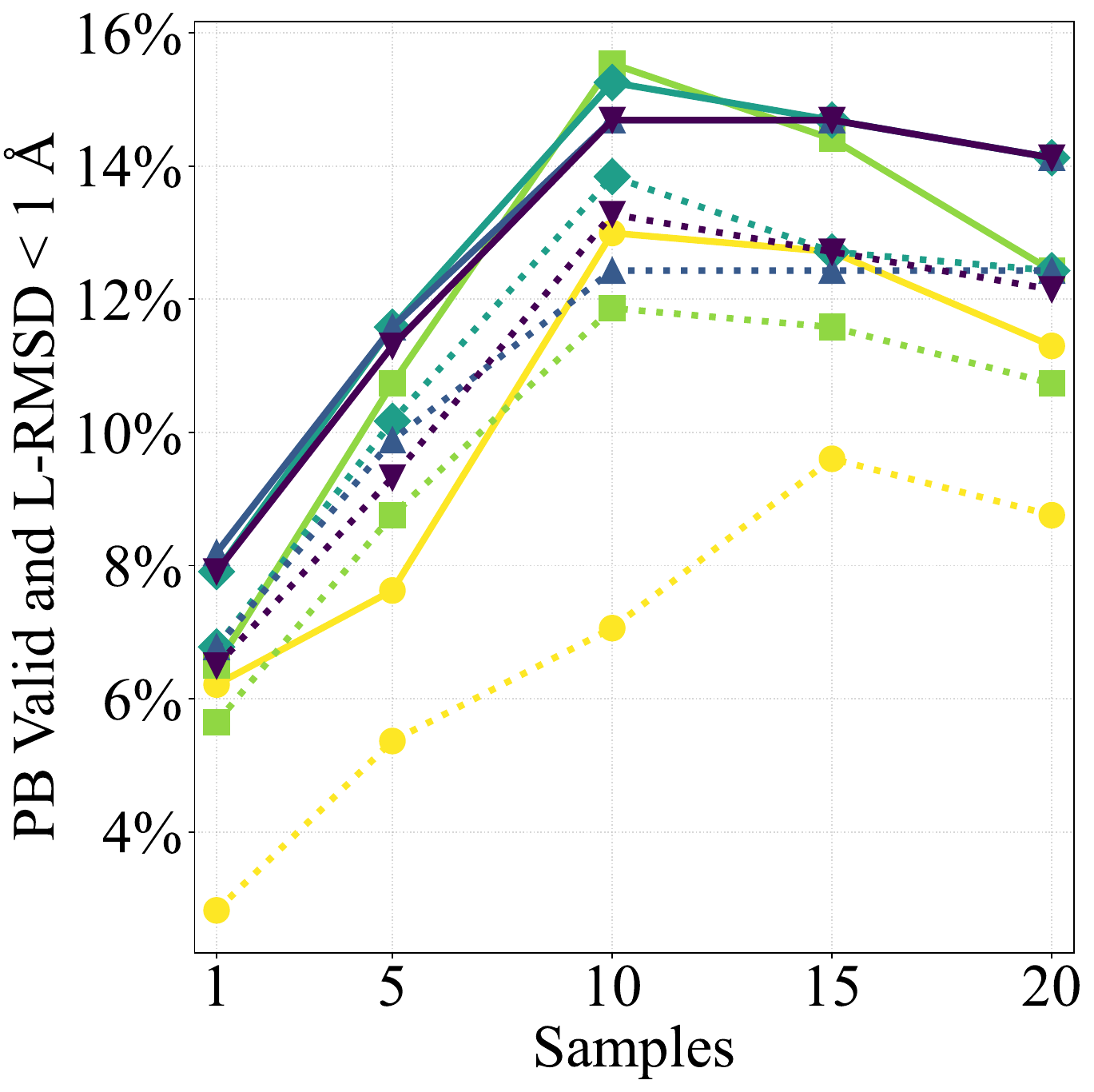}
    \caption{PB Valid and L-RMSD $<$1 \AA}
  \end{subfigure}
  \hfill
  \begin{subfigure}[t]{0.4\textwidth}
    \centering
    \includegraphics[width=\linewidth]{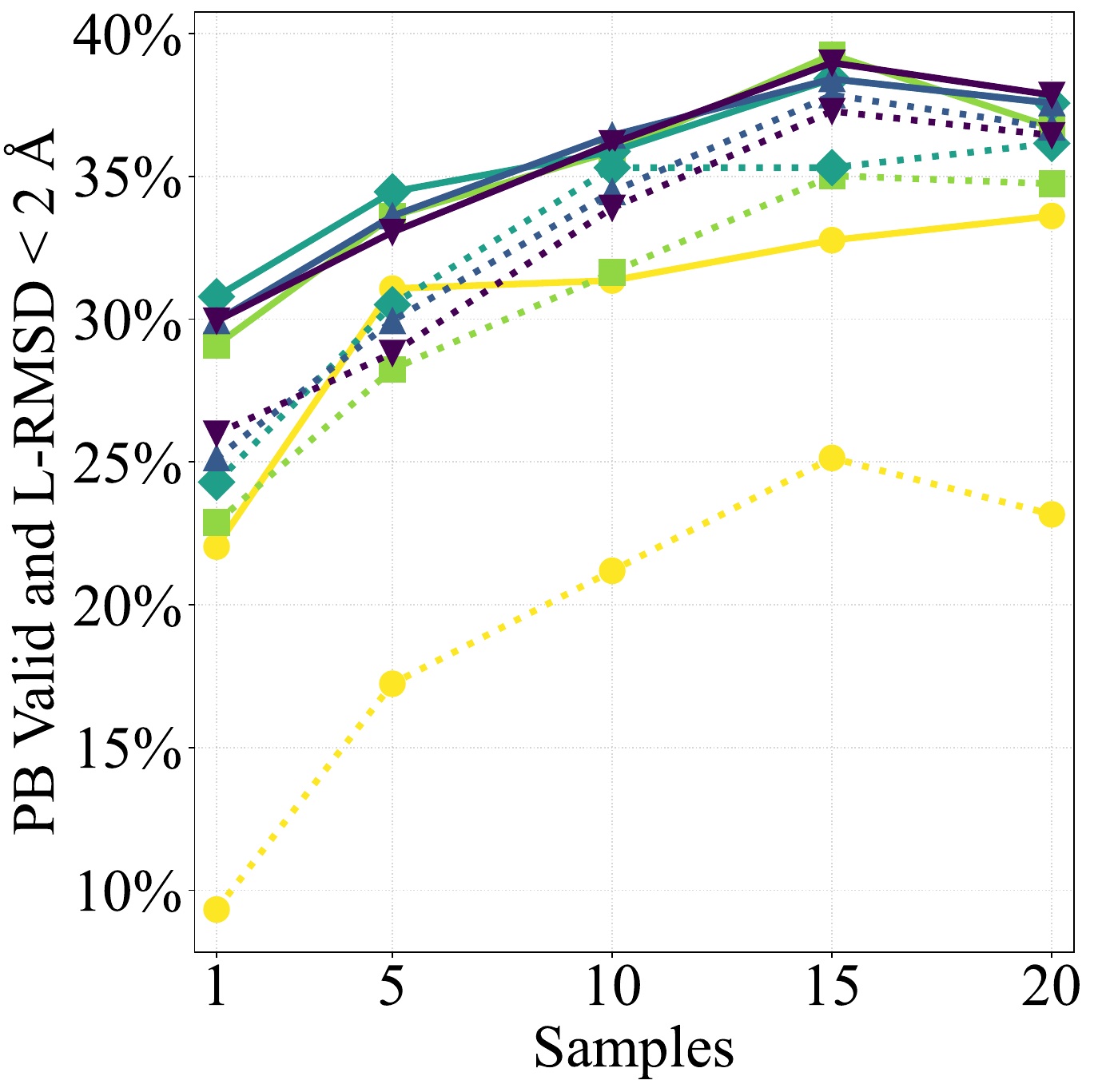}
    \caption{PB Valid and L-RMSD $<2$\AA}
  \end{subfigure}
  \hfill
  \begin{subfigure}[t]{0.4\textwidth}
    \centering
    \includegraphics[width=\linewidth]{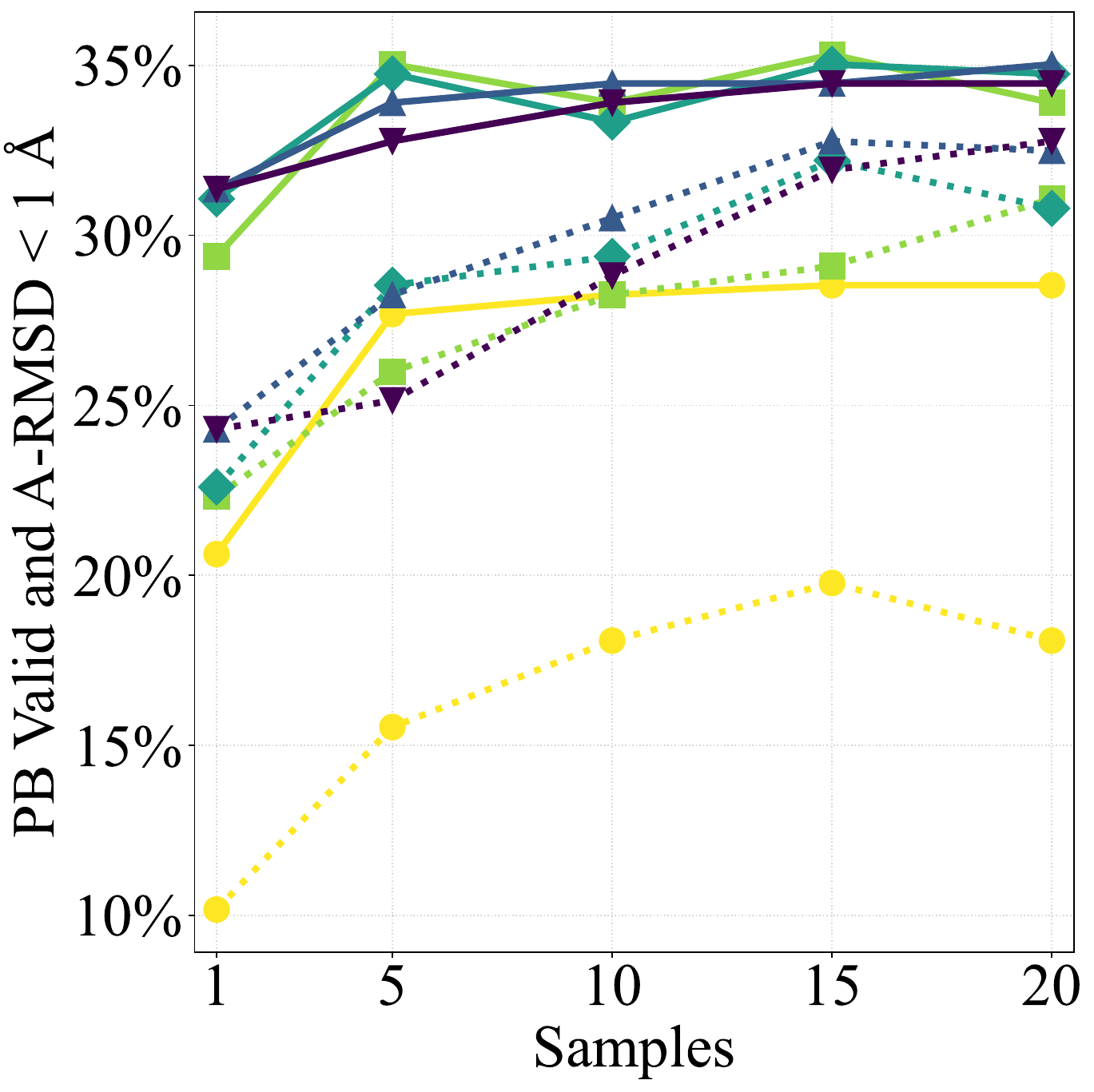}
    \caption{PB Valid and A-RMSD $<$1 \AA}
  \end{subfigure}
  \hfill
  \begin{subfigure}[t]{0.4\textwidth}
    \centering
    \includegraphics[width=\linewidth]{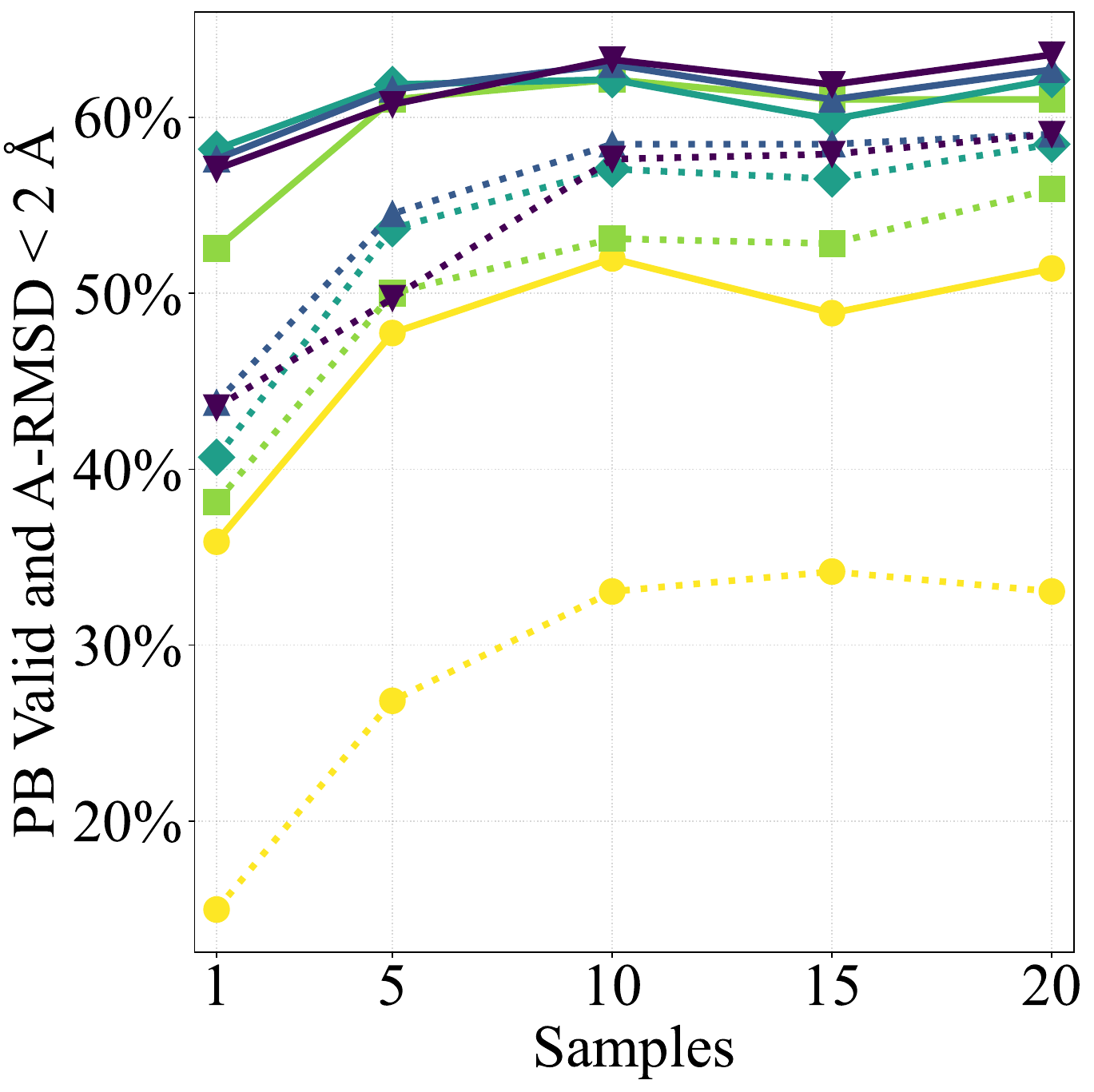}
    \caption{PB Valid and A-RMSD $<2$\AA}
  \end{subfigure}
  \hfill
  \begin{subfigure}[t]{0.4\textwidth}
    \centering
    \includegraphics[width=\linewidth]{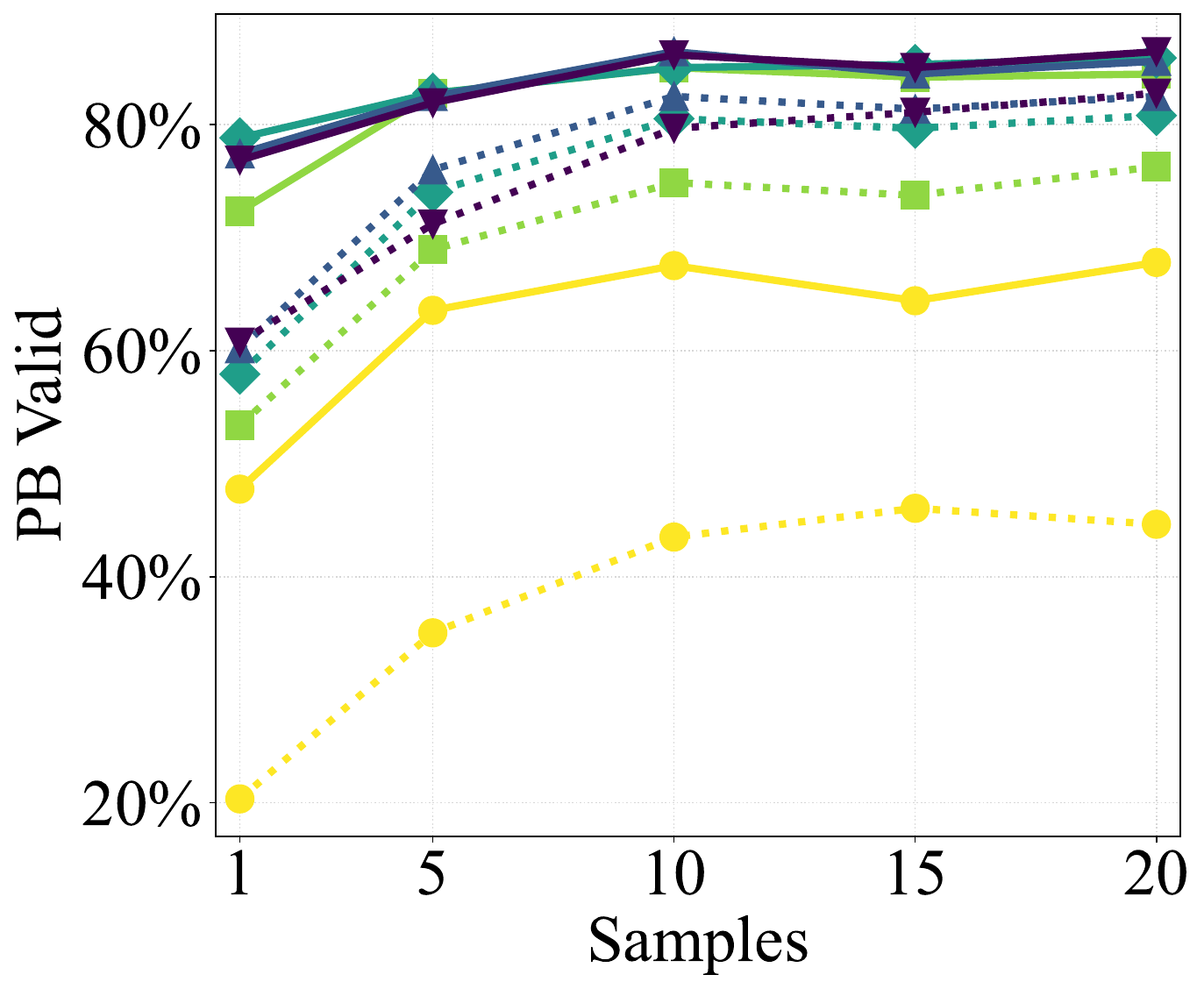}
    \caption{PB Valid}
  \end{subfigure}
\caption{Comparison of CCFM (solid lines) and FlexDock (dashed lines) under different sampling budgets and evaluation metrics.}
  \label{fig:various_metrics_md}
\end{figure}

\begin{table}[t]
  \centering
  \caption{Sensitivity analysis for probability of constraint satisfaction. The probability is determined by the function $\phi_n(t) = \frac{1}{2}t^n$, 
  and we adjust $n$ to observe the results. Results are reported using \#S=1 and \#Stp=10. PB Valid (+L $<2$\AA{} or +A $<2$\AA{}) denotes the percentage of physically valid samples that also satisfy the respective RMSD thresholds.}
  \resizebox{0.75\columnwidth}{!}{%
  \begin{tabular}{
    l % Metric Name
    c % Alpha 0.7
    c % Alpha 0.9
    c % Alpha 1.0
    c % Alpha 1.1
    c % Alpha 1.3
  }
    \toprule
    \textbf{Metric} & \textbf{CCFM(0.7)} & \textbf{CCFM(0.9)} & \textbf{CCFM(1.0)} & \textbf{CCFM(1.1)} & \textbf{CCFM(1.3)} \\
    \midrule
    PB Valid (\textbf{+L $<2$\AA}) & 29.85 & 30.16 & 30.79 & 30.48 & 30.48 \\
    PB Valid (\textbf{+A $<2$\AA}) & 58.19 & 57.85 & 58.19 & 58.19 & 58.88 \\
    \bottomrule
  \end{tabular}%
  }
 \label{tab:sensitivity-alpha}
\end{table}

\begin{table}[t]
  \centering
\caption{PDBBind docking results with standard deviation (\#S=1, \#Stp=10). }
\centering

\sisetup{detect-weight=true}
\resizebox{0.85\linewidth}{!}{%

\begin{tabular}{c
  *{2}{c} % lig
  c       % aa
  *{3}{c} % pb
  c       % running time
}
\toprule

\multirow[c]{2}{*}{\textbf{Methods}} & 
\multicolumn{2}{c}{\textbf{L-RMSD (\%)}} &
\multicolumn{1}{c}{\textbf{A-RMSD (\%)}} &
\multicolumn{3}{c}{\textbf{PB Valid (\%)}} \\ 

\cmidrule(lr){2-3}\cmidrule(lr){4-4}\cmidrule(lr){5-7}
% --- ---
& 
\textbf{$< 1 \text{\AA}$} & 
\textbf{$< 2\,\text{\AA}$} & 
\textbf{$< 2\,\text{\AA}$} &
\textbf{All} & 
\textbf{\textbf{+L $<2$\AA}} & 
\textbf{\textbf{+A $<2$\AA}} \\
\midrule

CCFM
& \shad {\bfseries $7.80 \pm (0.31)$} & \shad {\bfseries $32.57 \pm (0.57)$} & \shad $72.27 \pm (0.69)$ & \shad $79.15 \pm (0.41)$ & \shad $30.45 \pm (0.68)$ & \shad $58.53 \pm (0.49)$
\\
FlexDock
& $7.45 \pm (0.43)$ & $31.92 \pm (0.81)$ & $72.65 \pm (0.81)$ & $58.07 \pm (1.24)$ & $24.57 \pm (1.13)$ & $40.96 \pm (1.38)$ \\
RP
& {\bfseries $7.80 \pm (0.34)$} & $32.19 \pm (0.79)$ & $72.26 \pm (0.99)$ & $56.48 \pm (0.72)$ & $27.68 \pm (0.81)$ & $42.93 \pm (0.85)$  \\
\bottomrule
\end{tabular}%
}
\label{tab:PDBBind_extracted}
\end{table}

A comparison between CCFM and FlexDock across varying sample sizes and inference steps is shown in Figure~\ref{fig:various_metrics_md}. 

At the strictest setting (PB Valid with L-RMSD $<$ 1 \AA, step = 2, sample = 1), CCFM attains over 6\% success rate, more than doubling FlexDock’s ~3\%. A similar advantage persists when relaxing the cutoff to L-RMSD $<$ 2 \AA, where CCFM consistently surpasses FlexDock across all sample and step budgets.

When evaluated on A-RMSD, the margin becomes even more pronounced: for PB Valid with A-RMSD $<$ 1 \AA, CCFM exceeds FlexDock by at least 5\% in nearly every configuration, and maintains comparable superiority at the A-RMSD $< 2$ \AA \, threshold.

For PB Valid, with a moderate budget (5 samples, 5 steps), CCFM achieves a PB Valid level comparable to FlexDock with 20 samples and 20 steps, reducing the effective sample-step cost \textbf{from 400 to just 25}. These results highlight the overall trend: CCFM consistently matches or exceeds FlexDock validity within a fraction of the computational expense. 

The improvements are most striking under limited budgets, where CCFM achieves strong validity and fidelity with far fewer samples and steps. Even as the computational budget increases, CCFM continues to deliver higher feasibility and fidelity while maintaining a clear advantage in efficiency.

\begin{figure}[t]
  \centering
  \begin{subfigure}[t]{0.45\textwidth}
    \centering
    \includegraphics[width=\linewidth]{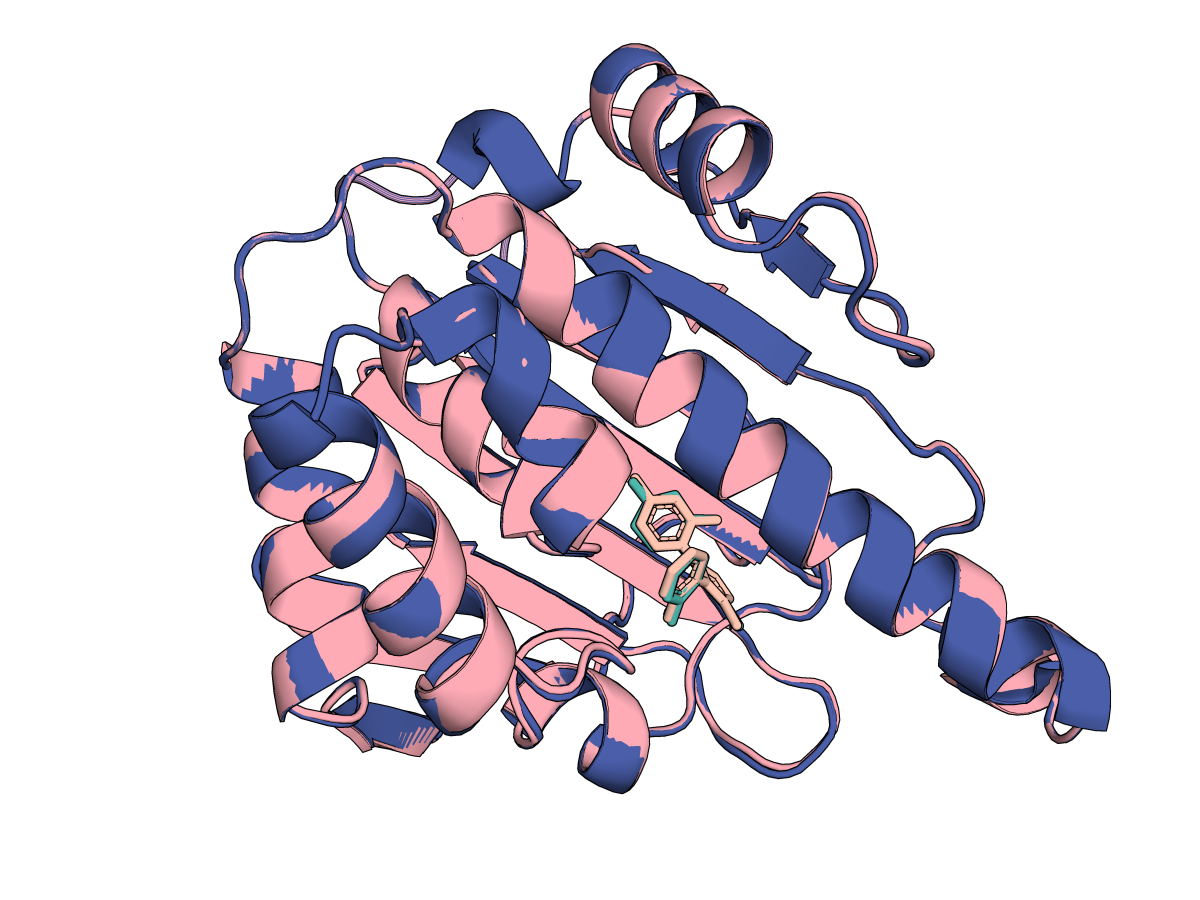}
    \caption{6hhr}
  \end{subfigure}
  \hfill
  \begin{subfigure}[t]{0.45\textwidth}
    \centering
    \includegraphics[width=\linewidth]{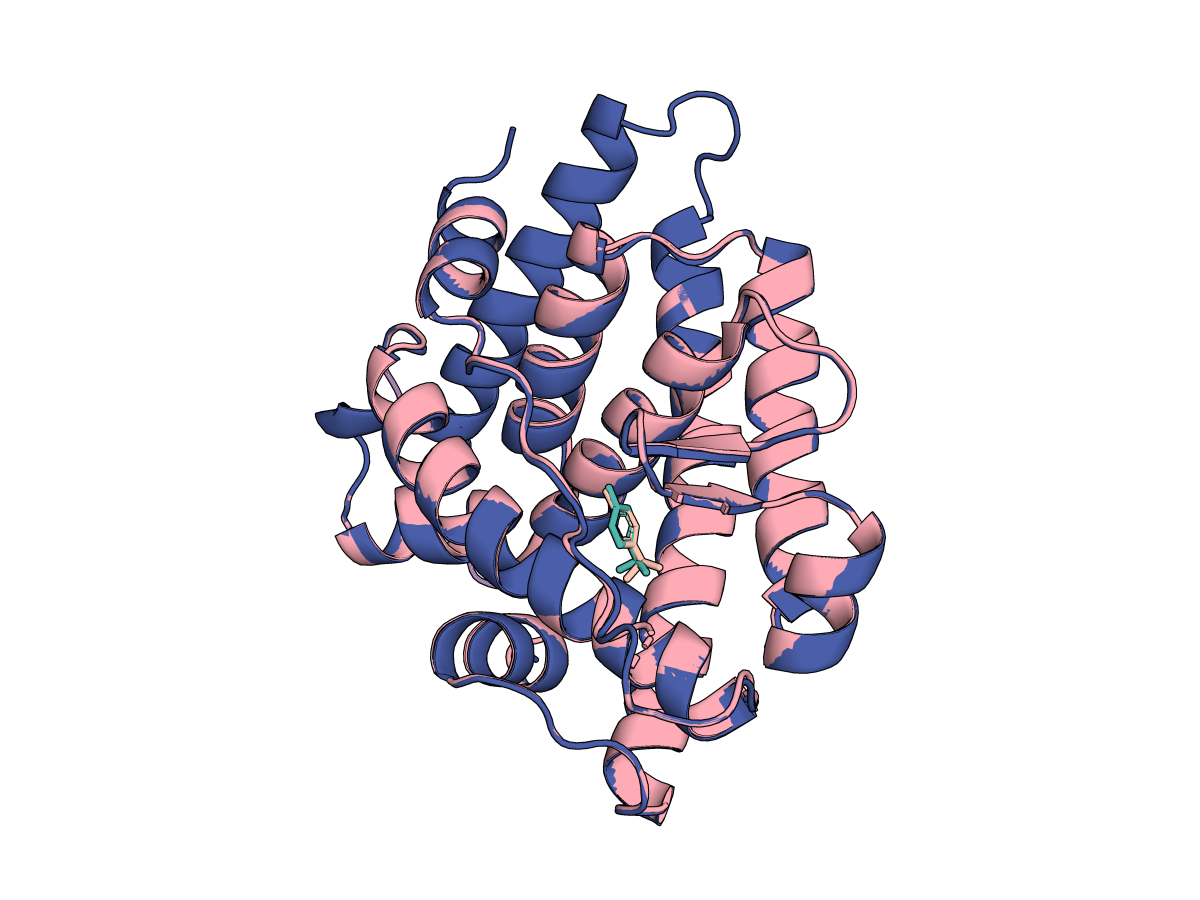}
    \caption{6i65}
  \end{subfigure}
  \hfill
  \begin{subfigure}[t]{0.45\textwidth}
    \centering
    \includegraphics[width=\linewidth]{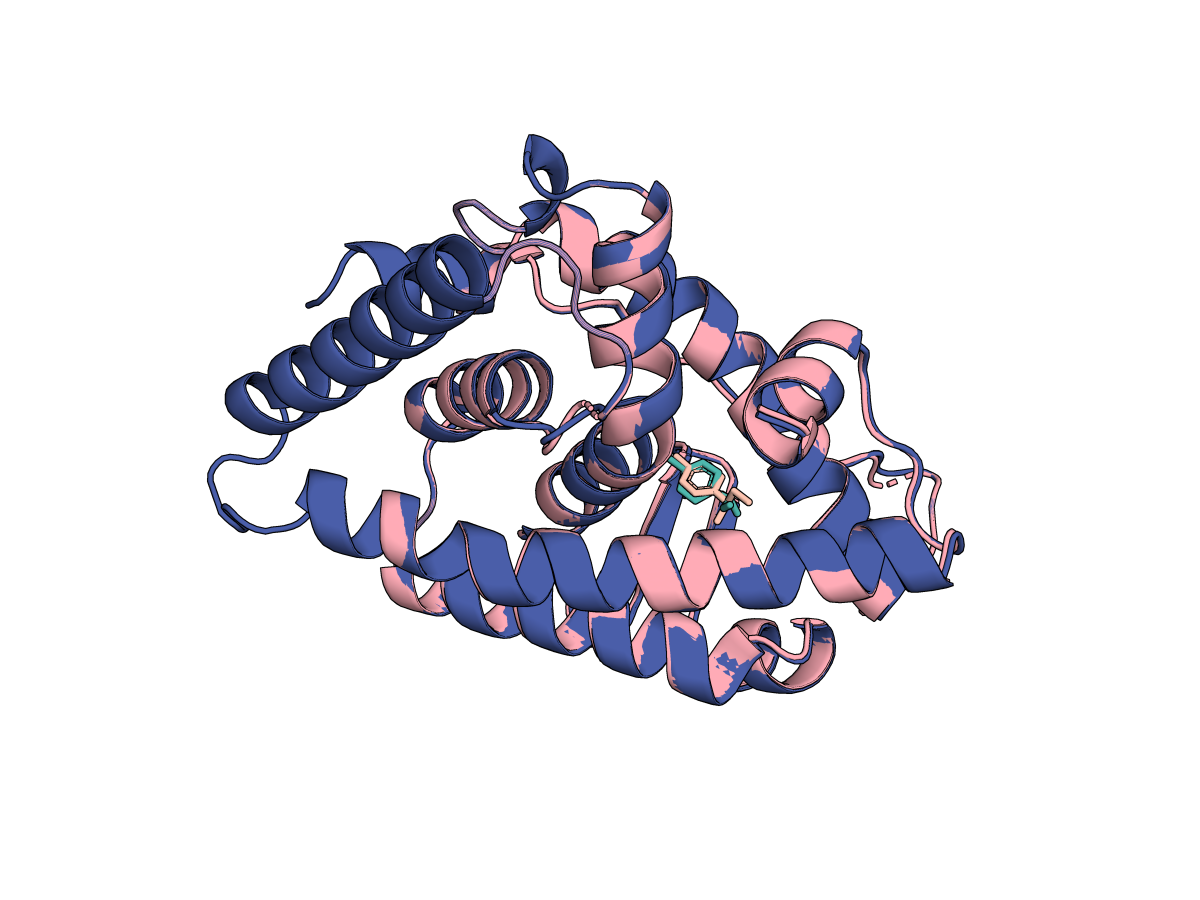}
    \caption{6i66}
  \end{subfigure}
  \hfill
  \begin{subfigure}[t]{0.45\textwidth}
    \centering
    \includegraphics[width=\linewidth]{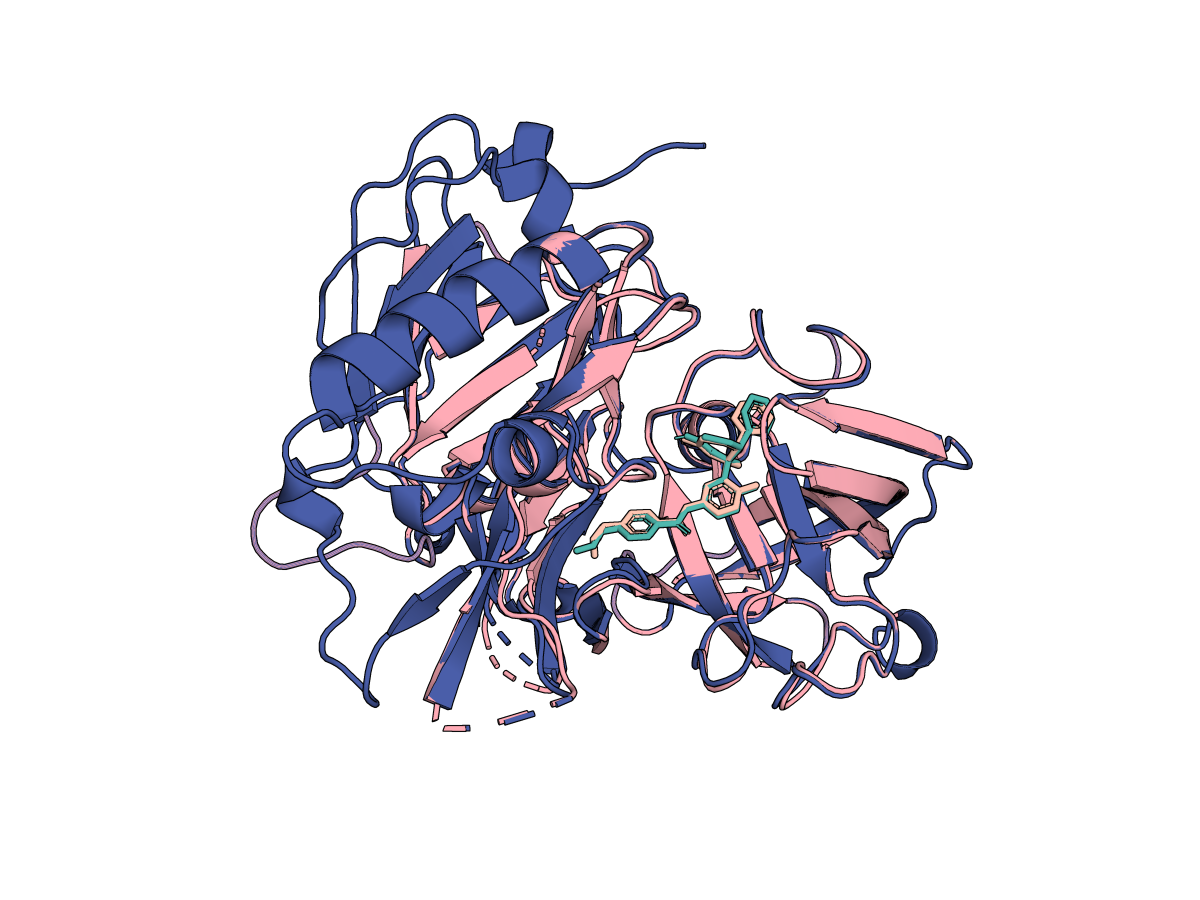}
    \caption{6jse}
  \end{subfigure}
  \hfill
  \begin{subfigure}[t]{0.45\textwidth}
    \centering
    \includegraphics[width=\linewidth]{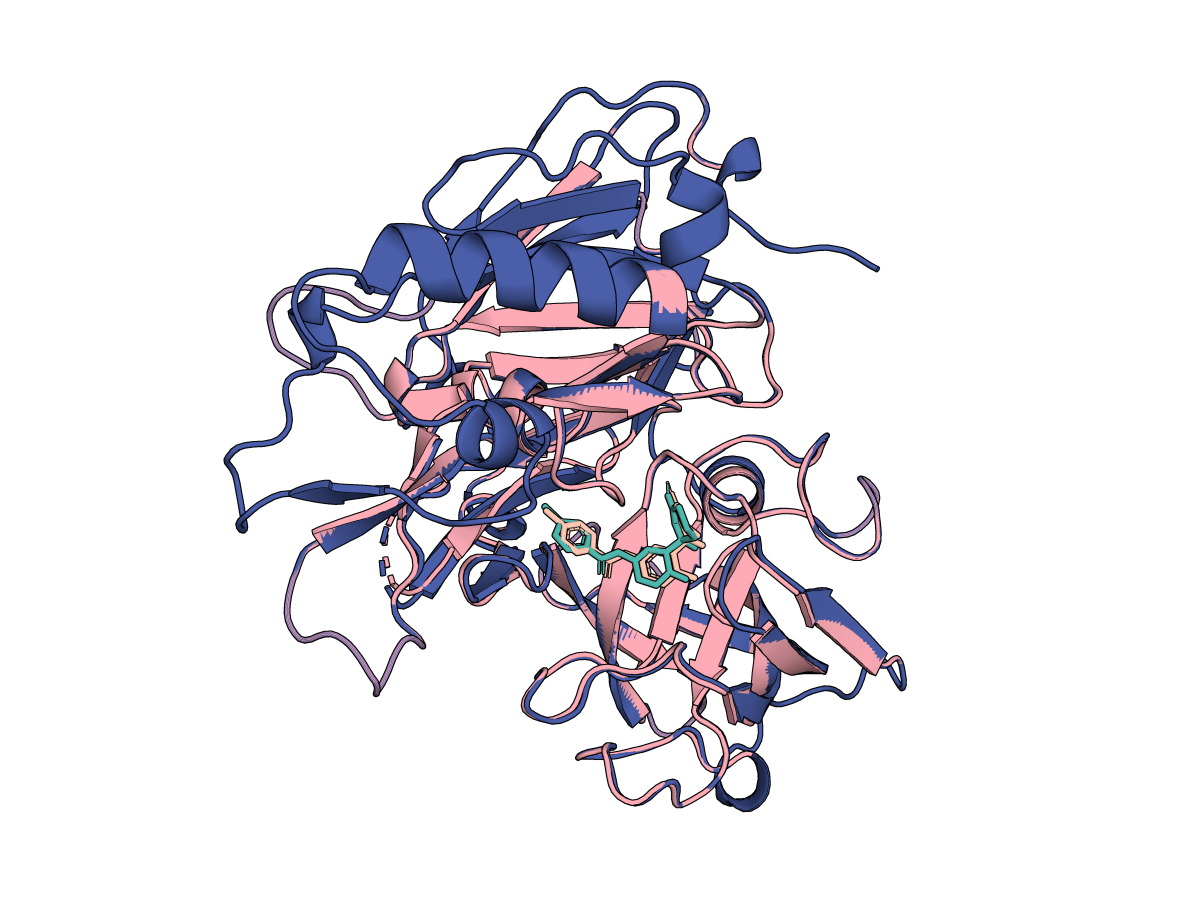}
    \caption{6jsg}
  \end{subfigure}
  \hfill
  \begin{subfigure}[t]{0.45\textwidth}
    \centering
    \includegraphics[width=\linewidth]{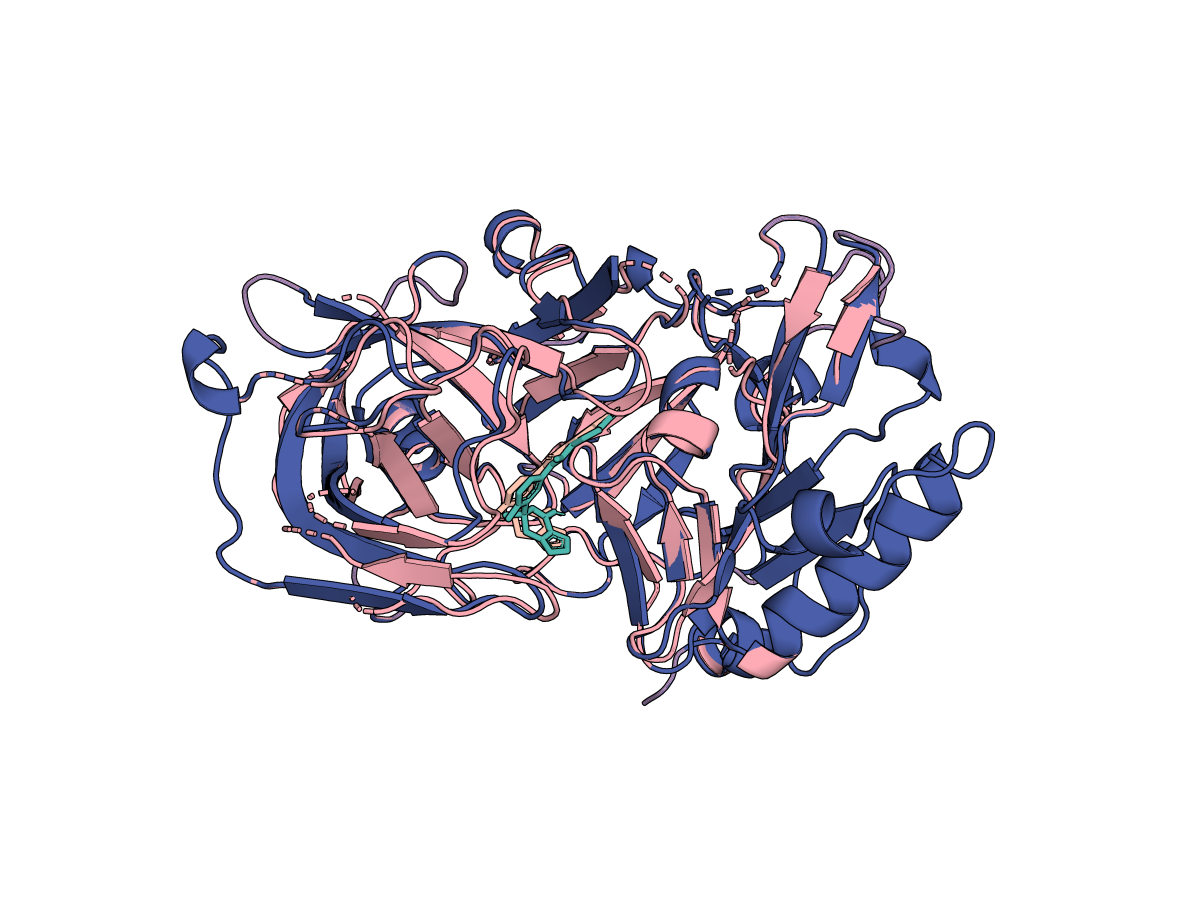}
    \caption{6od6}
  \end{subfigure}
\caption{Visualization for the generated complexes(6HHR, 6I65, 6I66, 6JSE, 6JSG, and 6OD6) from the PDBBind test dataset. Predicted protein structures are shown in pink, predicted ligands in teal, ground-truth proteins in dark blue, and ground-truth ligands in melon.}
  \label{fig:vis_protein_sample}
\end{figure}
Figure~\ref{fig:vis_protein_sample} illustrates representative docking results on six PDBBind test complexes, highlighting the alignment between predicted and ground-truth protein–ligand structures.

{Table~\ref{tab:sensitivity-alpha} shows that varying the probability scheduler parameter $n$ has a negligible impact on performance. This demonstrates the method's robustness to the choice of scheduler.

Table~\ref{tab:PDBBind_extracted} presents the mean $\pm$ standard deviation results from 5 independent experimental runs, validating docking methods on the PDBBind dataset. CCFM consistently demonstrates superior performance and stability compared to FlexDock and RP.}

\subsection{PDE solution generation}
\label{app:pde_additional_res}

We report additional results of the PDE solution generation tasks. This section
presents the running time for different methods, sensitivity analysis of the probability scheduler, MMSE curves of the generated solution fields varying with physical time, the visualization of SMSE and MMSE with all time frames (Navier--Stokes), the pointwise average residuals varying with physical time, and the comparison of models using an example of the generated solution field.

Table~\ref{tab:ccfm-runtime} shows that CCFM runs faster than PCFM while still ensuring feasibility. The gain is most pronounced in Reaction–Diffusion, where nonlinear projections slow down PCFM. By adaptively relaxing constraints with the chance-constrained operator, CCFM often has smaller active corrections due to relaxation, which avoids unnecessary projections and reduces computational cost.  
\begin{table}[t]
  \caption{Comparison of running time between CCFM and PCFM on Reaction--Diffusion and Navier--Stokes equations. }
  \centering
  \resizebox{0.6\linewidth}{!}{%
  \begin{tabular}{
    ll
    c % CCFM
    c % PCFM
  }
    \toprule
    \textbf{Problem} & \textbf{Metric} & \textbf{CCFM (ours)} & \textbf{PCFM} \\
    \midrule
    \multirow{1}{*}{Reaction--Diffusion}
      & Running Time (s) ($\downarrow$) & \multicolumn{1}{c}{2.703} & \multicolumn{1}{c}{3.874} \\
    \midrule
    \multirow{1}{*}{Navier--Stokes}
      & Running Time (s) ($\downarrow$) & \multicolumn{1}{c}{3.065} & \multicolumn{1}{c}{3.077} \\
    \bottomrule
  \end{tabular}%
  }
 \label{tab:ccfm-runtime}
\end{table}

\subsubsection{Reaction--Diffusion}
\begin{table}[t]
  \centering
  \caption{Sensitivity analysis for probability of constraint satisfaction. 
  The probability is determined by the function $\phi_n(t) = \frac{1}{2}t^n$, 
  and we adjust $n$ to observe the results.}
  \resizebox{0.95\linewidth}{!}{%
  \begin{tabular}{
    ll
    c % CCFM
    c % PCFM
    c % DiffusionPDE
    c % FFM
    c % FlexDock
  }
    \toprule
    \textbf{Problem} & \textbf{Metric} & \textbf{CCFM (0.1)} & \textbf{CCFM (0.3)} & \textbf{CCFM (0.5)} & \textbf{CCFM (0.7)} & \textbf{CCFM (0.9)} \\
    \midrule
    \multirow{4}{*}{Reaction--Diffusion}
      & MMSE $\times 10^{-2}$ ($\downarrow$)   & 3.6 & 3.6 & 3.3 & 3.2 & 3.2 \\
      & SMSE $\times 10^{-2}$ ($\downarrow$)   & 2.8 & 2.9 & 2.7 & 2.6 & 2.6 \\
      & CV (IC)  $\times 10^{-2}$ ($\downarrow$) & 0 & 0 & 0 & 0 & 0 \\
      & CV (CL)  $\times 10^{-15}$ ($\downarrow$) & 9.9 & 9.6 & 9.5 & 9.5 & 9.5 \\
    \bottomrule
  \end{tabular}%
  }
 \label{tab:sensitivity-phi}
\end{table}

Table~\ref{tab:sensitivity-phi} presents the sensitivity analysis of the probability scheduler hyperparameter. We observe that varying the parameter $n$ has almost no impact on constraint satisfaction, consistently ensuring feasibility across all settings. The effects on MMSE and SMSE are also marginal, with only slight fluctuations that remain within a reasonable range. This indicates that the method is robust to the choice of the probability scheduler.

\begin{figure}
\centering
    \begin{subfigure}{0.4\textwidth}
        \centering
        \includegraphics[width=\linewidth]{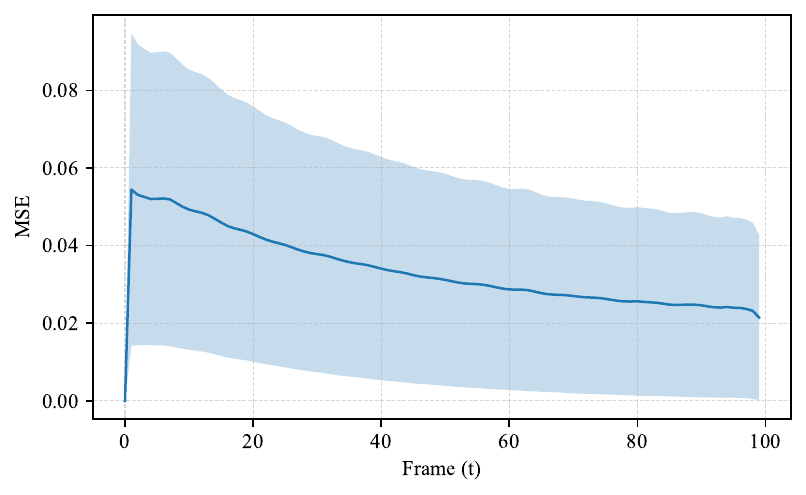}
        \caption{CCFM}
    \end{subfigure}
    \begin{subfigure}{0.4\textwidth}
        \centering
        \includegraphics[width=\linewidth]{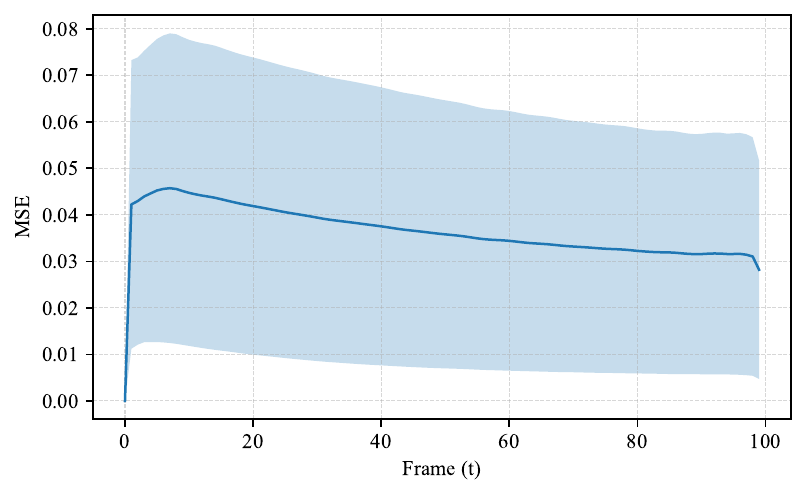}
        \caption{PCFM}
    \end{subfigure}
    
    \begin{subfigure}{0.4\textwidth}
        \centering
        \includegraphics[width=\linewidth]{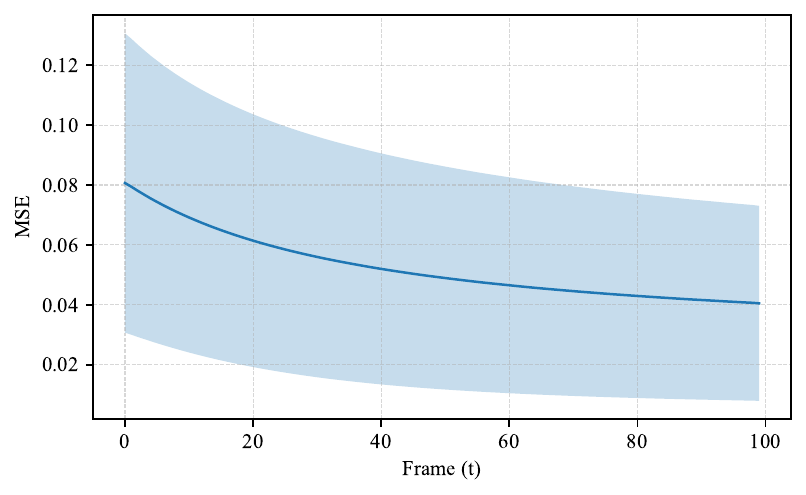}
        \caption{DPDE}
    \end{subfigure}
    \begin{subfigure}{0.4\textwidth}
        \centering
        \includegraphics[width=\linewidth]{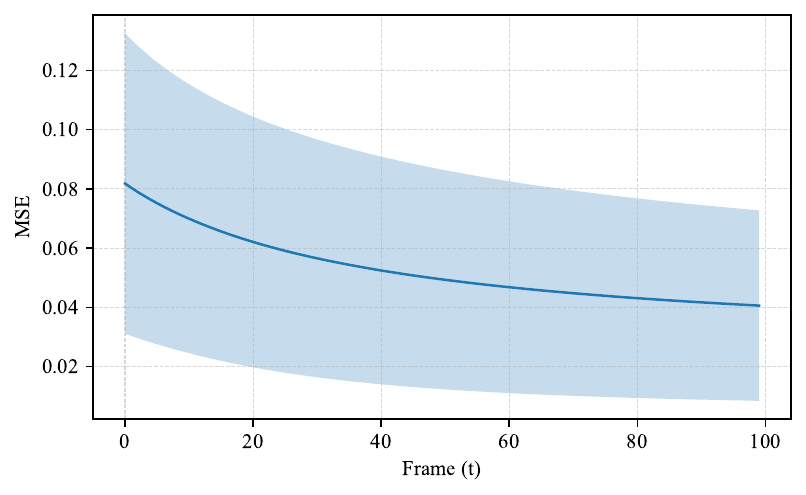}
        \caption{FFM}
    \end{subfigure}
    \caption{Comparison of generated solutions for the Reaction--Diffusion equation in terms of MMSE varying with physical time.}
    \label{fig:rd-mmse-time}
\end{figure}

\begin{figure}[t]
  \centering
  \begin{subfigure}[t]{0.23\textwidth}
    \centering
    \includegraphics[width=\textwidth]{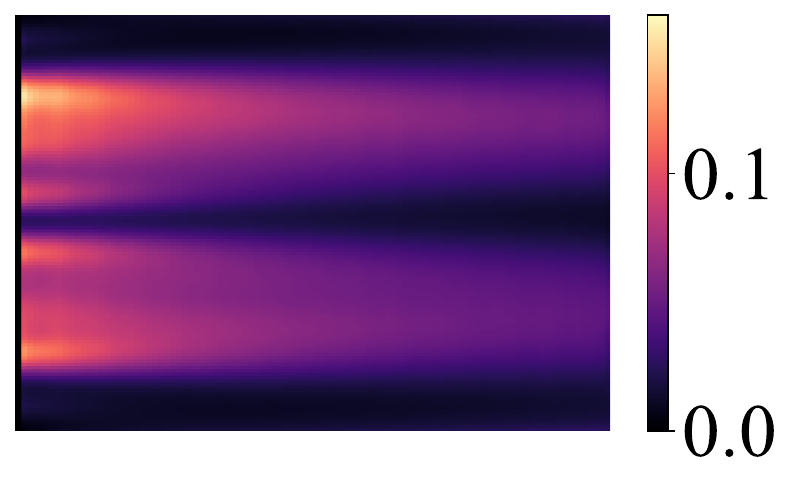}
    \caption{CCFM}
  \end{subfigure}
  \hfill
  \begin{subfigure}[t]{0.23\textwidth}
    \centering
    \includegraphics[width=\textwidth]{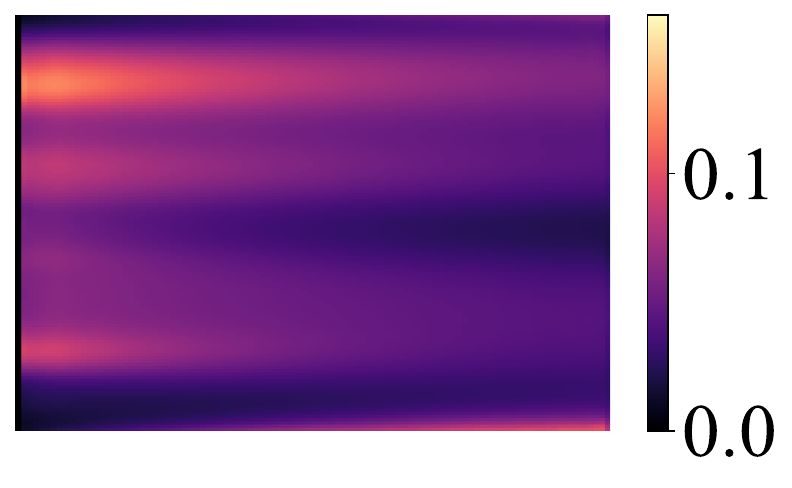}
    \caption{PCFM}
  \end{subfigure}
  \hfill
  \begin{subfigure}[t]{0.23\textwidth}
    \centering
    \includegraphics[width=\textwidth]{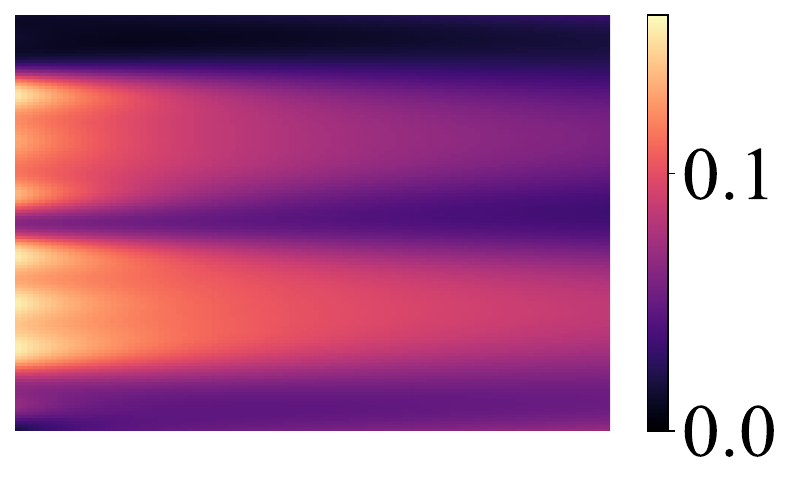}
    \caption{DPDE}
  \end{subfigure}
  \hfill
  \begin{subfigure}[t]{0.23\textwidth}
    \centering
    \includegraphics[width=\textwidth]{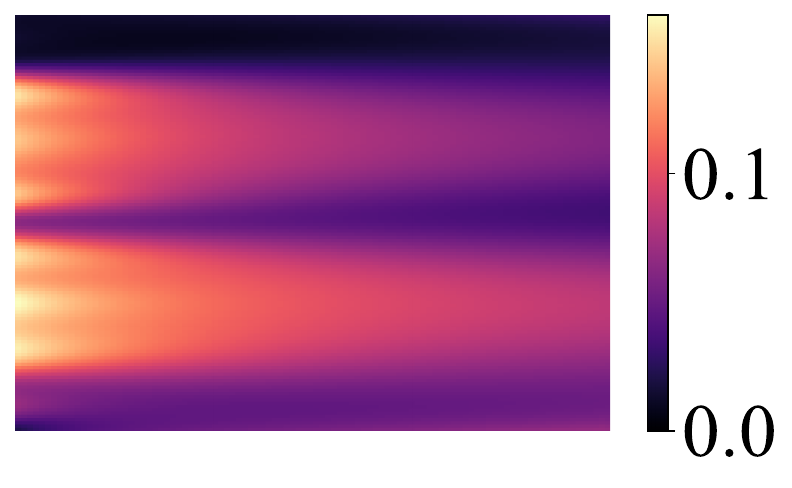}
    \caption{FFM}
  \end{subfigure}
  \caption{SMSE of Reaction--Diffusion solutions over physical time, with the horizontal axis representing time evolution and the vertical axis indicating the state at each time step.}
  \label{fig:rd_smse}
\end{figure}

\begin{figure}
    \centering
    \includegraphics[width=.8\linewidth]{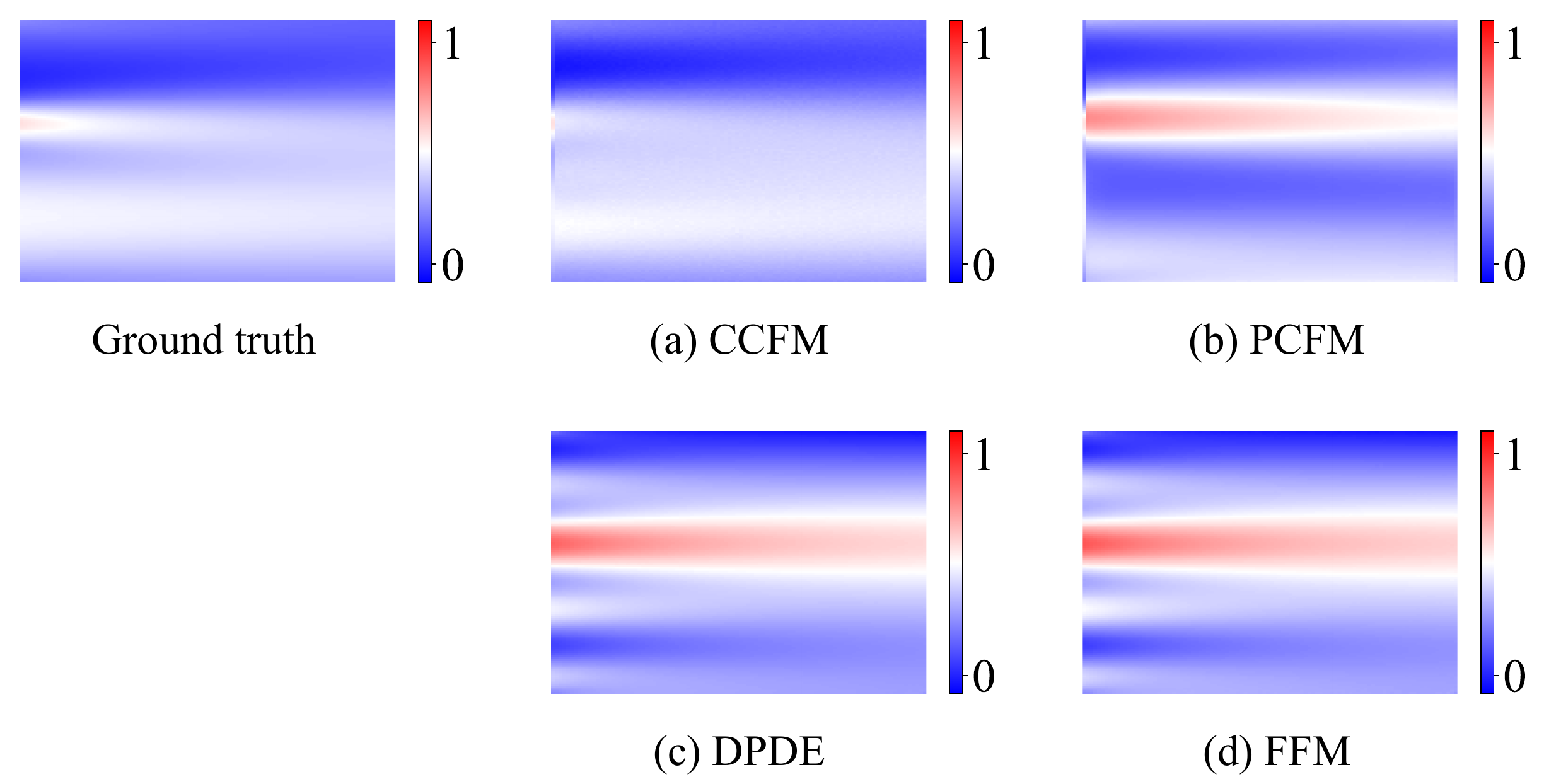}
    \caption{Example solution fields of the Reaction--Diffusion equation generated using the models compared with the ground truth.}
    \label{fig:rd_sample}
\end{figure}

Figure~\ref{fig:rd-mmse-time} shows the MMSE curves as a function of physical
time for the Reaction--Diffusion equation. CCFM demonstrates superior
performance compared with the other models used in this study, while exactly
satisfying the initial condition (with error equal to 0). PCFM has similar
performance.
Figure~\ref{fig:rd_smse} shows the SMSE, where the x-axis represents physical
time and the y-axis represents space. Similar to the MMSE
(Figure~\ref{fig:rd_mse}), the proposed CCFM achieves the lowest errors compared
with other models. The example solution fields in Figure~\ref{fig:rd_sample}
demonstrates the improvement from CCFM compared with other models in generating
a specific solution field. The CCFM-generated sample closely adheres to the
ground truth, while all other models show substantial discrepancies. These
additional results consistently highlight the superiority of the proposed CCFM
in generating solutions with high accuracy while ensuring constraint
satisfaction.

\subsubsection{Navier--Stokes}

\begin{figure}
\centering
    \begin{subfigure}{0.4\textwidth}
        \centering
        \includegraphics[width=\linewidth]{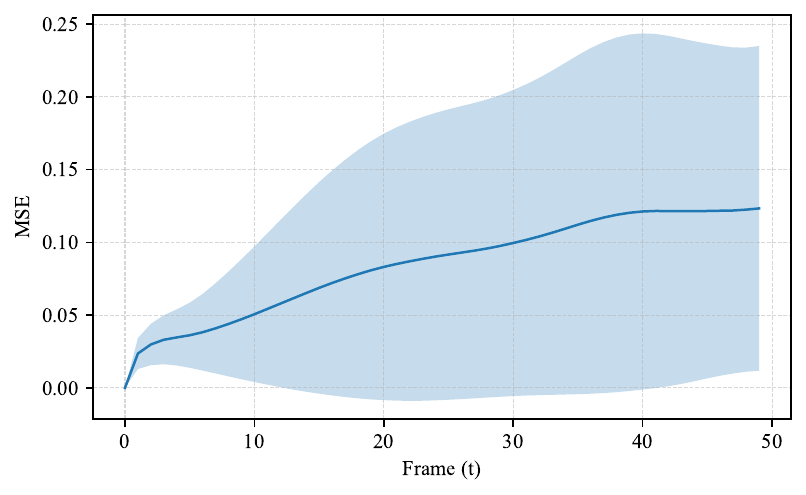}
        \caption{CCFM}
    \end{subfigure}
    \begin{subfigure}{0.4\textwidth}
        \centering
        \includegraphics[width=\linewidth]{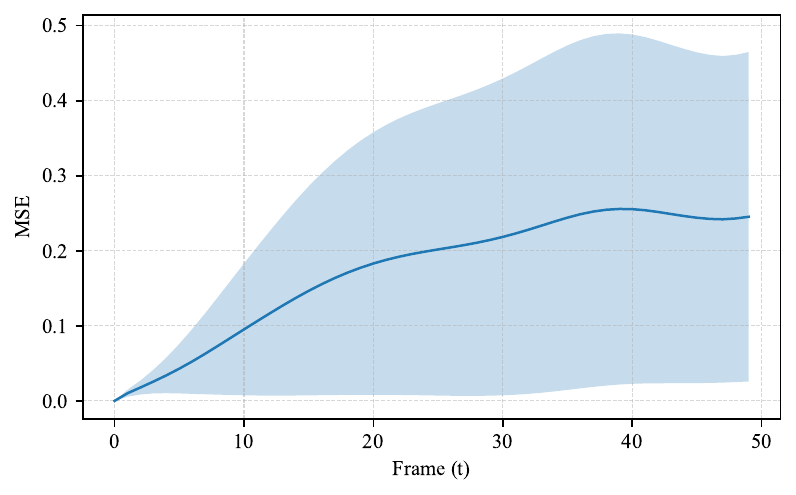}
        \caption{PCFM}
    \end{subfigure}
    
    \begin{subfigure}{0.4\textwidth}
        \centering
        \includegraphics[width=\linewidth]{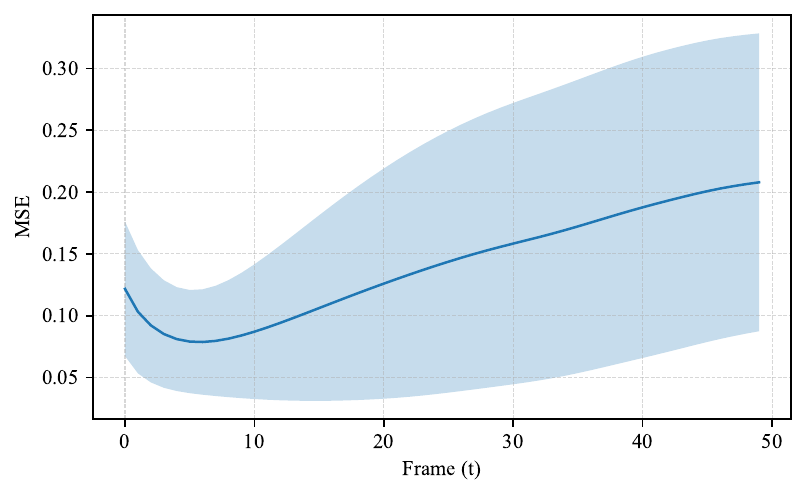}
        \caption{DPDE}
    \end{subfigure}
    \begin{subfigure}{0.4\textwidth}
        \centering
        \includegraphics[width=\linewidth]{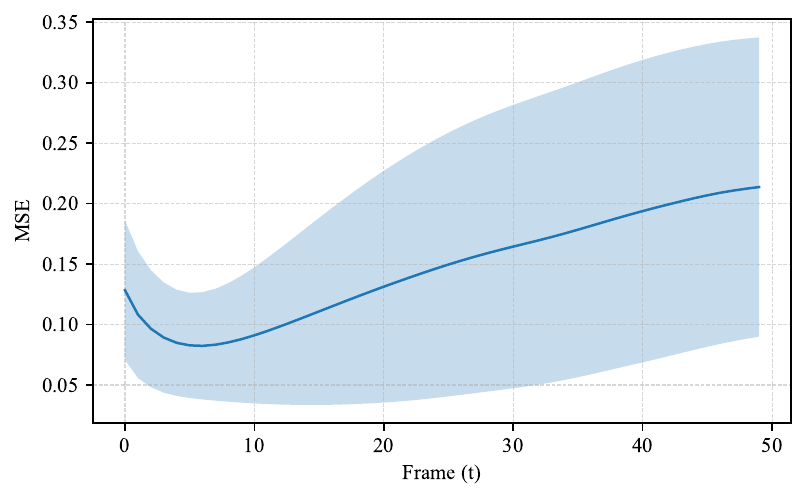}
        \caption{FFM}
    \end{subfigure}
    \caption{Comparison of generated solutions for the Navier--Stokes equation in terms of MMSE varying with physical time.}
    \label{fig:ns_mmse_time}
\end{figure}

\begin{figure}
\centering
    \begin{subfigure}{0.24\textwidth}
        \centering
        \includegraphics[width=\linewidth]{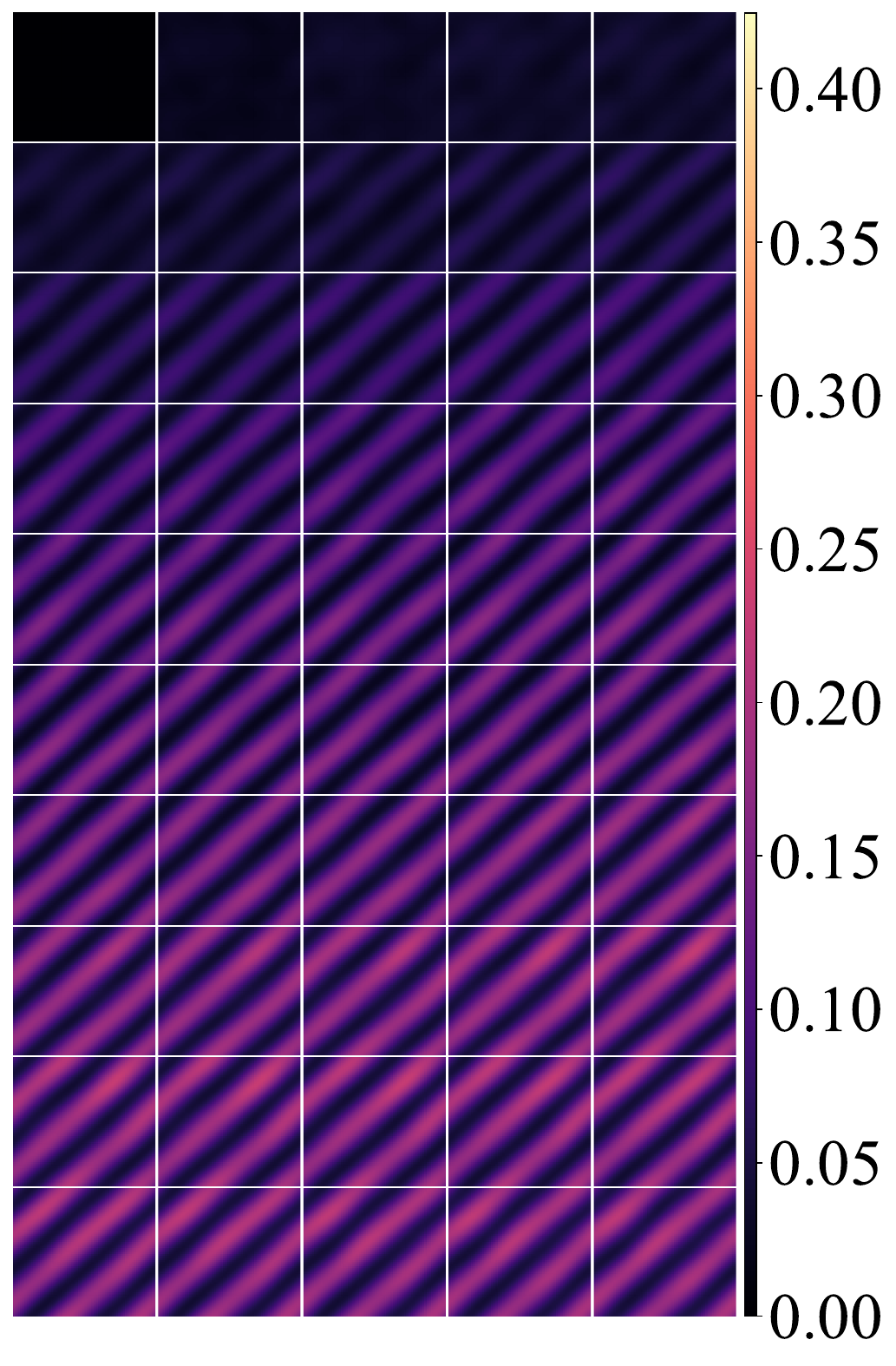}
        \caption{CCFM}
    \end{subfigure}
    \begin{subfigure}{0.24\textwidth}
        \centering
        \includegraphics[width=\linewidth]{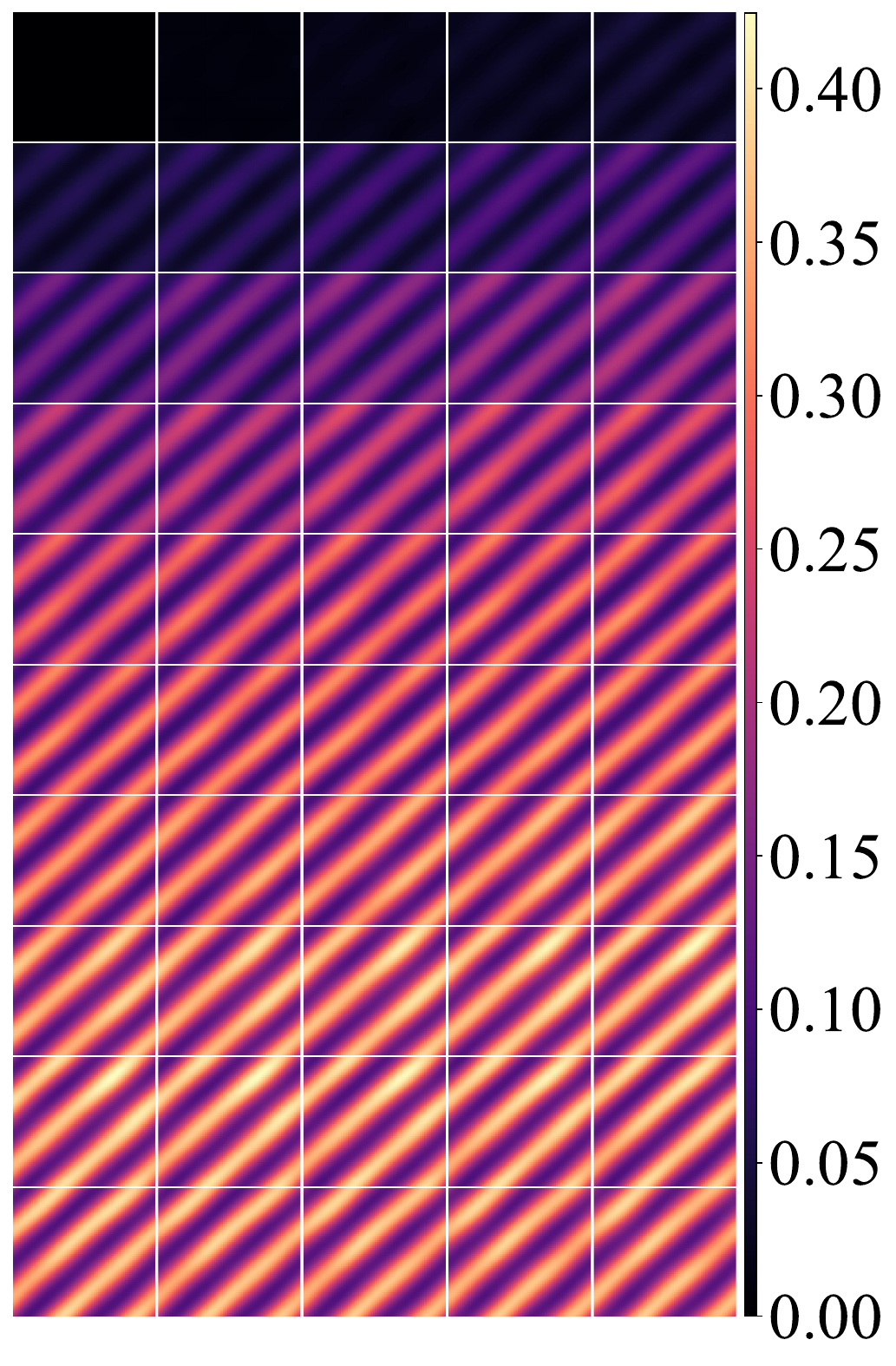}
        \caption{PCFM}
    \end{subfigure}
    \begin{subfigure}{0.24\textwidth}
        \centering
        \includegraphics[width=\linewidth]{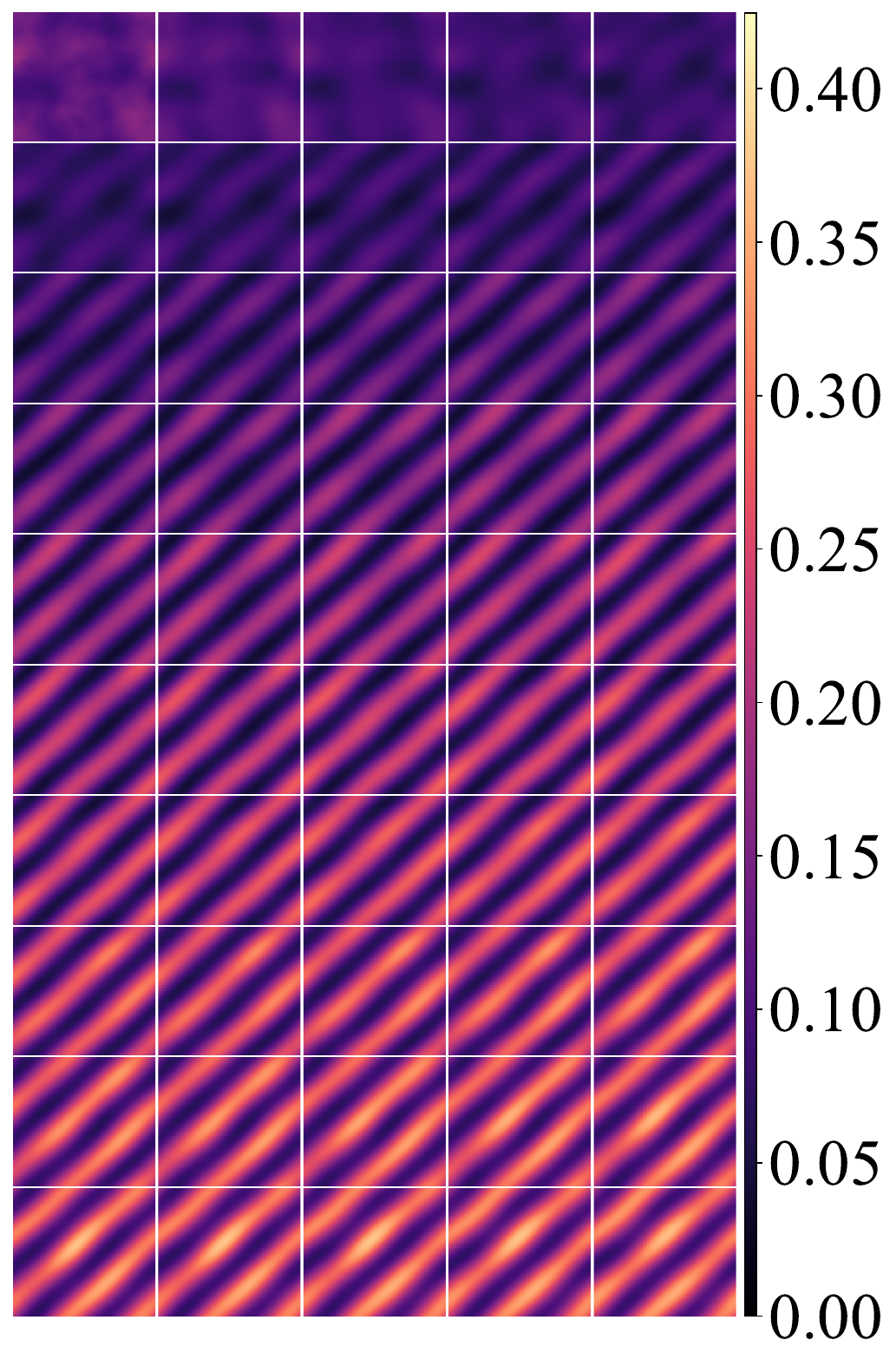}
        \caption{DPDE}
    \end{subfigure}
    \begin{subfigure}{0.24\textwidth}
        \centering
        \includegraphics[width=\linewidth]{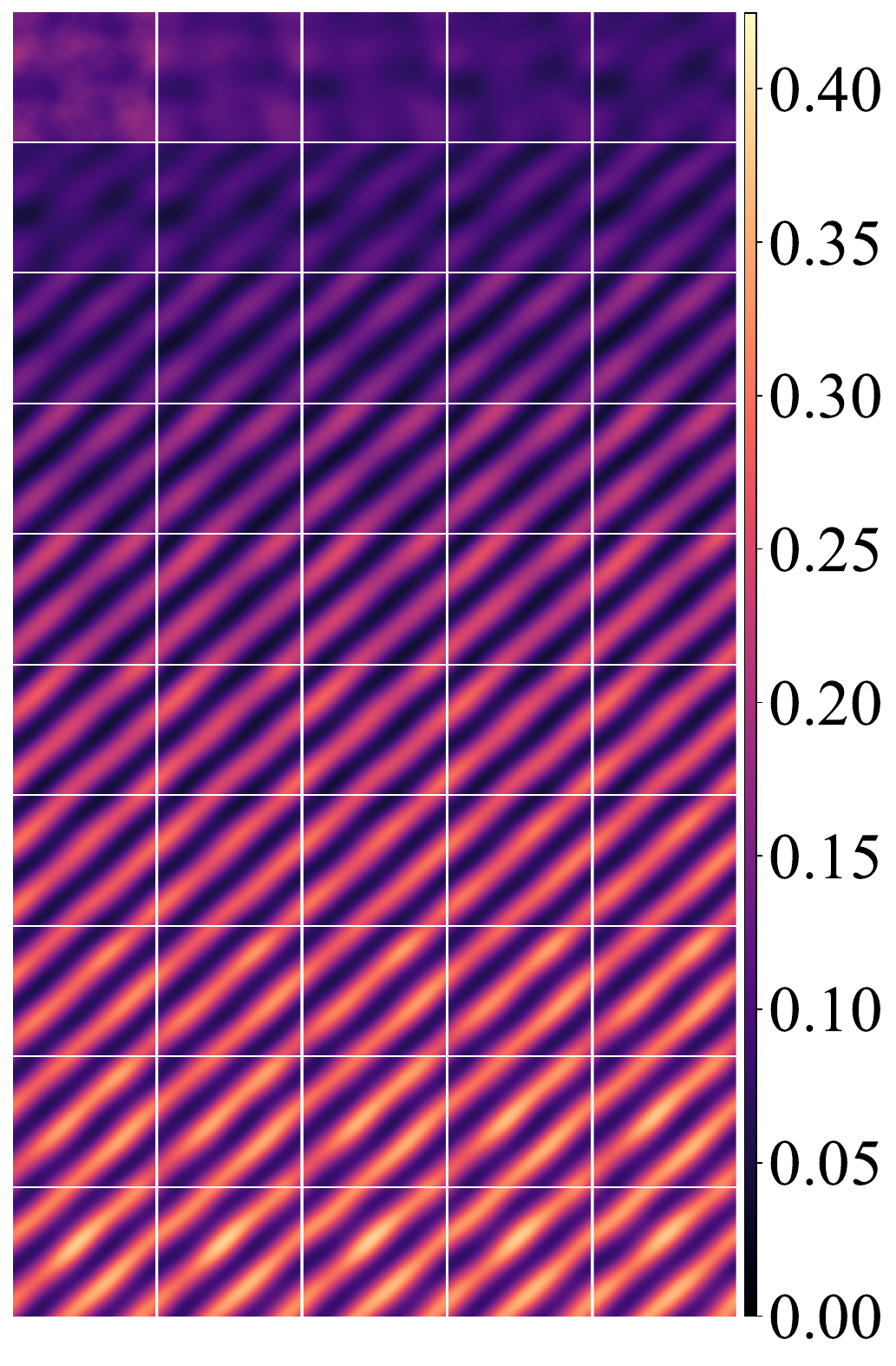}
        \caption{FFM}
    \end{subfigure}
    \caption{MMSE (the darker the better) of Navier--Stokes solutions over physical time with all 50 frames. Each frame corresponds to the system state at a given physical time.}
    \label{fig:ns_MMSE_all_frames}
\end{figure}

\begin{figure}
\centering
    \begin{subfigure}{0.24\textwidth}
        \centering
        \includegraphics[width=\linewidth]{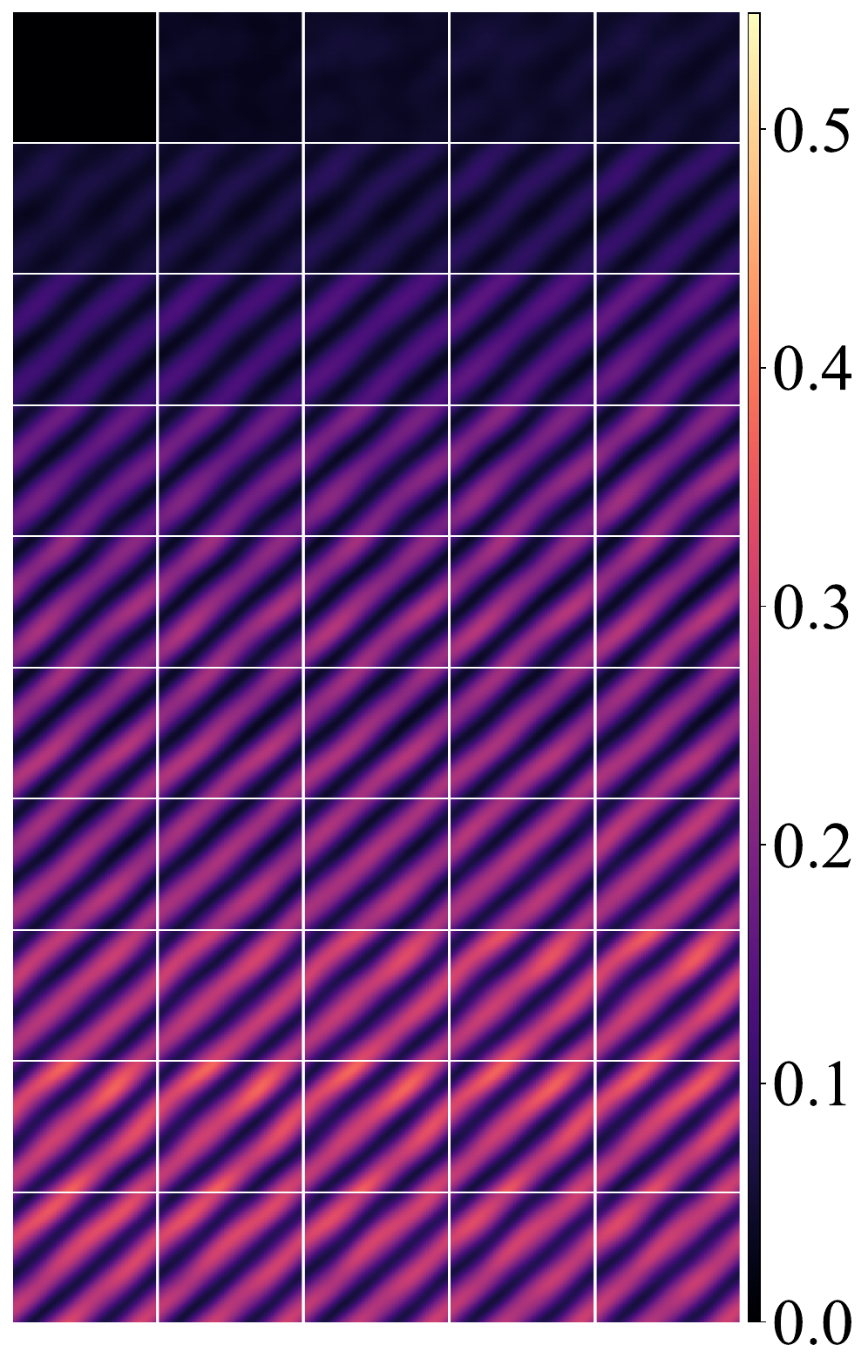}
        \caption{CCFM}
    \end{subfigure}
    \begin{subfigure}{0.24\textwidth}
        \centering
        \includegraphics[width=\linewidth]{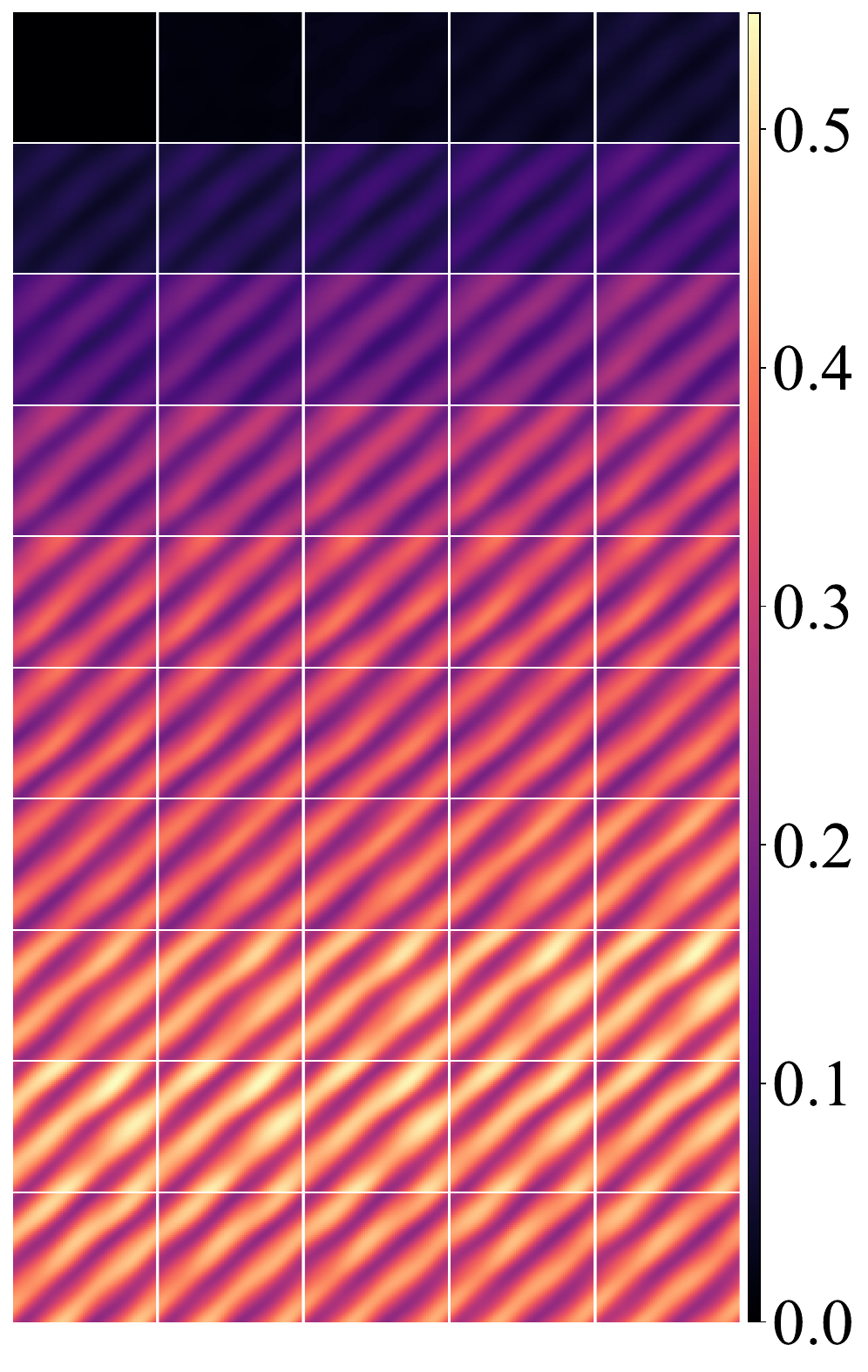}
        \caption{PCFM}
    \end{subfigure}
    \begin{subfigure}{0.24\textwidth}
        \centering
        \includegraphics[width=\linewidth]{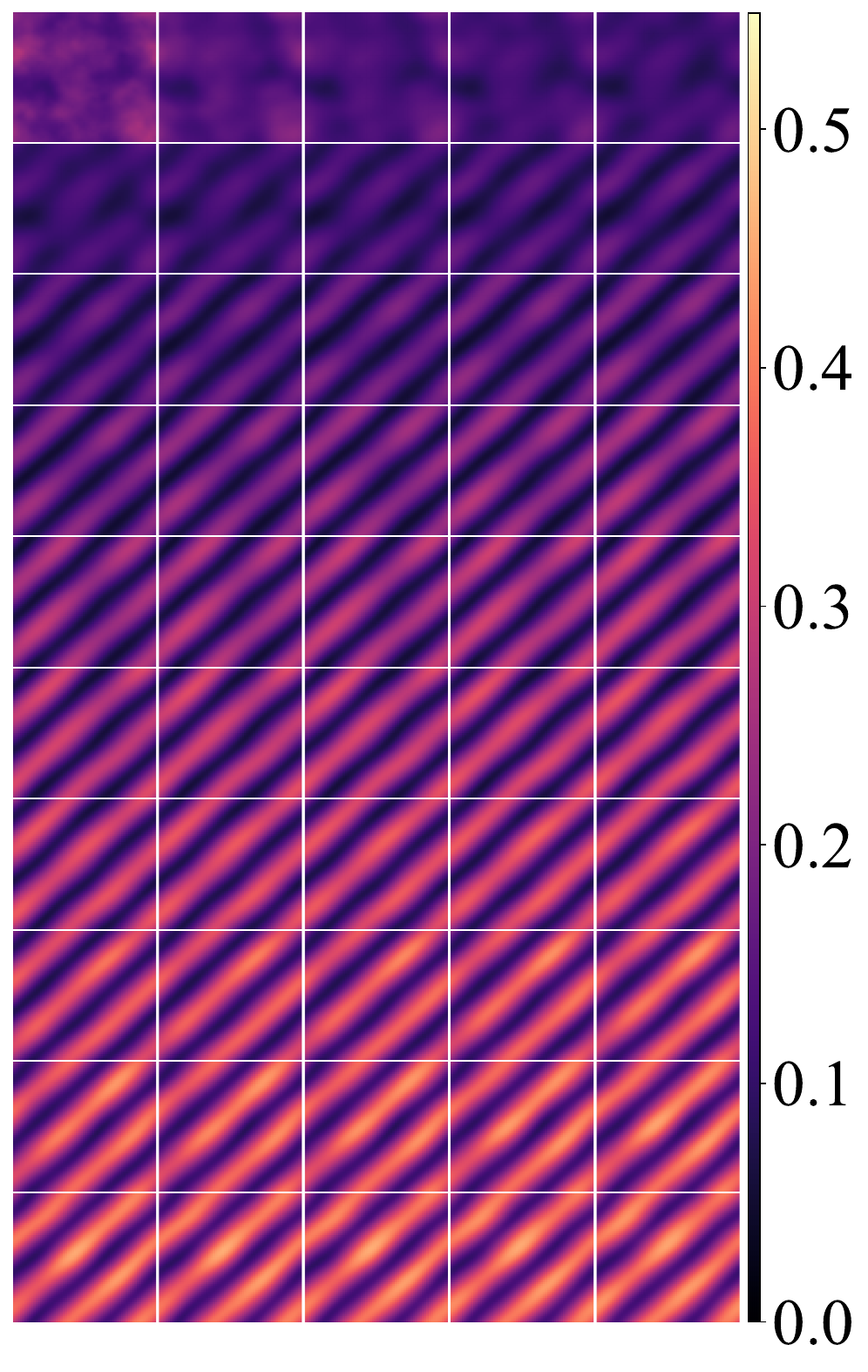}
        \caption{DPDE}
    \end{subfigure}
    \begin{subfigure}{0.24\textwidth}
        \centering
        \includegraphics[width=\linewidth]{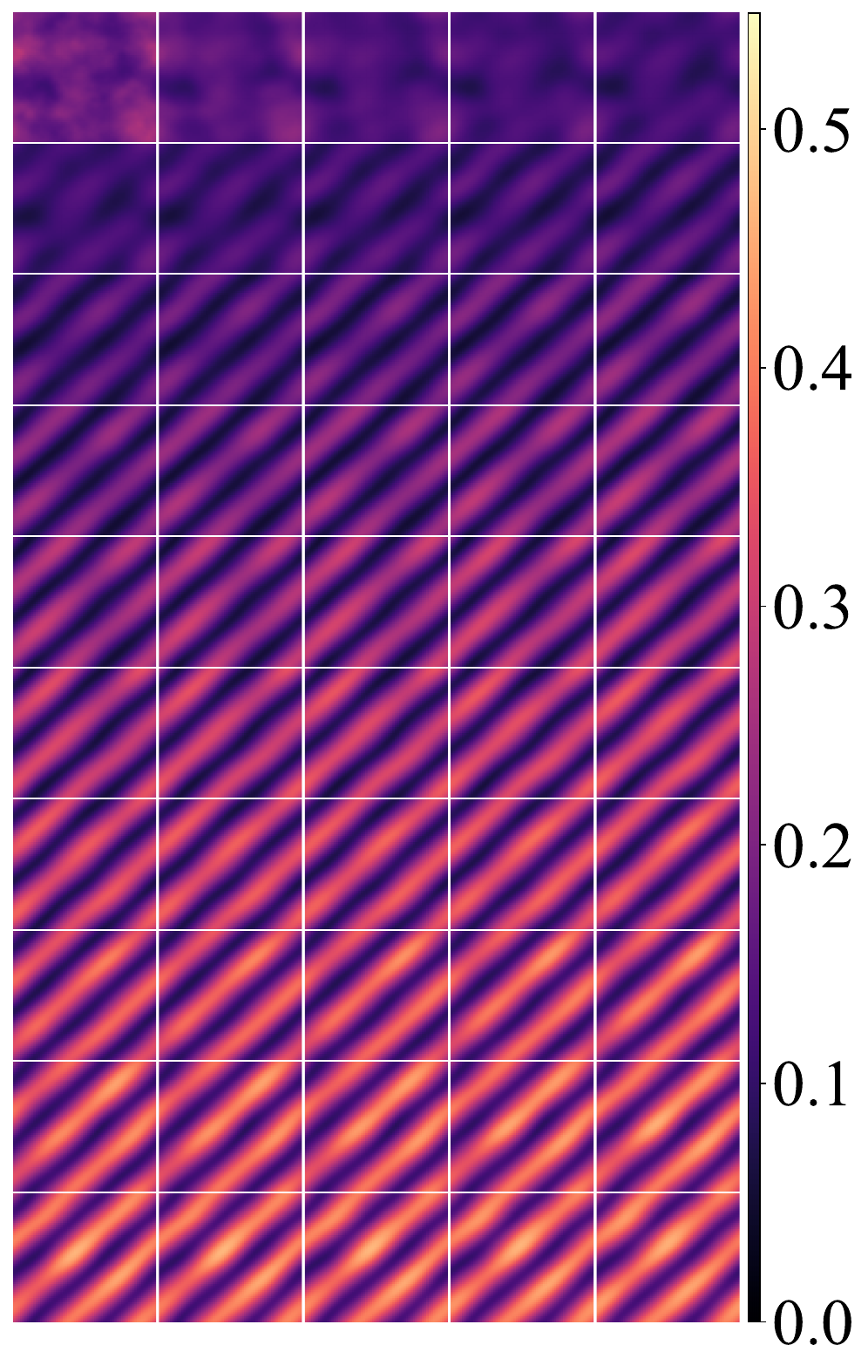}
        \caption{FFM}
    \end{subfigure}
    \caption{SMSE (the darker the better) of Navier--Stokes solutions over physical time with all 50 frames. Each frame corresponds to the system state at a given physical time.}
    \label{fig:ns_SMSE_all_frames}
\end{figure}

\begin{figure}
\centering
    \begin{subfigure}{0.24\textwidth}
        \centering
        \includegraphics[width=\linewidth]{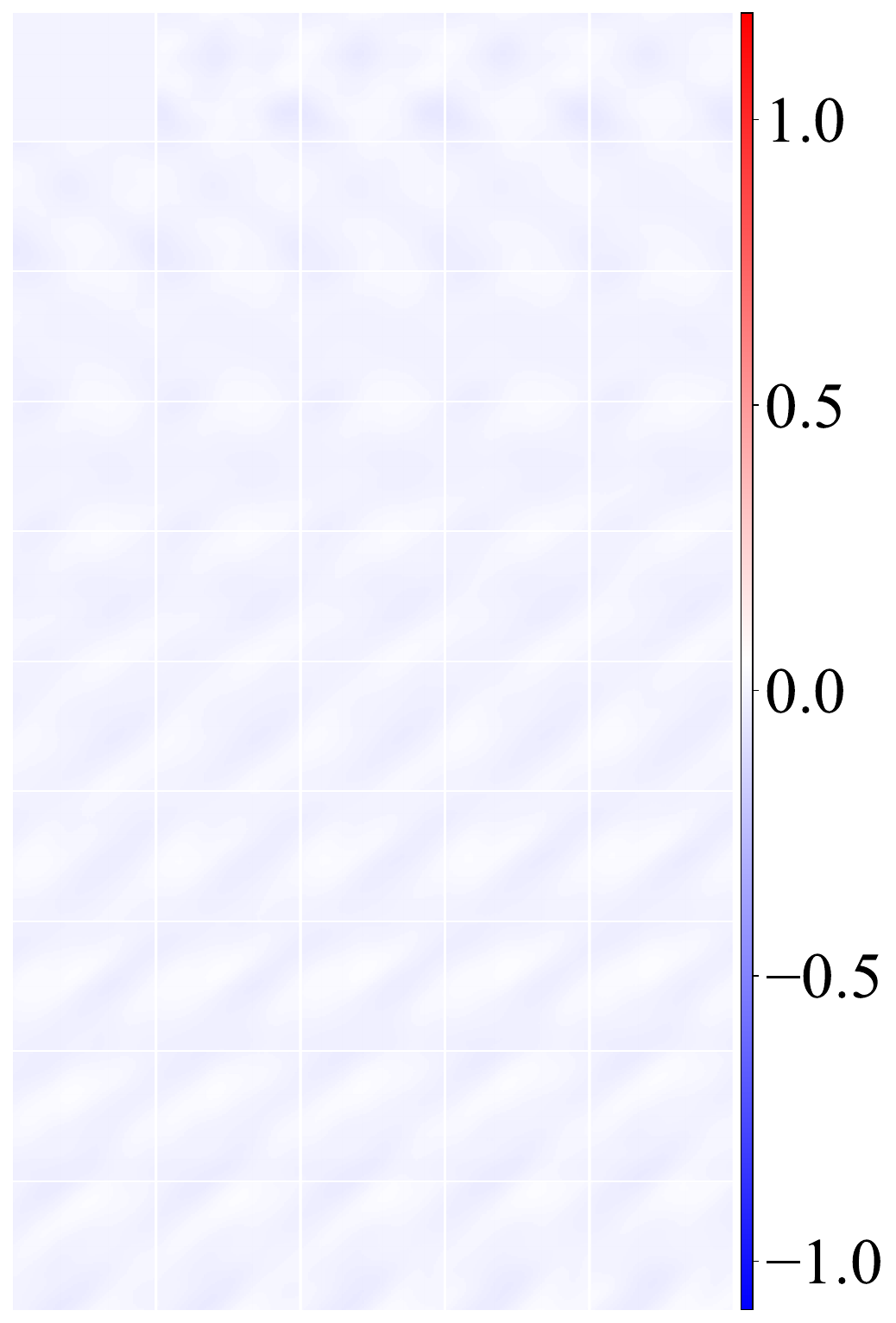}
        \caption{CCFM}
    \end{subfigure}
    \begin{subfigure}{0.24\textwidth}
        \centering
        \includegraphics[width=\linewidth]{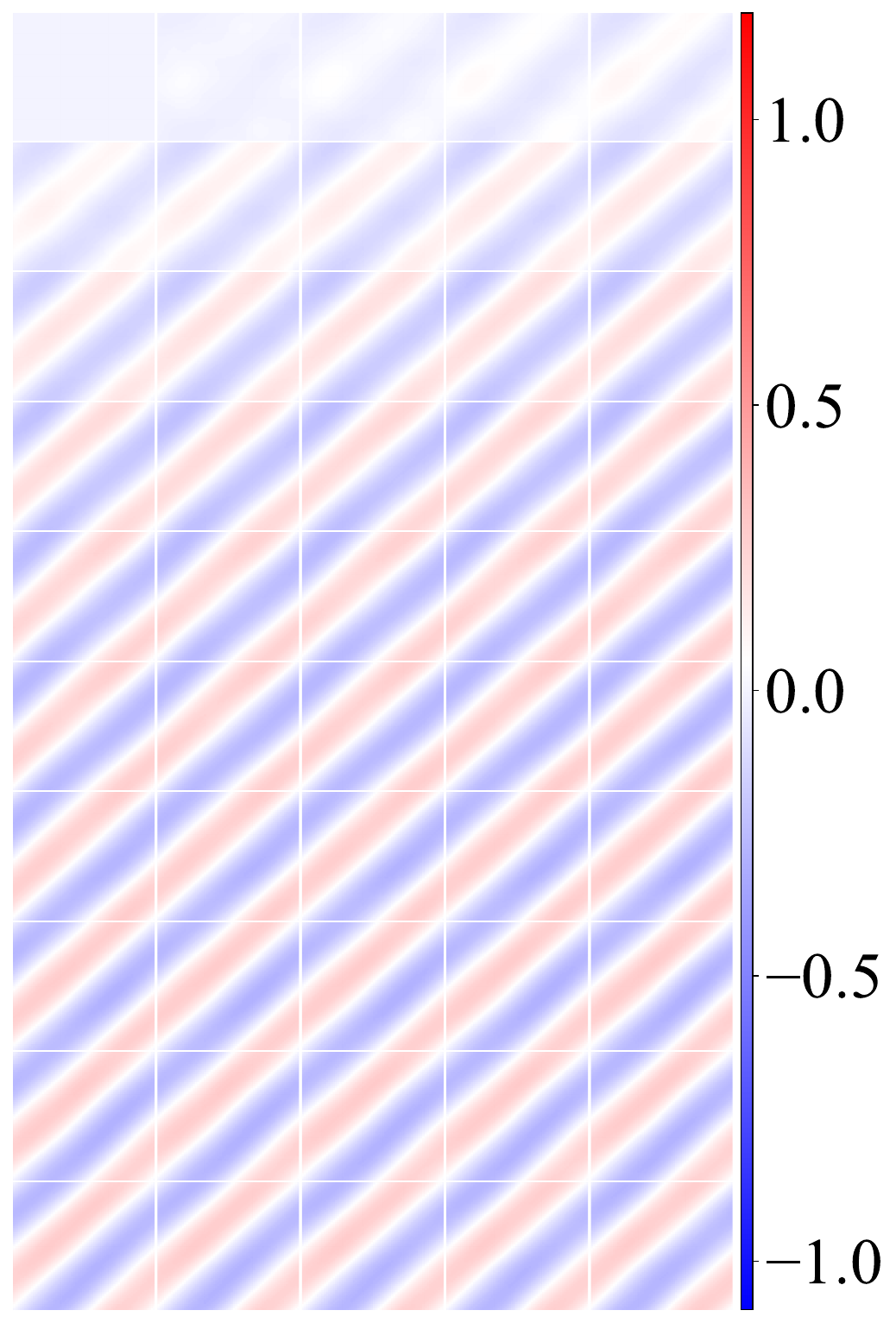}
        \caption{PCFM}
    \end{subfigure}
    \begin{subfigure}{0.24\textwidth}
        \centering
        \includegraphics[width=\linewidth]{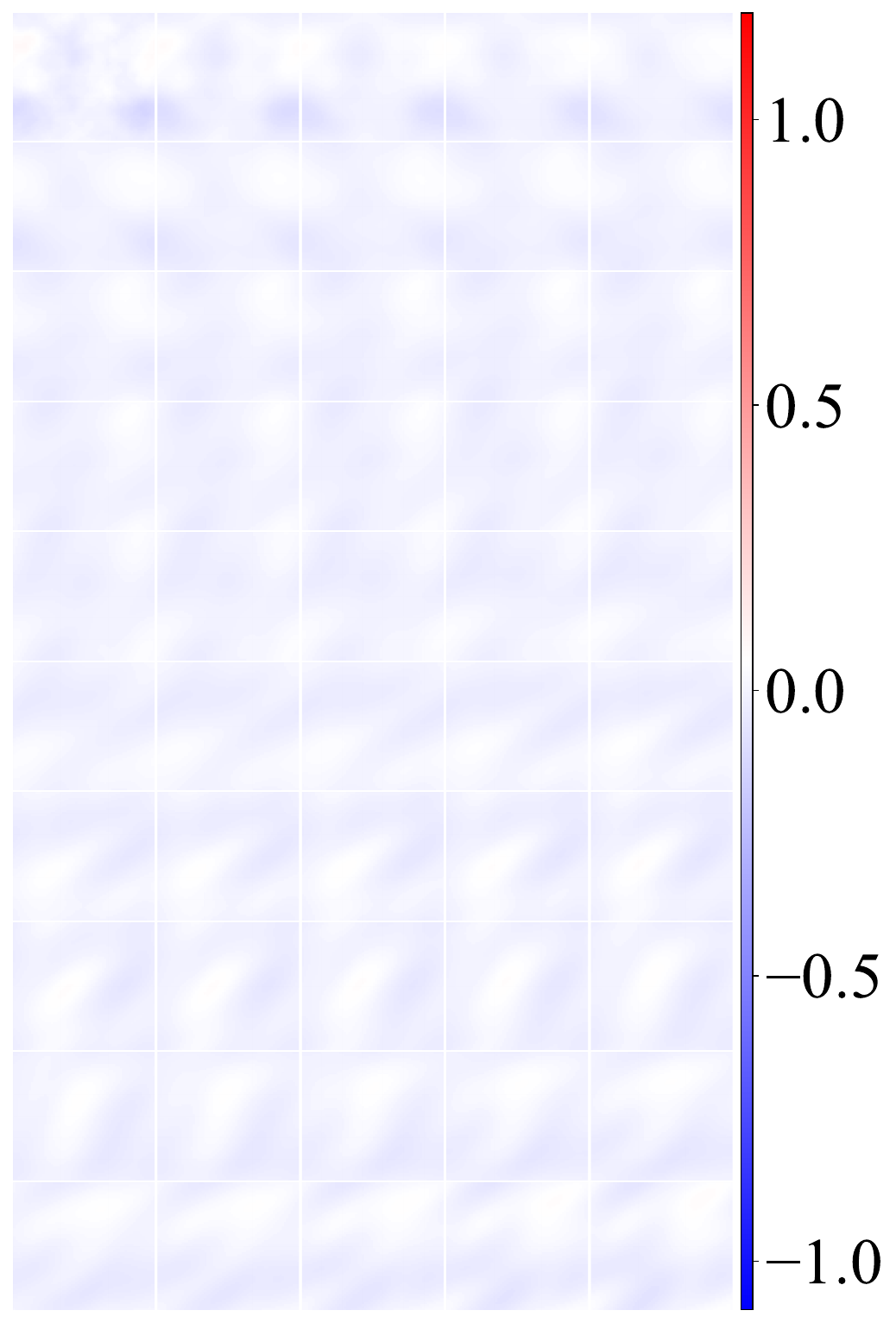}
        \caption{DPDE}
    \end{subfigure}
    \begin{subfigure}{0.24\textwidth}
        \centering
        \includegraphics[width=\linewidth]{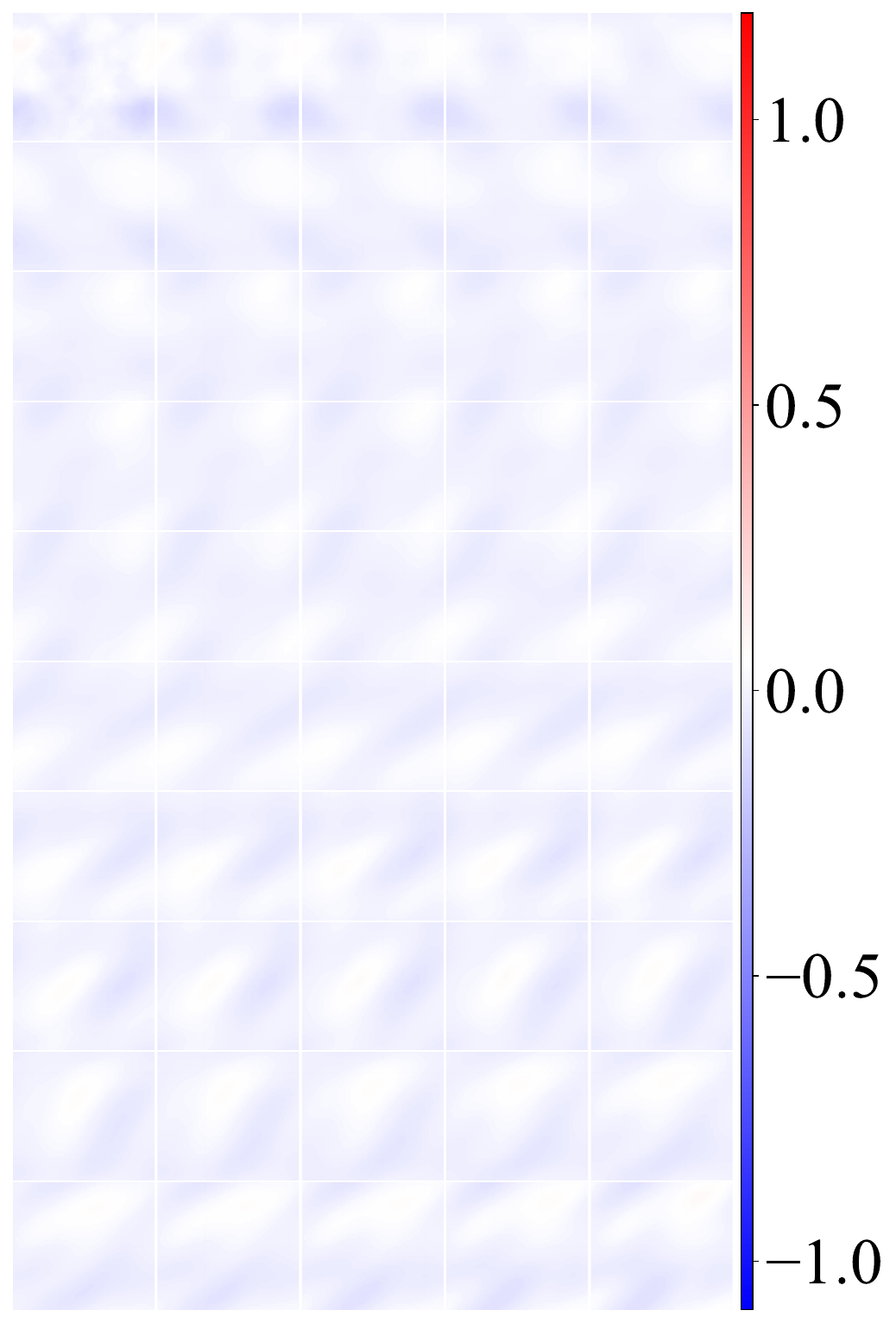}
        \caption{FFM}
    \end{subfigure}
    \caption{Point-wise averaged residuals (the lighter the better) of Navier--Stokes solutions over physical time with all 50 frames. Each frame corresponds to the system state at a given physical time.}
    \label{fig:ns_residual}
\end{figure}

\begin{figure}
\centering
    \begin{subfigure}{0.19\textwidth}
        \centering
        \includegraphics[width=\linewidth]{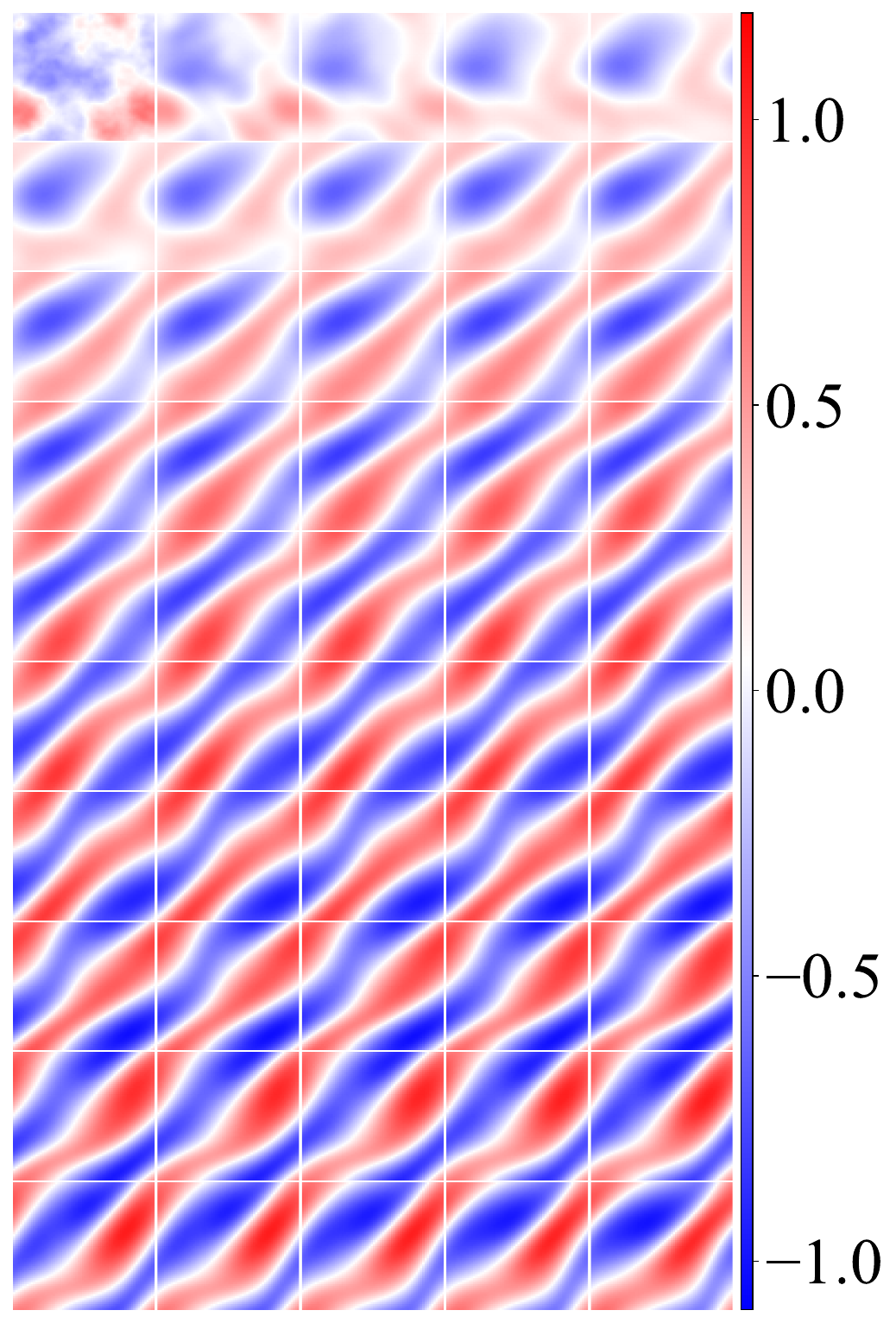}
        \caption{Ground truth}
    \end{subfigure}
    \begin{subfigure}{0.19\textwidth}
        \centering
        \includegraphics[width=\linewidth]{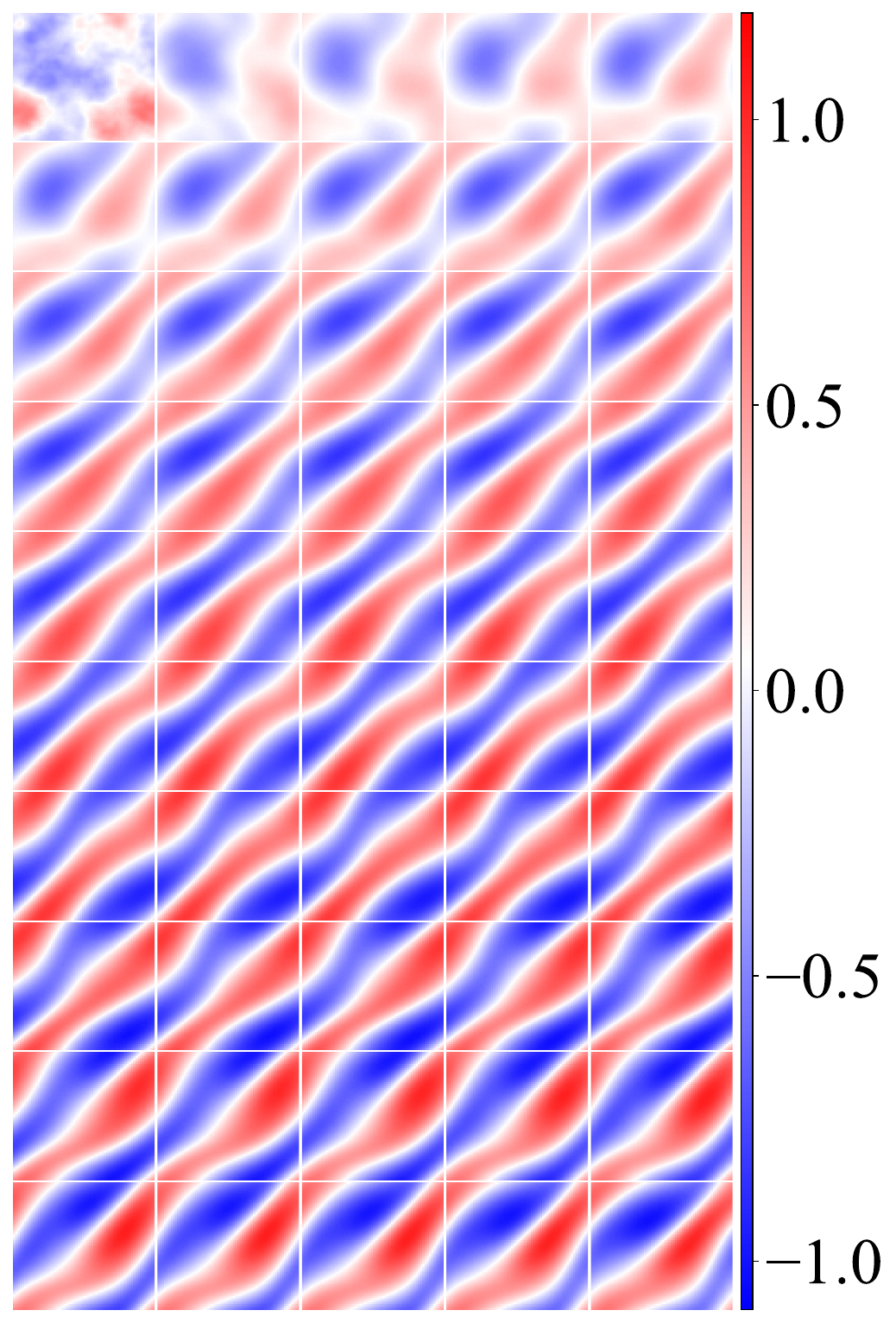}
        \caption{CCFM}
    \end{subfigure}
    \begin{subfigure}{0.19\textwidth}
        \centering
        \includegraphics[width=\linewidth]{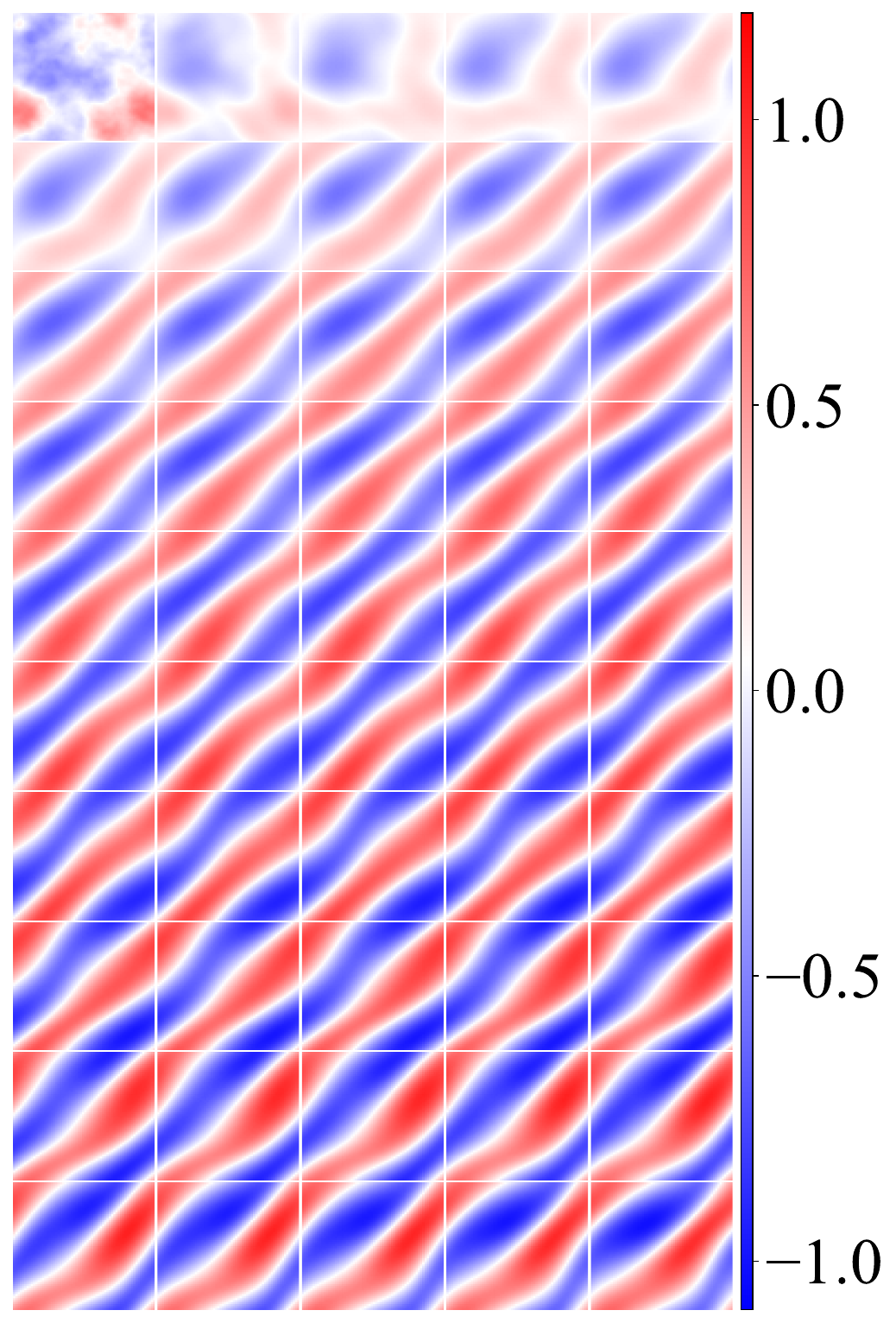}
        \caption{PCFM}
    \end{subfigure}
    \begin{subfigure}{0.19\textwidth}
        \centering
        \includegraphics[width=\linewidth]{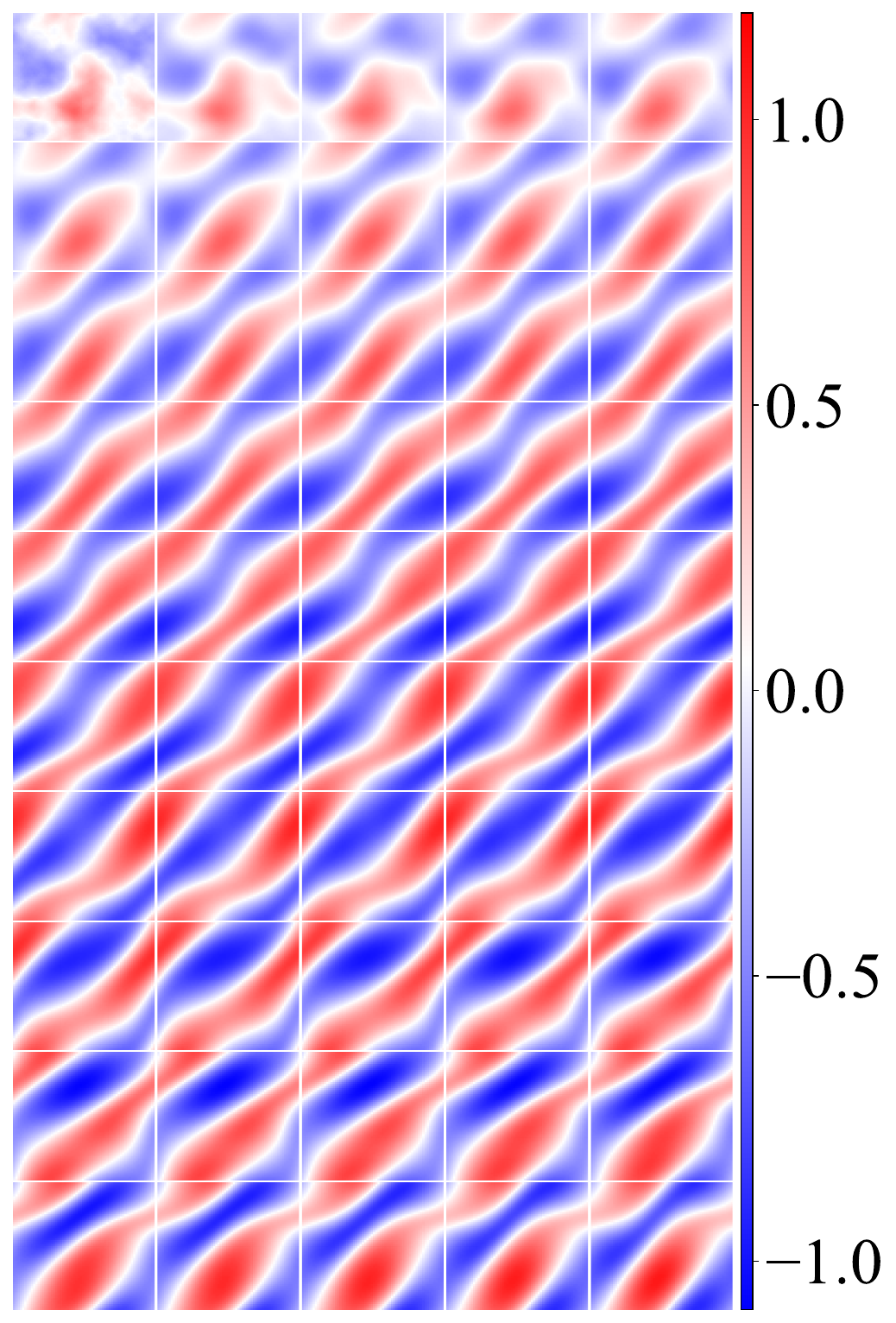}
        \caption{DPDE}
    \end{subfigure}
    \begin{subfigure}{0.19\textwidth}
        \centering
        \includegraphics[width=\linewidth]{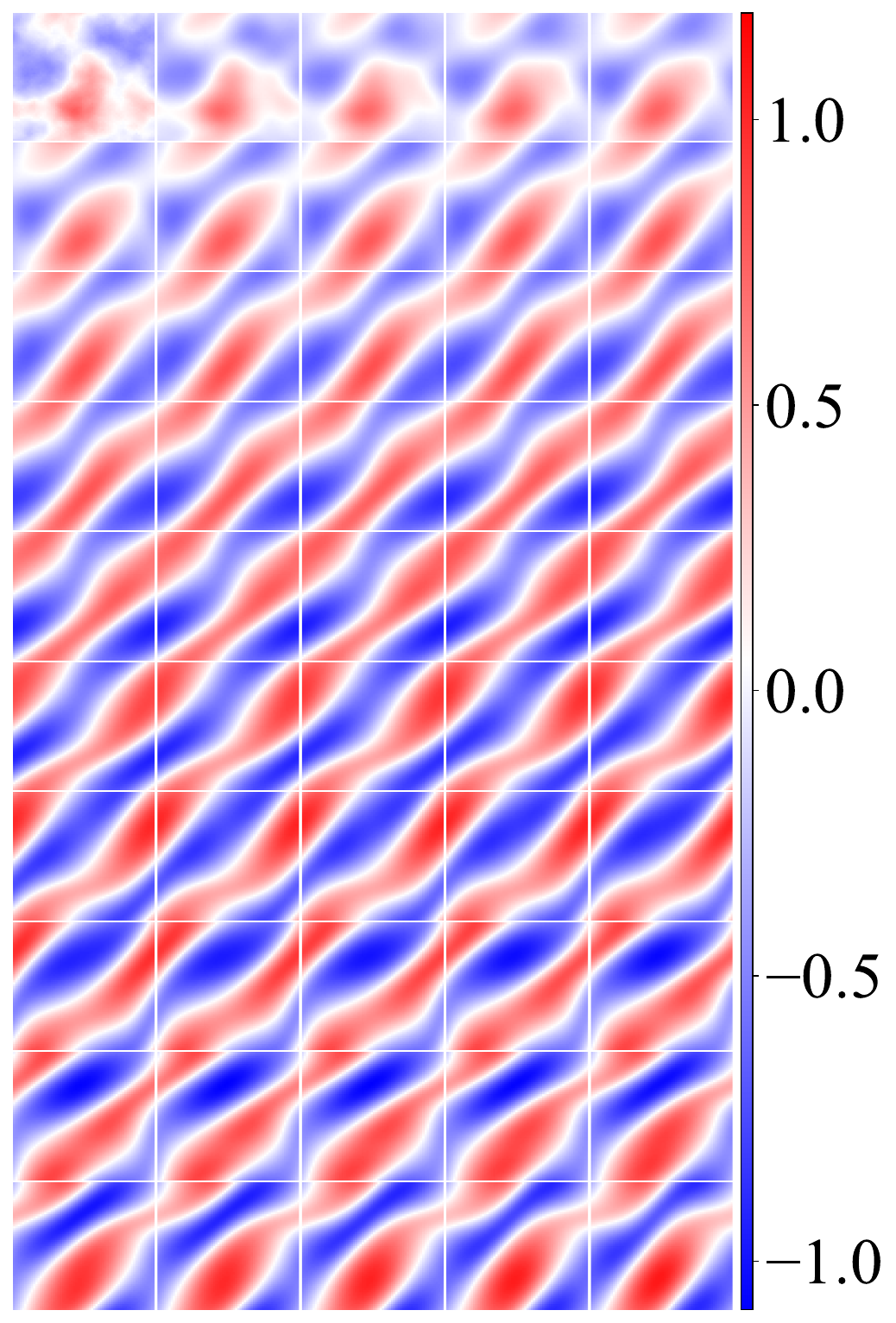}
        \caption{FFM}
    \end{subfigure}
    \caption{Example solution fields of the Navier--Stokes equation generated using the models compared with the ground truth.}
    \label{fig:ns_sample}
\end{figure}

\begin{table}[t]
  \centering
  \caption{Results with standard deviations for solution generation on the 2D Navier--Stokes equation. A lower value indicates better performance. Bold indicates the best value in each row.}
  \resizebox{0.95\linewidth}{!}{%
  
  \begin{tabular}{
    ll
    c % CCFM
    c % PCFM
    c % DiffusionPDE
    c % FFM
    c % RP
  }
    \toprule
    \textbf{Problem} & \textbf{Metric} & \textbf{CCFM (ours)} & \textbf{PCFM} & \textbf{DiffusionPDE} & \textbf{FFM} & \textbf{RP} \\
    \midrule
    \multirow{4}{*}{Navier--Stokes}
      & MMSE $\times 10^{-1}$ ($\downarrow$)         & $0.79 \pm (0.03)$ & $1.67 \pm (0.09)$ & $1.44 \pm (0.10)$ & $1.56 \pm (0.16)$ & $1.08 \pm (0.08)$ \\
      & SMSE $\times 10^{-1}$ ($\downarrow$)         & $0.81 \pm (0.11)$  & $1.53 \pm (0.18)$ & $9.21 \pm (0.39)$ & $9.57 \pm (0.53)$ & $0.97 \pm (0.25)$ \\
      & CV (IC) ($\downarrow$)                        & 0 & 0 & $0.12 \pm (0.03)$ & $0.13 \pm (0.08)$ & 0 \\
      & CV (CL) $\times 10^{-15}$ ($\downarrow$)      & $3.9 \pm (0.01)$ & $3.4 \pm (0.02)$ & $1.2 \pm (0.13)$ $\times 10^{10}$ & $1.2 \pm (0.12)$ $\times 10^{9}$  & $4.1 \pm (0.04)$ \\
    \bottomrule
  \end{tabular}%
  }
  \label{tab:ccfm- sensitivity-navier-stokes}
\end{table}

For the Navier--Stokes problem, to complement the results presented in
Section~\ref{sec:pde_results}, we report evaluation results for all physical
time frames. Figure~\ref{fig:ns_MMSE_all_frames} shows the MMSE of generated
samples from the models as a function of physical time. While the errors
increase over time, CCFM maintains relatively low error levels across all
frames compared with other models. This is also reflected in
Figure~\ref{fig:ns_mmse_time}, where CCFM outperforms other models, with the
maximum spatially averaged MMSE around 0.12 (others $>$ 0.20). Meanwhile, the
average solution residuals from CCFM, shown in Figure~\ref{fig:ns_residual},
are significantly lower than those of other models across all time frames. In
all cases, CCFM not only produces solutions with high accuracy but also yields
samples that strictly satisfy the constraints. Figure~\ref{fig:ns_sample} shows
an example of the true solution field compared with generated ones from the
models. Visually, the CCFM-generated sample consistently aligns with the ground
truth, especially at early time steps where DPDE and FFM fail. PCFM achieves
closer performance in the early steps yet exhibits larger magnitude
discrepancies (lighter blue than the ground truth). Table~\ref{tab:ccfm- sensitivity-navier-stokes} presents a comprehensive performance comparison of five different methods on the 2D Navier–Stokes equation solution generation task. Crucially, all results are reported as the mean performance with associated standard deviations, derived from 5 independent runs, demonstrating that the CCFM method achieves the best overall and most stable performance across the key error and conservation metrics.

\subsection{Physics-Informed Motion}
{Figure \ref{fig:pim} illustrates generated trajectories under Moon gravity, where the proposed approaches align closely with the ground truth (red circles) compared to the baseline while ensuring high-fidelity generated samples.}
\begin{figure}
    \centering
    \includegraphics[width=\linewidth]{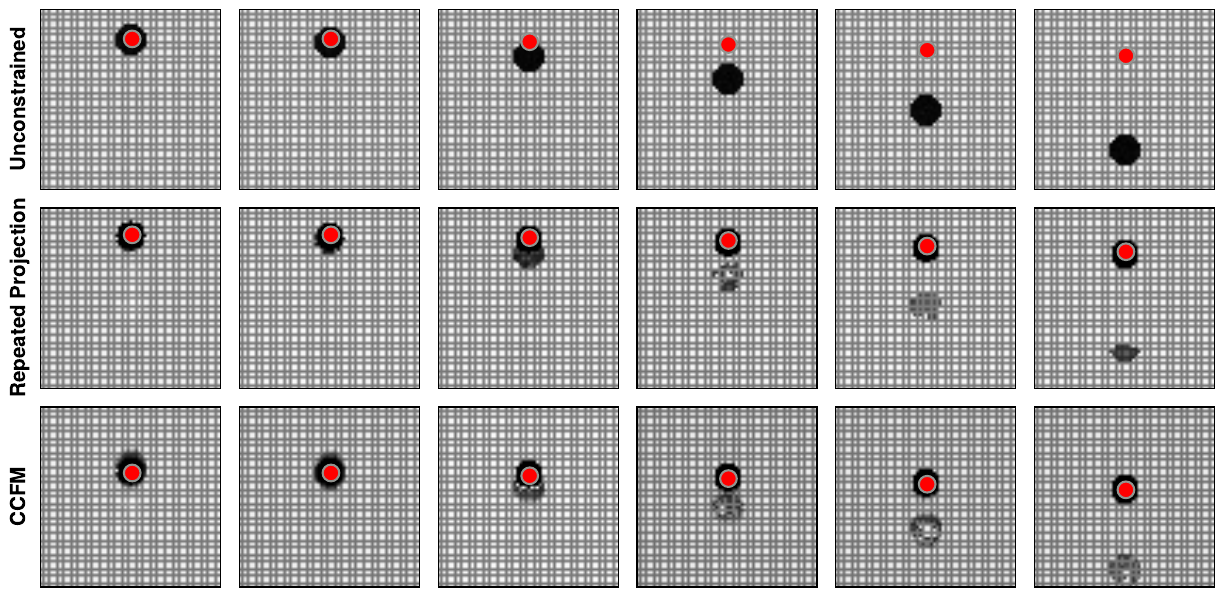}
    \caption{Generated trajectories instances for the physics-informed motion generation. In each frame, the black solid objects are generated using three constraint-handling approaches. The red circles represent the correct locations of objects under Moon gravity.}
    \label{fig:pim}
\end{figure}

\end{document}

%% file: math_commands.tex
\usepackage{amsmath,amsfonts,bm}

\def\eqref#1{Equation~\ref{#1}}
\def\1{\bm{1}}

\DeclareMathAlphabet{\mathsfit}{\encodingdefault}{\sfdefault}{m}{sl}
\SetMathAlphabet{\mathsfit}{bold}{\encodingdefault}{\sfdefault}{bx}{n}

\DeclareMathOperator*{\argmin}{arg\,min}

%% file: raise-lab.tex
\usepackage{xcolor}
\usepackage[most]{tcolorbox}
\usepackage[framemethod=tikz]{mdframed}
\usepackage{xparse}
\usepackage{tikz}

\newtcbox{\bracketeq}{
  nobeforeafter, enhanced, colback=gray!10, frame empty, math upper,
  boxsep=2pt, left=8pt, right=8pt, top=4pt, bottom=4pt, arc=4pt,
  tcbox raise base, % proper baseline for inline math
  overlay={
    \def\thk{1.0pt}   % line thickness (change to taste)
    \def\stub{4pt}    % length of the corner ticks
    \draw[line width=\thk] (frame.north west) -- (frame.south west);
    \draw[line width=\thk] (frame.north east) -- (frame.south east);
    \draw[line width=\thk] (frame.north west) -- ([xshift=\stub]frame.north west);
    \draw[line width=\thk] (frame.south west) -- ([xshift=\stub]frame.south west);
    \draw[line width=\thk] (frame.north east) -- ([xshift=-\stub]frame.north east);
    \draw[line width=\thk] (frame.south east) -- ([xshift=-\stub]frame.south east);
  }
}

\newtcbox{\boxeq}{%
  enhanced,
  on line,
  nobeforeafter,
  math upper,          % typeset the argument in math mode
  frame empty,
  colback=gray!10,
  boxsep=2pt, left=4pt, right=4pt, top=4pt, bottom=4pt,
  arc=3pt,
  tcbox raise base,
  overlay={
    \def\r{3pt}         % <-- give a unit
    \def\thk{1.0pt}     % line thickness of the bracket
    \def\stub{4pt}      % how much of the top/bottom to KEEP near corners
    \def\wipe{1pt}      % vertical clearance to fully cover the stroke

    \draw[line width=\thk, rounded corners=\r]
      (frame.north west) rectangle (frame.south east);

    \fill[white]
      ([xshift=\stub,yshift=\wipe]frame.north west)
      rectangle
      ([xshift=-\stub,yshift=-\wipe]frame.north east);
    \fill[white]
      ([xshift=\stub,yshift=\wipe]frame.south west)
      rectangle
      ([xshift=-\stub,yshift=-\wipe]frame.south east);
  }
}